\documentclass[11pt]{article}

\usepackage[a4paper,margin=1in]{geometry}
\usepackage[utf8]{inputenc}
\usepackage[T1]{fontenc}
\usepackage{amsmath,amsthm,amssymb,amsfonts}
\usepackage{bm}
\usepackage{graphicx,wrapfig}
\usepackage{subcaption}
\usepackage{booktabs}
\usepackage{nicefrac}
\usepackage{microtype}
\usepackage{enumitem}
\usepackage[table,dvipsnames]{xcolor}
\usepackage{tikz}
\usepackage{textcomp}
\usepackage{etoolbox}
\usepackage{natbib}
\usepackage[textsize=tiny]{todonotes}
\usepackage{url}
\usepackage{hyperref}
\usepackage{cleveref}

\graphicspath{{../}}

\colorlet{OliveLime}{OliveGreen!80!LimeGreen}
\definecolor{UrlBlue}{HTML}{1D4ED8}
\hypersetup{
  colorlinks=true,
  linkcolor=BrickRed,
  citecolor=OliveLime,
  urlcolor=UrlBlue,
  filecolor=UrlBlue,
  pdfborder={0 0 0}
}

\definecolor{InstructionBg}{HTML}{FFF8E1}      
\definecolor{InstructionBorder}{HTML}{F2C14E}  
\usetikzlibrary{arrows.meta,positioning,calc}
\theoremstyle{plain}
\newtheorem{theorem}{Theorem}[section]
\newtheorem{proposition}[theorem]{Proposition}
\newtheorem{lemma}[theorem]{Lemma}
\newtheorem{corollary}[theorem]{Corollary}
\theoremstyle{definition}

\newtheorem{assumption}[theorem]{Assumption}
\theoremstyle{remark}
\newtheorem{remark}{Remark}

\newcommand{\red}[1]{{\color{red}#1}}
\newcommand{\orange}[1]{{\color{orange}#1}}
\newcommand{\blue}[1]{#1}
\newcommand{\blfootnote}[1]{%
  \begingroup
  \renewcommand{\thefootnote}{}%
  \footnote{#1}%
  \addtocounter{footnote}{-1}%
  \endgroup
}

\usepackage{mathrsfs}

\def\argmin{{\arg\min}}

\def\bA{\mathbf{A}}

\def\bbE{\mathbb{E}}

\def\bbP{\mathbb{P}}

\def\bbR{\mathbb{R}}

\def\bSigma{{\boldsymbol\Sigma}}

\def\cA{\mathcal{A}}

\def\cG{\mathcal{G}}

\def\cI{\mathcal{I}}

\def\cM{\mathcal{M}}
\def\cN{\mathcal{N}}

\def\cT{\mathcal{T}}

\def\cX{\mathcal{X}}

\usepackage{yfonts}

\def\KL{{\sf KL}}

\def\Norm#1{\Vert#1\Vert}
\def\innerprod#1{\left\langle#1\right\rangle}

\def\pr{{\bbP}}

\def\tr{{\sf tr}}

\def\TV{{\sf TV}}

\def\op{{\rm op}}

\def\Cov{\text{Cov}}

\def\evel{\varepsilon^{\rm v}}
\def\escore{\varepsilon^{\rm s}}
\def\meanevel{\varepsilon_{\rm v}}
\def\meanescore{\varepsilon_{\rm s}}
\def\evelJone{\varepsilon^{\rm J,1}}
\def\meanevelJone{\varepsilon_{\rm J,1}}
\def\evelJtwo{\varepsilon^{\rm J,2}}
\def\meanevelJtwo{\varepsilon_{\rm J,2}}
\def\evelH{\varepsilon^{\rm H}}
\def\meanevelH{\varepsilon_{\rm H}}

\def\tr{{\sf Trace}}

\usepackage{bbm}
\def\ind{{\mathbbm{1}}}
\mathchardef\mhyphen="2D
\def\dx{{\mathrm{d}x}}

\def\logit{{\rm logit}}
\def\cf{{\it cf.}}

\newcommand{\bc}[1]{\left\{{#1}\right\}}
\newcommand{\br}[1]{\left({#1}\right)}
\newcommand{\bs}[1]{\left[{#1}\right]}

\def\1{\bm{1}}

\DeclareMathAlphabet{\mathsfit}{\encodingdefault}{\sfdefault}{m}{sl}
\SetMathAlphabet{\mathsfit}{bold}{\encodingdefault}{\sfdefault}{bx}{n}

\newcommand{\Law}{\text{Law}}

\DeclareMathOperator{\Tr}{Tr}

\title{Low-Dimensional Adaptation of Rectified Flow: A  Diffusion and Stochastic Localization Perspective}
\author{%
  Saptarshi Roy\textsuperscript{1}\quad
  Alessandro Rinaldo\textsuperscript{2}\quad
  Purnamrita Sarkar\textsuperscript{2}\\[0.5em]
  \small\textsuperscript{1}CMS Division, Mohamed bin Zayed University of Artificial Intelligence, Abu Dhabi, UAE\\
  \small\textsuperscript{2}Department of Statistics and Data Science, University of Texas at Austin, USA
}
\date{}

\begin{document}

\maketitle

\begin{abstract}
  In recent years, Rectified flow (RF) has gained considerable popularity largely due to its generation efficiency and state-of-the-art performance. In this paper, we investigate how well RF automatically adapts to the intrinsic low dimensionality of the support of the target distribution to accelerate sampling. We show that, using a carefully designed \blue{end-point adapted} time-discretization scheme and with sufficiently accurate drift estimates, the RF sampler enjoys an iteration complexity of order $O(k/\varepsilon)$ (up to log factors), where $\varepsilon$ is the precision in total variation distance and $k$ is the intrinsic dimension of
 the target distribution. In addition, we show that the denoising diffusion probabilistic model (DDPM) procedure is equivalent to a stochastic version of RF by establishing a novel connection between these processes and stochastic localization. Building on this connection, we further design a stochastic RF sampler that also adapts to the low-dimensionality of the target distribution under mild requirements on the accuracy of the drift estimates, and also with a specific time schedule. We illustrate the efficacy of newly designed time-discretization schedules with simulations on the synthetic data and text-to-image (T2I) data experiments.
\end{abstract}


\section{Introduction}
Recently, Rectified flow or Flow matching \cite{liu2022rectified, lipman2022flow} has gained much attention due to its state-of-the-art generative performance across various modalities including image \citep{liu2023instaflow, yan2024perflow}, audio \citep{ audio_RF_wang2024frieren, audio_RF_liu2025flashaudio}, and video \citep{video_RF_chen2025goku, video_RF_liu2025improving}. 
Its effectiveness arises from a deterministic ordinary differential equation (ODE) based sampler and the ability of reflow to learn nearly straight transport paths, resulting in faster sampling. This has naturally sparked theoretical interests \cite{bansal2025wasserstein, guan2025mirror, zhou2025an, kumar2026conditional} to understand the generation quality of RF from a convergence viewpoint under a suitable metric. However, these recent results suffer from \textit{curse of dimensionality}. Therefore, they fail to explain the impressive performance of RF for \textit{high-dimensional} data, e.g., image data, even though they reside in a low-dimensional space \cite{ansuini2019intrinsic, pope2021intrinsic}. In contrast, recent works \citep{potaptchik2024linear, azangulov2024convergence,huang2024denoising,li2024adapting,liang2025low, kumar2026dimension} have established low-dimensional adaptation of DDPM and denoising diffusion implicit models (DDIM)  during sampling. This explains the accelerated and high-quality generative capabilities of DDPM and DDIM in practice, even for high-dimensional data.
To the best of our knowledge such results are not currently present for RF models.
To this end, we make three key contributions in this paper: 
\begin{enumerate}
    \item \textbf{Low-dimensional adaptation of RF:}~We establish that a class of endpoint-adapted time-discretization schedules provably enables RF to exploit the intrinsic low-dimensional structure (\cf~Theorem \ref{thm: RF - general convergence of Eulerian scheme} and Theorem \ref{thm: RF - convergence of Eulerian scheme}) of the target distribution during the \textit{first reflow}. Our theoretical findings complement the empirical observations of \cite{lee2024improvingtrainingrectifiedflows}, who demonstrated that U-shaped time discretization improves the generative quality of the first-reflow model.
    In particular, the RF sampler enjoys $O(k/\varepsilon)$ (upto log factors) iteration complexity, where $\varepsilon$ is the precision of the generated samples in total variation distance and $k$ is the intrinsic dimension of the target distribution. This is particularly useful because it can eliminate repeated reflow steps, which are computationally costly and often cause model collapse~\citep{model_collapse_zhu2024analyzing, reflow_bad_kim2025balanced}.

 \item 
\textbf{Unification of samplers through stochastic localization:} Using the stochastic localization \cite{eldan2013thin,eldan2020taming} perspective, we demonstrate that DDPM is equivalent to the stochastic version of RF (which we refer to as STOC-RF) under a proper time change. Stochastic localization has been recently shown to provide a unifying and effective framework for representing and analyzing diffusion-based sampling processes \citep{benton2023nearly, montanari2023sampling, el2022information}.
 This equivalence is convenient as it allows us to view STOC-RF through the lens of well-understood DDPM processes, which are known to adapt to the low-dimensionality of the target distribution. In fact, Section \ref{sec: SL} shows that stochastic localization can be leveraged to provide a unified and convenient representation of a variety of sampling methods. 

 \item \textbf{Designing efficient stochastic RF sampler:} Using the connection between DDPM and STOC-RF, we also show that the stochastic RF sampler provably adapts to the low-dimensional structure (\cf~Theorem \ref{thm: stoc-RF sampler non-Eulerian}) of the target distributions under less stringent assumptions on the learned model. We illustrate its competitive performance through various simulation experiments. 
\end{enumerate}
\textbf{Notation.} {Throughout the paper, with some abuse of notations, we will identify  a probability distribution with its Lebesgue density, which we will assume exists unless otherwise noted.} 
For two  distributions $P$ and $Q$ on $\mathbb{R}^d$ with Lebesgue densities $p$ and $q$ respectively, their total variation distance is denoted by $\TV(p,q):= \sup_{A} \vert{P(A) - Q(A)}\vert = 1/2 \int_{\mathbb{R}^d} |p(x) - q(x)| dx $ , where the supremum is over all Borel subsets of $\mathbb{R}^d$.

\section{Background and Preliminaries}
\label{sec: background}

\subsection{Rectified Flow}
\label{sec: Rf}
In this section, we briefly introduce the basics of Rectified flow, a generative model that transitions between two distributions $p_0$ and $p_1$ by solving ODEs. We refer the reader to \cite{Liu2024_FlowNotes,liu2022rectified} for an exhaustive treatment.  
We let $p_1 := \Law(X_1)$ be the target data distribution with its support $\cX \subseteq \bbR^d$, the data generating distribution from which we would like to draw samples. The linear-interpolation process is defined by
\begin{equation}
\label{eq: RF linear process}
X_t = t X_1 + (1-t) X_0, \quad 0\leq t \leq 1
\end{equation}
where $p_0$ is the \textit{standard Gaussian} distribution. We will assume an independent initial coupling,  i.e., $(X_0, X_1) \sim p_0 \otimes p_1$ and let  $p_t = \Law(X_t)$, for any $t \in [0,1]$.
The RF procedure is a deterministic sampler arising from an ODE with {\it drift (or velocity) function} $v:\bbR^d \times [0, 1] \rightarrow \bbR^d$ defined as the solution to the optimization problem 
\begin{align}
    v &= \argmin_{f} \int \bbE{\Norm{X_1 - X_0 - f(X_t, t)}_2^2} dt , \label{drift_obj}
\end{align}
where the minimization is over all functions $f:\bbR^d \times [0, 1] \rightarrow \bbR^d$. 
For each $t$, The above objective in (\ref{drift_obj}) is minimized at
$
 (x,t) \mapsto   v_t(x) := v(x, t) = \bbE[X_1-X_0 \mid X_t = x]. 
$
 Under appropriate regularity conditions (investigated in \cite{bansal2025wasserstein, mena2025statisticalpropertiesrectifiedflow}), the ODE
\begin{align}
    dZ_t = v_t(Z_t)\ dt, \quad Z_0 \sim \cN(0, I_d) \label{eq: ode-true}
\end{align}
 satisfies the marginal preserving property,
 i.e., $\Law(Z_t) = \Law(X_t) = p_t$, owing to the Fokker-Planck equation
 \begin{equation}
     \label{eq: RF-FP}
     \frac{\partial p_t}{\partial t} + \nabla. (v_t p_t) = 0.
 \end{equation}
 Hence, ODE \eqref{eq: ode-true} can be used for sampling. 

To implement this procedure, one faces two key challenges. 
First, given a sample from the target distribution $p_1$, one has to estimate the drift function, which is naturally done by empirically minimizing the objective in (\ref{drift_obj}) over a large and expressive function class (e.g., U-Net). 
Secondly, for sampling, one must rely on a time-discretization scheme of the ODE \eqref{eq: ode-true}. One such scheme is the Eulerian update rule, which, for a given discretization of the time course in $N$ intervals, with $0 = t_0 < t_1 < \dots < t_N =1$, evaluates 
\begin{align}
\label{eq: emp-ode-disc}
    Y_{t_{i+1}} =  Y_{t_{i}} + \eta_i\widehat v_{t_{i}}(Y_{t_{i}}); \; \eta_i :=  t_{i+1}-t_{i}, 0 \le i \le N-1,
\end{align}
where $Y_0 \sim \cN(0, I_d)$ and $\widehat v_t(\cdot)$ is the estimated velocity at time $t$. Let $\widehat p_1 := \Law(Y_1)$ denote the distribution of the final step. Other types of time discretized schemes of the ODE \eqref{eq: ode-true} can be chosen for the sampling step, and the resulting discretization errors impact the sampling fidelity differently. This has been the focus of a recent line of works \citep{benton2023nearly, chen2022sampling, li2024towards, li2024sharp, huang2024denoising, liang2025low} aimed at identifying time-discretized versions of DDPM and DDIM samplers to achieve tight convergence rates.

\paragraph{The velocity function.} 
Before proceeding, we illustrate the connection between the velocity $v_t(x)$ and the score function $s_t(x):= \nabla \log p_t(x)$. Recall that $X_0 \sim \cN(0, I_d)$ and independent of $X_1\sim p_1$. Therefore, Tweedie's formula \citep{efron2011tweedie, meng2021estimating} immediately yields $v_t(x) =  \frac{x}{t} + \frac{(1-t)}{t} s_t(x)$. Thus, estimating $v_t$ is essentially equivalent to estimating the score $s_t$ for $t \in (0,1)$. 


\subsection{Stochastic rectified flow (STOC-RF)}
Since the work of DDIM \blue{\cite{song2021denoising} and score-based diffusion models \cite{song2021scorebased}, it is well known that an SDE admits a corresponding probability-flow ODE that maintains the same marginal. Conversely, an ODE can also be augmented with appropriately matched drift and diffusion terms to construct a stochastic sampler that preserves these marginals. However, the joint distributions could be different; resulting in mismatch between the stochastic processes.} In fact, prior works \cite{albergo2023stochastic, xue2024unifying, Hu_2025_ICCV} suggest that stochastic samplers typically enjoy better generation quality. One reason behind the inferior performance of deterministic samplers is the possibility of error accumulation during the sampling step along the discrete trajectories stemming from inaccurate velocity estimates. In other words, if $Y_{t_i}$ differs 

\begin{figure} 
  \begin{center}
    \includegraphics[width = 0.8\textwidth]{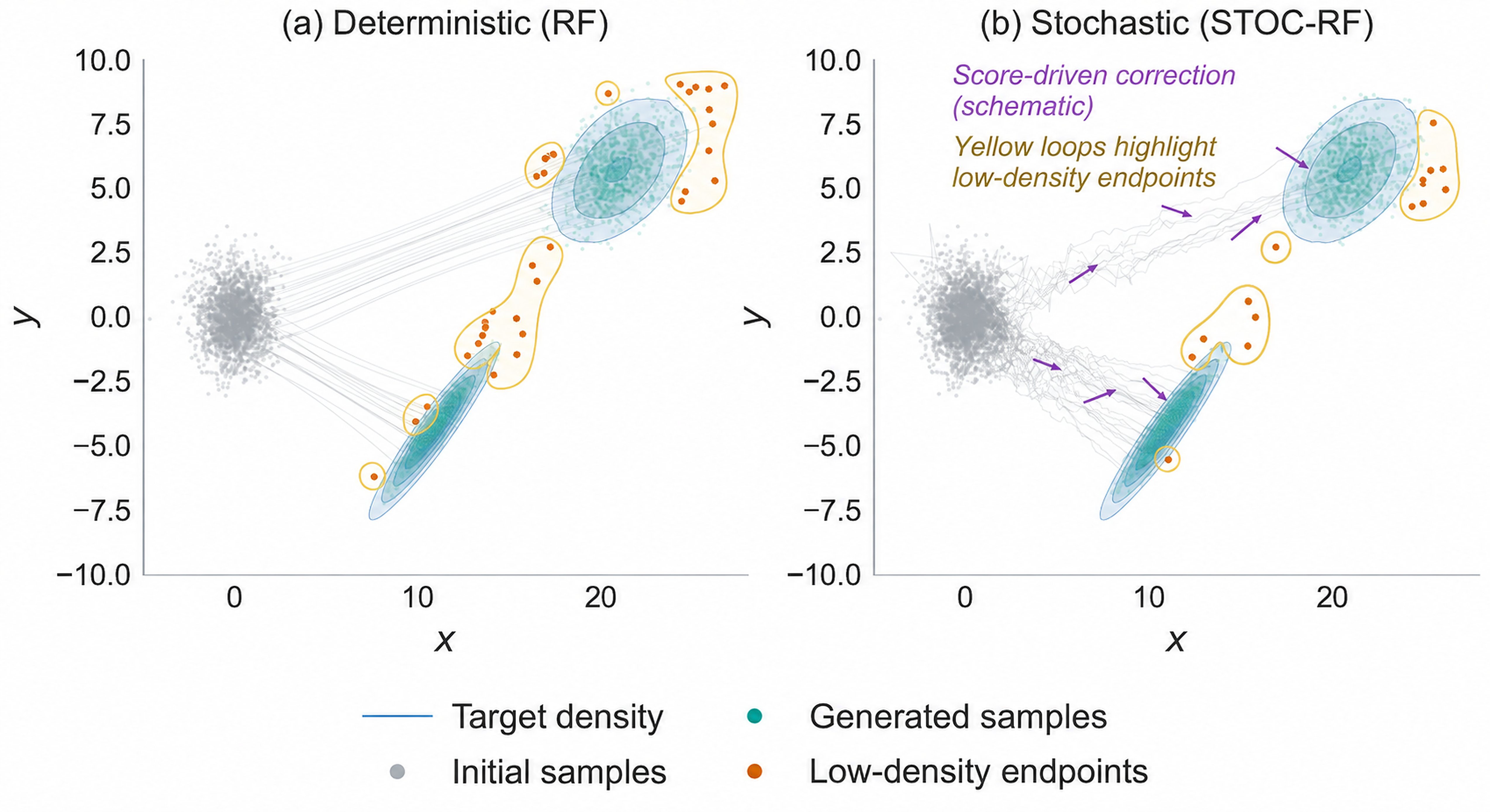}
  \end{center}
  \caption{Trajectories and generated samples from deterministic (RF) and stochastic (STOC-RF) samplers targeting a two-component Gaussian mixture. \blue{Blue} contours indicate target density, and highlighted endpoints lie in low-density regions. {\color{violet}Purple} arrows illustrate the score-driven component of the Langevin correction, which locally directs samples toward higher density and can counteract accumulated sampling errors—acting as a soft “guardrail.” The stochastic sampler exhibits fewer low-density endpoints in this example.
}
  \label{fig: ODE vs SDE 2-GMM}
\end{figure}
significantly from $Z_{t_i}$ for some $t_i$, then the Euler step \eqref{eq: emp-ode-disc} accumulates error due to inaccurate evaluation of $\widehat v_{t_i}(Y_{t_i})$. Therefore, this can further reinforce the deviation in evaluating $Y_{t_{i+1}}$.
A way to mitigate the compounding error from successive Euler
steps is to introduce a stochastic Langevin correction at each step. Empirically, \cite{hu2025amo} showed that the resulting stochastic RF samplers   enjoy better text rendering quality over  deterministic samplers. We further illustrate this point with Figure \ref{fig: ODE vs SDE 2-GMM}, showing that the ODE-based RF trajectories appear to generate more outliers compared to the SDE-based STOC-RF sampler. 

Following e.g. \cite[Chapter 5]{Liu2024_FlowNotes}, the stochastic version of RF is specified by the SDE 
\begin{equation}
\label{eq: RF-SDE}
\begin{aligned}
    & d\tilde Z_t = v_t(\tilde Z_t)dt + \underbrace{\gamma_t s_t(\tilde Z_t)dt + \sqrt{2 \gamma_t} d \tilde B_t}_{\text{Langevin correction}};
\end{aligned}
\end{equation}
for $t\in[t_0,1)$, where $t_0>0$ and $\tilde Z_{t_0}\sim p_{t_0}$. This is obtained from adding a Langevin correction to the original RF ODE \eqref{eq: ode-true}.
Above, $\gamma_t\geq  0$ is a possibly time-varying diffusion coefficient, and $\{\tilde B_t\}_{t\ge 0}$ is the standard Brownian motion in $\bbR^d$. 
One can check (see, e.g.,  \cite[Section 5.2]{Liu2024_FlowNotes} or \cite{song2021scorebased}) that the marginal $\tilde p_t := \Law(\tilde Z_t)$ satisfies  the continuity equation 
\begin{equation}
    \label{eq: RF-SDE-FP}
    \frac{\partial \tilde p_t}{\partial t} + \nabla. (\tilde v_t \tilde p_t) = 0,
\end{equation}
where $\tilde v_t(x) = v_t(x) + \gamma_t s_t(x) - \gamma_t \nabla \log \tilde p_t(x)$.   With this matching initial law, the standard uniqueness argument for the marginal evolution gives $\tilde p_t=p_t$, i.e., $\Law(\tilde Z_t)=\Law(Z_t)$, for $t\in[t_0,1)$. Thus the SDE follows the RF probability path towards $p_1$. In practice, this is accomplished via an appropriate time discretization scheme
\begin{equation}
    \label{eq: RF-SDE-disc}
    \tilde Z_{t_{i+1}} = \tilde Z_{t_i} + \tilde \eta_i \bc{v_{t_i}(\tilde Z_{t_i}) +  \gamma_{t_i} s_{t_i}(\tilde Z_{t_i})} + \sqrt{2 \tilde \eta_i \gamma_{t_i}} \xi_{t_i},
\end{equation}
where $\bc{t_i}_{i=0}^N$ is an appropriate schedule of times, the $\xi_{t_i}$'s are i.i.d. draws from $\cN(0, I_d)$, and the $\tilde \eta_i$'s are (typically small) step-sizes. The Langevin term essentially acts as a course-corrector to the trajectory of \eqref{eq: RF-SDE-disc} to adjust the trajectory distribution closer to $p_t$. Finally, one has freedom to choose the diffusion coefficients $\gamma_t$. In this paper we will restrict ourselves to the choice $\gamma_t = (1-t)/t$, which, as we will see below, is the \textit{appropriate noise scaling that yields equivalence between the STOC-RF SDE \eqref{eq: RF-SDE} and DDPM through the lens of Stochastic localization}. Endpoint behavior is discussed in Appendix \ref{app: equivalence of backward and stoc-RF}. This also addresses an question raised during the ICML tutorial talk  \cite{liu2025flowing} regarding the motivation of such choice.

\section{Stochastic localization, RF and DDPM: \blue{A Unified View}}
\label{sec: SL}
In this section, we present a unified framework for studying and analyzing the stochastic DDPM sampler and the deterministic RF sampler. We show that stochastic localization acts as a unifying bridge between STOC-RF and DDPM, and it allows us to seamlessly transit the convergence property of one to another. This framework also extends to deterministic samplers, including RF and DDIM; see Appendix \ref{sec: more details on SL}. Figure \ref{fig:SL-unified} provides a schematic of the equivalence results. 

\begin{figure*}[h]
\centering
\includegraphics[width=0.7\textwidth]{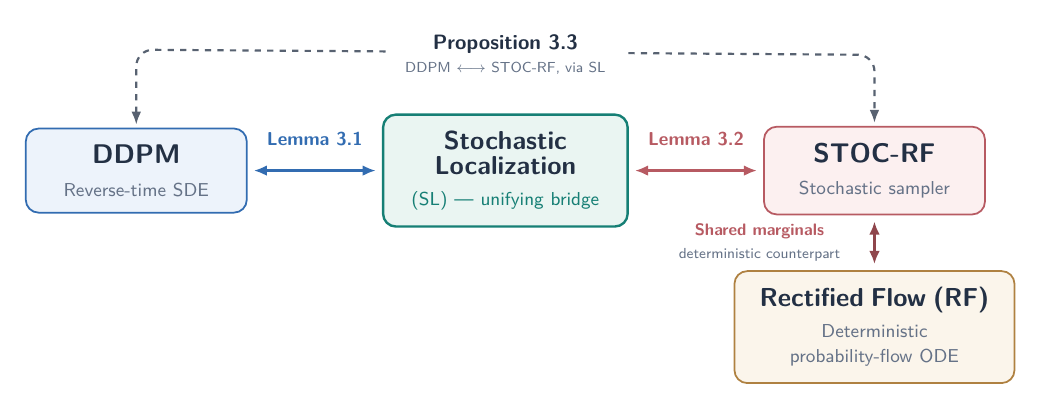}
\caption{
Stochastic localization (SL) provides a unifying bridge between
DDPM and STOC-RF. 
Lemmas~3.1 and~3.2 establish the respective
equivalences in law under suitable time reparameterizations and
scalings, leading to Proposition~3.3, which connects the
DDPM and STOC-RF samplers.
}
\label{fig:SL-unified}
\end{figure*}

\paragraph{Stochastic localization (SL).} 
SL  is a probabilistic technique to study the properties of high-dimensional distributions due to \cite{eldan2013thin,eldan2020taming} that provides a powerful unifying framework to understand, analyze and improve generative models and sampling algorithms based on scores. See  \cite{el2022information}, \cite{montanari2023sampling} and the survey by \cite{shi2025perspectivesstochasticlocalization}. To sample from the target data distribution $p_1$, SL focuses on an appropriate measure-valued stochastic process $\{\nu_s\}_{s\ge 0}$ such that $\nu_s$ ``localizes'' towards $p_1$ as $s \to \infty$. A straightforward way to formulate this  process is to let, for any $s \geq 0$, $\nu_s := \Law (X_1 \mid U_s)$, where $X_1 \sim p_1$ and
\begin{equation}
\label{eq: SL process}
    U_s = s X_1 + B_s,
\end{equation}
with $\{B_s\}_{s \ge 0}$  a standard Brownian motion in $\bbR^d$ \citep{montanari2023sampling}, independent of $X_1$. Thus, $\nu_s$ is a random measure depending on $U_s$. Note that $U_s/s \to X_1$ almost surely as $s \to \infty$, i.e., $\nu_s$ does indeed localize around $p_1$. More details on SL is deferred to Appendix \ref{sec: more details on SL}. 


\paragraph{SL and DDPM.}
DDPMs are a class of diffusion-based sampling mechanism for generative AI introduced by \cite{ho2020denoising}.  DDPM first starts with a \textit{forward process} $(Y'_\tau)_{\tau\ge 0}$ starting from $Y_0' = X_1 \sim p_1$ and  evolving according to the SDE
\begin{equation}
    \label{eq: forward process}
    dY_\tau' = - \beta(\tau) Y_\tau' d\tau + \sqrt{2 \beta(\tau)} dB^\prime_\tau, \; Y_0' \sim p_1,
\end{equation}
where $\beta(\tau)>0$ is a continuous diffusion coefficient with $\int_0^\infty\beta(u)\,du=\infty$ and $\bc{B^\prime_\tau}_{\tau\ge 0}$ is standard Brownian motion. If $\beta(\tau) = 1$, then the SDE \eqref{eq: forward process} reduces to the classical Ornstein-Uhlenbeck (OU) process.
Typically, The forwards process is continued till a large time $\tau = N$ such that $Y_N'$ is approximately Gaussian. Let $q_\tau(y_\tau')$ be the marginals of $y_\tau'$ following \eqref{eq: forward process}. Then in order to sample from the target distribution $p_1$, a \emph{reverse process} $\{\overleftarrow Y_\tau\}_{\tau \in [0,N]} $ is formulated, starting at $\overleftarrow Y_0 = Y_N'$ and evolving according to the reverse-time SDE  
\begin{equation}
\label{eq: reverse OU process}
\begin{aligned}
  d\overleftarrow Y_{\tau} =  & \{\overleftarrow Y_\tau + 2 \nabla \log q_{N - \tau}(\overleftarrow y_\tau)\} \beta(N-\tau)d\tau + \sqrt{2 \beta(N-\tau)} d \overleftarrow B_\tau,
  \end{aligned}
\end{equation}
where $\{\overleftarrow B_\tau\}_{\tau \ge 0}$ is a standard Brownian motion.
Classical results in SDE \cite{anderson1982reverse, haussmann1986time} show that $\overleftarrow Y_\tau \overset{d}{=} Y_{N - \tau}'$ which allows us to generate samples from $p_1$ via simulating the backward SDE \eqref{eq: reverse OU process}.

It is known that DDPMs and stochastic localization are equivalent \citep{montanari2023sampling}. 
In particular, \cite{benton2023nearly}  showed that the forward process \eqref{eq: forward process} is \textit{equivalent in law} to the SL process \eqref{eq: SL process} under an appropriate time change. In this section, we restate this result in a slightly different fashion that will allow us to easily see the connection between DDPM and RF.   
To this end, we define  $\omega_{\tau} := \exp(-2 \int_0^\tau \beta(u) du)$. Then the following lemma shows that the the \textit{forward process \eqref{eq: forward process} and the SL process \eqref{eq: SL process} are equivalent:}
\begin{lemma}
\label{lemma: time change DDPM-SL}
    Let $\{\tau(s)\}_{s>0}$ satisfy $\frac{1 -\omega_{\tau(s)}}{\omega_{\tau(s)}} = \frac{1}{s}$. Then,
    $\bc{\frac{Y'_{\tau(s)}}{\sqrt{\omega_{\tau(s)}}}}_{s>0} \overset{d}{= } \bc{\frac{U_s}{s}}_{s>0}.$
\end{lemma}
If $\beta(\tau) = 1$ for all $\tau$ (OU process), then the time transformation in Lemma \ref{lemma: time change DDPM-SL} simplifies to be $\tau(s)  = \frac{1}{2}\log(1 + s^{-1})$, and we get $\bc{\frac{Y'_{\tau(s)}}{e^{- \tau(s)}}}_{s>0} \overset{d}{= } \bc{\frac{U_s}{s}}_{s>0}$. Note, this matches the findings in \cite[Section 1.2]{benton2023nearly}. The proof is deferred to Appendix \ref{app: SL = RF = DDPM}.

\paragraph{SL, RF and STOC-RF.} Next, we will also show that the probability path of RF is also equivalent to the SL process \eqref{eq: SL process}. Recall that that RF starts by constructing the linear process $X_t = t X_1 + (1-t)X_0$, where $X_1 \sim p_1, X_0 \sim \cN(0, I_d)$ and $t \in [0,1]$. To this end, we consider the process 
\begin{equation}
    \label{eq: RF linear stoc process}
    \tilde X_t = t X_1 +  t W_{(1-t)^2/t^2}, \quad t \in (0,1]
\end{equation}
where $\{W_\tau\}_{\tau\ge0}$ is a Brownian motion independent of $X_1$. Note that $\Law(\tilde X_t)=\Law(X_t)$ for all $t\in(0,1]$. In fact, the entire process \eqref{eq: RF linear stoc process} has the same law of the SL process \eqref{eq: SL process}, up to time-change.
\begin{lemma}
\label{lemma: SL and RF are equivalent}
    Let $\{t(s)\}_{s>0}$ satisfy $\br{\frac{1 - t(s)}{t(s)}}^2 = \frac{1}{s}$. Then,
    $
        \bc{\frac{\tilde X_{t(s)}}{t(s)}}_{s>0} \overset{d}{=} \bc{\frac{U_s}{s}}_{s>0}.
    $
\end{lemma}
The proof is in Appendix \ref{app: SL = RF = DDPM}. 
Additionally, \Cref{lemma: time change DDPM-SL} and \Cref{lemma: SL and RF are equivalent} leads to the following:
\begin{equation}
    \label{eq: DDPM and RF equivalent}
    \bc{Y^{'}_{\tau}/\sqrt{\omega_{\tau}}}_{\tau \ge 0} \overset{d}{=} \bc{\tilde X_{t(\tau)}/t(\tau)}_{\tau \ge 0},
\end{equation}
where $t(\tau) = \frac{\sqrt{\omega_\tau}}{\sqrt{\omega_\tau} + \sqrt{1 - \omega_{\tau}}}$. Now, on any positive-time interval with initial law $p_{t_0}$, \eqref{eq: RF linear stoc process} has the same law as the solution of \eqref{eq: RF-SDE} with $\gamma_t=(1-t)/t$, and  \eqref{eq: DDPM and RF equivalent} suggests that under the same time-change the STOC-RF process \eqref{eq: RF-SDE} and the backward process \eqref{eq: reverse OU process} are equivalent.
\begin{proposition}
    \label{prop: RF-SDE and backwrd process are equivalent}
    Let $t(\tau)=\frac{\sqrt{\omega_\tau}}{\sqrt{\omega_\tau}+\sqrt{1-\omega_\tau}}$, and let $\{\tilde Z_t\}_{t\in[t(N),1)}$ solve \eqref{eq: RF-SDE} with $\gamma_t=(1-t)/t$ and $\tilde Z_{t(N)}\sim p_{t(N)}$. Then,
    $\bc{\frac{\sqrt{\omega_{N-\tau}}\tilde Z_{t(N-\tau)}}{t(N-\tau)}}_{0\le\tau<N}$ is a solution of \eqref{eq: reverse OU process} with initial law $q_N$. Conversely, if $\{\overleftarrow Y_\tau\}_{0\le\tau<N}$ solves \eqref{eq: reverse OU process} with $\overleftarrow Y_0\sim q_N$, then
    $\bc{\frac{\overleftarrow Y_{\tau(t)}}{\sqrt{\omega_{N-\tau(t)}}+\sqrt{1-\omega_{N-\tau(t)}}}}_{t\in[t(N),1)}$ solves \eqref{eq: RF-SDE} with initial law $p_{t(N)}$. Here $\tau(t)\in[0,N)$ is the unique solution of
    $\int_0^{N-\tau(t)}\beta(u)\,du=\log\br{\sqrt{1+\frac{(1-t)^2}{t^2}}}$.
 \end{proposition}
The above proposition is very useful in analyzing the convergence rate of STOC-RF. Essentially, Proposition \ref{prop: RF-SDE and backwrd process are equivalent} allows us to transit between DDPM and STOC-RF under simple time change and scaling. Therefore, convergence properties of one process can be carried over to the other one with the same transformation. This is the key insight that we will use to prove our convergence result of STOC-RF in later sections. The detailed proof is deferred to Appendix \ref{app: equivalence of backward and stoc-RF}.

\section{Main Results}
\label{sec: main results}
In this section, we present our main convergence results for both deterministic and stochastic samplers. We  begin by introducing and commenting on our key assumptions. Throughout, $\bc{t_i}_{i=0}^N$ refers to the time schedule for both the RF and STOC-RF procedures, as in  \eqref{eq: emp-ode-disc} and \eqref{eq: RF-SDE-disc}, respectively.

\subsection{Low-dimensionality of target distribution}
To formalize the notion of low-dimensionality of $p_1$, we turn to the geometric complexity of its support $\cX$ as measured by its metric entropy. 

\begin{assumption}[Low-dimensionality]
\label{assumption: intrinsic dimension}
    Let $\epsilon_0 = N^{-c_{\epsilon_0}}$ for some sufficiently large $c_{\epsilon_0}>0$. The $\epsilon_0$ covering number of $\cX$ with respect to $\Norm{\cdot}_2$ norm, $\cN(\cX, \Norm{\cdot}_2, \epsilon_0)$, satisfies
    $
    \log \cN(\cX, \Norm{\cdot}_2, \epsilon_0) \le C_{\rm cover} k \log N,
    $
    for some constant $C_{\rm cover}>0$ and integer $k$. The quantity $k$ is referred as the intrinsic-dimension of $\cX$.
\end{assumption}
This assumption was proposed in recent studies by \cite{huang2024denoising,li2024adapting,liang2025low} on adapting diffusion methods to the intrinsic low-dimensional structure of the target data. The covering resolution and the entropy bound in this assumption both depend on the number of steps $N$.  Informally, Assumption \ref{assumption: intrinsic dimension} requires the support of $p_1$ to be concentrated \textit{on or near} a $k$-dimensional set.
The above definition of intrinsic dimension is fairly general, as it covers a variety of important low-dimensional structures. Examples satisfying Assumption \ref{assumption: intrinsic dimension} include $k$-dimensional linear subspace in $\bbR^d$ and $k$-dimensional non-linear manifolds (provided $\cX$ is polynomially bounded as in Assumption \ref{assumption: bounded support}). Structures with doubling-dimension $k$ \citep{dasgupta2008random, kpotufe2012tree} also satisfies Assumption \ref{assumption: intrinsic dimension}. More detailed discussion on this can be found in Section 4.1.1 of \cite{huang2024denoising}. Moreover, in practice, target distribution (e.g., image data) is supported on a lower dimensional manifold and the intrinsic dimension $k \ll d$ \citep{facco2017estimating, ansuini2019intrinsic, pope2021intrinsic}.
\begin{assumption}[Bounded support]
    \label{assumption: bounded support}
    There exists a universal constant $c_R>0$ such that 
    $
    \sup_{x \in \cX} \Norm{x}_2 \le R, \text{where $R = N^{c_R}$}.
    $
\end{assumption}
The bounded support assumption is not as stringent as it may appear at first, as the diameter of $\cX$ is allowed to scale polynomially (with arbitrarily large degree) in the
number of iterations of the sampler, which accommodates a wide range of practical applications, such as
image generation.
\subsection{Assumptions on velocity approximation error}
Our next set of assumptions are concerned with the approximation error of the estimator $\hat v_{t}$. 
 
\begin{assumption}
\label{assumption: velocity approximation}
    Define $(\evel_i)^2:= \bbE \Norm{v_{t_i}(X_{t_i}) - \widehat{v}_{t_i}(X_{t_i})}_2^2$. Then, 
    $\frac{1}{N} \sum_{i =1}^{N-2} (\evel_i)^2 \le \meanevel^2.$
\end{assumption}
The preceding assumption effectively calls for strong control of the mean-squared error in the estimated velocity field. 
\blue{Recent studies \cite{boffi2025shallow, kumar2026flow} have demonstrated that modern neural networks can attain low approximation error through low-dimensional adaptation. }
Similar assumptions have been considered in \cite{bansal2025wasserstein, guan2025mirror} to prove Wasserstein convergence of RF. In the context of score-matching algorithms, similar assumptions related to the score-approximation error have also appeared in prior works such as \cite{benton2023nearly, chen2022sampling, li2024sharp, liang2025low}. 


\begin{assumption}[Higher-order approximations]
    \label{assumption: higher-order approximation error}
   $\;$
    \begin{enumerate}[label=(\alph*)]
        \item \label{assumption: Jone} Let $(\evelJone_i)^2 := \bbE \Norm{\nabla v_{t_i}(X_{t_i}) - \nabla \widehat v_{t_i}(X_{t_i})}_F^2$. Then, $N^{-1} \sum_{i =1}^{N-2} (\evelJone_i)^2 \le \meanevelJone^2$. Also, we assume $ \eta_i u^\top\nabla \widehat v_{t_i}(x)u \ge - \Norm{u}_2^2/2$ for all $u, x \in \bbR^d$ and $i \in\{0, \ldots N-2\}$.

        \item \label{assumption: Jtwo} Let $(\evelJtwo_i)^2 := \bbE [\Tr(\nabla v_{t_i}(X_{t_i}) - \nabla \widehat v_{t_i}(X_{t_i}))^2]$. Then, $N^{-1} \sum_{i =1}^{N-2} (\evelJtwo_i)^2 \le \meanevelJtwo^2$.

        \item \label{assumption: Hessian} Let $(\evelH_i)^2 := \bbE \Norm{\nabla \Tr(\nabla v_{t_i}(X_{t_i}) - \nabla \widehat v_{t_i}(X_{t_i}))}_2^2$. Then, $N^{-1} \sum_{i =1}^{N-2} (\evelH_i)^2 \le \meanevelH^2$.
    \end{enumerate}
\end{assumption}
These assumptions are standard and similar conditions have been used in \cite{liang2025low} in the context of DDIM. \blue{In practice, one may adopt higher-order velocity/score matching framework in \cite{lu2022maximum} for controlling the error terms above. However, it is well known that estimating higher-order derivatives is a difficult problem in nonparametric statistics \cite{nonpara_fan2018local, nonpara_tsybakov2009introduction}. }
\subsection{Convergence rate of rectified flow}
 \blue{Our convergence guarantees for the RF sampler \eqref{eq: emp-ode-disc} depend crucially on a carefully constructed end-point adapted time schedule. We first make the following assumptions on the time schedule.
 \begin{assumption}[Endpoint concentration]
     \label{assumption: assumption of time schedule}
     There exists small enough $\delta, h 
     \in (0,1)$ such that $t_{N-1}\le 1 - \delta$, and 
     \[
\max_{1\le i\le N-2}
\left\{
\frac{\eta_i}{t_i}\mathbbm{1}\{t_i\le 1/2\}
+
\frac{\eta_i}{1-t_{i+1}}\mathbbm{1}\{t_i>1/2\}
\right\}
\le h.
\]
 \end{assumption}
 The first condition simply guarantees that the second-to-last time step stays sufficiently far from 1. The second condition guarantees that the time points are tightly concentrated near the endpoints $t =0$ and $t = 1$. This brings us to our first main result of this section.
\begin{theorem}
\label{thm: RF - general convergence of Eulerian scheme}
   Let Assumptions \ref{assumption: intrinsic dimension}-\ref{assumption: higher-order approximation error} hold, with covering resolution satisfying \eqref{eq: covering resolution compatibility}, and write $\mu_1 = \bbE_{X \sim p_1}[X]$. Let $\{Y_{t_i}\}_{i\ge0}$ denote the RF updates of \eqref{eq: emp-ode-disc} with time schedule satisfying Assumption \ref{assumption: assumption of time schedule}.
   For sufficiently small $h$ (depending on $k$ and $N$), we have
   \begin{equation}
   \label{eq: RF general convergence Euler scheme}
    \begin{aligned}
    \TV(p_{X_{t_{N-1}}}, p_{Y_{t_{N-1}}})
    & \lesssim (h+Nh^2)k\log N+Nh^2\meanevelJone^2+Nh^2\meanevel^2 + Nh  (\meanevel + \meanevelJone + \meanevelJtwo + \meanevelH)\\
     &+    \eta_0^2 \bbE \Norm{X_1 - \mu_1}_2^2
     +  \eta_0^2 (\evel_0)^2    +  \eta_0^2 (\evelJone_0)^2 +  (\evel_0 + \evelJone_0 + \evelJtwo_0 + \evelH_0)\\  
     & + \frac{h+Nh^2}{\delta^2N^{10}}+\frac{h^3}{\delta^6N^9}  .
\end{aligned}
\end{equation}
\end{theorem}

\textbf{Adaptation to low-dimensionality and accelerated convergence.} The weighted covariance-trace sum contributes an additional initial boundary term of order $h k\log N$. Assuming $(h+Nh^2)\log N=o(1)$, the upper bound \eqref{eq: RF general convergence Euler scheme} scales linearly only in $k$, indicating that the RF sampler adapts to the low-dimensional structure of the target distribution. For the geometric schedule below, $h\lesssim Nh^2$ and $h\asymp\log(1/\delta)/N$. Thus, assuming perfect knowledge of the velocity drift (i.e. $\widehat v_t = v_t$ and the higher approximation terms of \Cref{assumption: higher-order approximation error} are all zero), the RF sampler \eqref{eq: emp-ode-disc} needs $\tilde O(k/\epsilon)$ iterations to achieve $\epsilon$-approximate sampling.
Empirical validation of the above scaling can be found in Appendix \ref{sec: empirical validation of upper bounds}. If the intrinsic dimension $k \ll d$, then the RF sampler automatically accelerates without any prior knowledge of the low-dimensionality of $p_1$. In the worst case, the moment term $\bbE \Norm{X_1 - \mu_1}_2^2$ can scale as $O(d)$. However, if $\eta_0 = O(1/d)$, the moment term in the upper bound can be atmost $O(1/d)$.
 }

\paragraph{Example of time schedule.}
The sufficient leading-term condition for low-dimensional (low-d) adaptation in the general bound is $(h+Nh^2)\log N=o(1)$. If we choose the linear uniform schedule, i.e., $t_i=i/N$, then $\eta_1/t_1=1$, so the small-step condition in Theorem \ref{thm: RF - general convergence of Eulerian scheme} is not satisfied. We discuss a detailed recipe for constructing an admissible time schedule in Appendix \ref{sec:time-schedule admissibility} and point out several examples in Table \ref{tab:time-schedules} that includes well known log-uniform, and EDM style time schedules. The general bound requires control of the endpoint step ratios through $h$, together with the displayed approximation-error and remainder terms. In this section, we adopt one such transparent example for clarity.
For given, user-specified values $h\in (0,1)$ and $\delta \ge 0$, we consider the non-uniform time partition 
\begin{equation}
\label{eq: gemoetric time discretization}
    \begin{aligned}
      & \text{
    1. $t_0 = 0, t_1 =\delta, t_j = (1+h)t_{j-1}$ for $2\le j \le N/2$,
    }\\ 
    & \text{
    2. $t_N = 1, t_{N-1} = 1- \delta, t_j = 1 - (1+h)(1-t_{j+1})$
    }\quad \text{ for $N/2< j \le N-2$}.
    \end{aligned}
\end{equation}
 We set $h$ such that $t_{N/2} = \delta (1+h)^{(N-2)/2} = 1/2$. Using $N-2 > N/2$ (for $N>4$), we arrive at 
\begin{equation}
\label{eq: value of h}
\frac{h}{2} \le \log(1+h) = \frac{2\log (\frac{1}{2 \delta})}{N-2} \Rightarrow  h \le \frac{8 \log (\frac{1}{2 \delta})}{N}.
\end{equation}
 We remark that the proposed time schedule is practicable, as it does not depend on any properties of $p_1$.
Typically, $\delta$ is chosen to be small, as
this give rise to a U-shaped time discretization as shown in Figure \ref{fig: hist of sde timegrid}(a). Furthermore, the schedule enjoys certain properties (\cf~Lemma \ref{lemma: properties of time-steps}) that ensure $h+Nh^2=O(\log^2(1/\delta)/N)$. This type of choice is not unusual \cite{lee2024improvingtrainingrectifiedflows}, as it yields more accurate estimates  $\widehat v_t$ during training when $t \approx 0$ and $t \approx 1$. Prior empirical work by 
non-uniform time discretizations have also been considered by \cite{huang2024denoising, li2024adapting, liang2025low} in the context DDPM and DDIM to establish low-dimensional adaptation of the models. It is important to emphasize that the specific form of the time schedule \eqref{eq: gemoetric time discretization}, while conforming to best empirical practices, is theoretically grounded, and it was derived directly as a result of our technical analysis of the RF dynamics (see Step 7 of the proof of Theorem \ref{thm: RF - convergence of Eulerian scheme}, Appendix \ref{app: RF convergence}). We now present the convergence result that establishes low-d adaptation under \eqref{eq: gemoetric time discretization}. 
\begin{theorem}
\label{thm: RF - convergence of Eulerian scheme}
   Let Assumptions \ref{assumption: intrinsic dimension}-\ref{assumption: higher-order approximation error} hold, with covering resolution satisfying \eqref{eq: covering resolution compatibility}, and write $\mu_1 = \bbE_{X \sim p_1}[X]$. Let $\{Y_{t_i}\}_{i\ge0}$ denote the RF updates of \eqref{eq: emp-ode-disc} with time schedule \eqref{eq: gemoetric time discretization}. For $\delta=1/(N\vee d)$, suppose $N$ is sufficiently large so that $h\le1/4$ and $16C_3kh\log N<1$, where $C_3$ is the constant in \eqref{eq: trace_Cov bound}. Then, we have
   \begin{equation}
   \label{eq: RF convergence Euler scheme}
    \begin{aligned}
    \TV(p_{X_{t_{N-1}}}, p_{Y_{t_{N-1}}})
    & \lesssim   \frac{k \log^3 (\frac{1}{\delta})  }{N} +   \frac{\log^2 (\frac{1}{\delta})}{N}   \meanevelJone^2 +  \frac{\log^2 (\frac{1}{\delta})}{N} \meanevel^2 + \log \br{1/\delta} (\meanevel + \meanevelJone + \meanevelJtwo + \meanevelH) \\
     &+    \delta^2 \bbE \Norm{X_1 - \mu_1}_2^2
     +  \delta^2 (\evel_0)^2    +  \delta^2 (\evelJone_0)^2 +  (\evel_0 + \evelJone_0 + \evelJtwo_0 + \evelH_0) \\
     &\quad + \frac{\log(1/\delta)+\log^2(1/\delta)}{\delta^2N^{11}}+\frac{\log^3(1/\delta)}{\delta^6N^{12}}.
\end{aligned}
\end{equation}
\end{theorem}

\textbf{Guarantee on perturbed data.} The \TV~guarantee in Theorem \ref{thm: RF - convergence of Eulerian scheme} is on the penultimate update $Y_{t_{N-1}}$ of \eqref{eq: emp-ode-disc}, and not on $Y_{t_N} = Y_1$. It is worthwhile to mention that $\TV(p_{X_1}, p_{Y_{t_N}})$ might not be a useful or even meaningful quantity, as $X_1$ and $Y_1$  may have  supports of different dimensions. For example, $X_1$ could be supported on a low-dimensional space, whereas $Y_{t_N}$ would typically have full dimensional support, thereby rendering $\TV (X_1, Y_{t_N}) =1 $. However, note that for any $\delta > 0$, $X_{\delta} = (1 - \delta) X_1 + \delta X_0$ has full-dimensional  support due to the slight Gaussian perturbation. Therefore, the \TV~distance in Theorem \ref{thm: RF - convergence of Eulerian scheme} is well defined.

\textbf{Higher-order approximation error terms.} The upper bound in Theorem \ref{thm: RF - convergence of Eulerian scheme} is heavily dependent on the higher-order approximation errors of $\widehat v_t$. These terms are required to be small as deterministic samplers are unable to self-correct their trajectories. In addition, at $t = 0$, the true drift $v_0(x) = \mu_1 - x$ can be estimated by $\widehat v_0(x) = n^{-1} \sum_{i \in [n]} {X_1^{(i)}} - x$, where $\{X_1^{(i)}\}_{i\in [n]} \overset{i.i.d.}{\sim } p_1$. With this choice of $\widehat v_0$, we have $\evelJone_0 = \evelJtwo_0 = \evelH_0 = 0$, and $\evel_0 = \sqrt{\bbE \Norm{X_1 - \mu_1}_2^2/n} $.

\subsection{Convergence rate of STOC-RF}
In this section, we design an SDE sampler that also adapts to the low-dimensionality of $p_1$ without the restrictive Assumption \ref{assumption: higher-order approximation error}. The key idea here is to leverage the equivalence between DDPM and STOC-RF (\cf~Figure \ref{fig:SL-unified}) and exploit the low-dimensional adaptivity of DDPM samplers shown in \cite{liang2025low}. In this regard, Proposition \ref{prop: RF-SDE and backwrd process are equivalent} and \eqref{eq: DDPM and RF equivalent} provides the crucial time-change map to facilitate a simple reduction argument --- the  map $
\tau \in [0, \infty)\mapsto t(\tau) := \frac{\sqrt{\omega_\tau}}{\sqrt{\omega_\tau} + \sqrt{1 - \omega_\tau}}
$ ensures that the forward process $\bc{Y'_\tau}_{\tau\ge 0}$ in $\eqref{eq: forward process}$ and the STOC-RF process $\{\tilde Z_t\}_{t\in[0,1]}$ enjoy $\frac{Y^\prime_\tau}{\sqrt{\omega_\tau}} \overset{d}{=} \frac{\tilde Z_{t(\tau)}}{t(\tau)}$. 

In practice, we simulate the forward process by discretizing $\tau$ to \textit{positive integers} in $\{1, \ldots, N\}$ and setting specific values for $\omega_\tau$. Following \cite{liang2025low}, we let $\omega_\tau = \prod_{j =1}^\tau \alpha_j$ with $\alpha_j$'s defined implicitly as 
\begin{equation}
    \label{eq: DDPM time scheduling}
    \begin{aligned}
        & \beta_1 := 1 - \alpha_1 = 1/N^{c_0}, \beta_{\tau+1} := 1 - \alpha_{\tau+1} = \frac{c_1 \log N}{N} \min \bc{\beta_1 \br{1 + \frac{c_1 \log N}{N}}^\tau, 1}, 
    \end{aligned}
\end{equation}
for a user-specified choice of the parameters $c_0, c_1 >0$.
Similar choices also appear in \cite{potaptchik2024linear, li2024sharp, li2024adapting, huang2024denoising} to analyze the convergence properties of DDPM. 
The discretized DDPM sampler then takes the form 
\begin{equation}
    \label{eq: DDPM sampler liang2025}
    \begin{aligned}
    & \hat Y^\prime_{\tau -1} = \frac{1}{\sqrt{\alpha_\tau}} \bc{ \hat Y^\prime_\tau + \delta_\tau \widehat s_{Y^\prime_\tau}(\hat Y^\prime_\tau) + \nu_\tau \xi_\tau}, \hat Y^\prime_N \sim \cN(0, I_d)\\
    & \delta_\tau = 1 - \alpha_\tau,  \nu_\tau = \sqrt{\frac{(\alpha_\tau - \omega_\tau)(1 - \alpha_\tau)}{1 - \omega_\tau}}, \tau   = N , N-1, \ldots, 2, \quad \{\xi_\tau\}_{\tau \ge 1} \overset{i.i.d.}{\sim} \cN(0,I_d),
    \end{aligned}
\end{equation}
where $ \widehat s_{Y^\prime_\tau}$ is an approximate score of $Y^\prime_\tau$.

 Next, it remains to convert \eqref{eq: DDPM sampler liang2025} to a STOC-RF sampler using the equivalence between $Y^\prime_\tau$ and $\tilde Z_{t(\tau)}$.  To this end, we let $\sigma_t^2 = (1 -t)^2 + t^2$ for any $t \in [0,1]$, so that 
 \begin{equation*}
 \sigma_{t(\tau)}^2 = \frac{1}{(\sqrt{\omega(\tau)} + \sqrt{1 - \omega(\tau)})^2} = \frac{t^2(\tau)}{\omega_\tau}.
 \end{equation*}
 Therefore, we have $Y^\prime _\tau \overset{d}{ = } \tilde Z_{t(\tau)}/ \sigma_{t(\tau)} $, and simple change of variable yields $s_{Y^\prime_\tau}(y) = \sigma_{t(\tau)} s_{t(\tau)}(\sigma_{t(\tau)} y)$. Additionally, if we define $t_i := t(N - i)$ for $i = 0 , 1, \ldots, N-1$, and $R_i = t_i/ \sigma_{t_i}$,  we arrive at an equivalent reformulation of \eqref{eq: DDPM sampler liang2025} for RF sampling (\cf~Appendix \ref{app: Stoc-RF convergence}) as described below: For $i = 0, \ldots, N-2$ update as
 \begin{equation}
     \label{eq: stoc-RF sampler non-Eulerian}
     \begin{aligned}
      &\frac{Y_{t_{i+1}}}{\sigma_{t_{i+1}}} = \frac{R_{i+1}}{R_i} \bc{ \frac{Y_{t_i}}{\sigma_{t_i}} + \eta_i \sigma_{t_i} \widehat s_{t_i}(Y_{t_i}) + \sqrt{\psi_i} W_{t_i}}, \;  \frac{Y_{t_0}}{\sigma_{t_0}} \sim N(0,  I_d), \; \{W_{t_i}\}_{i \ge 0}\overset{i.i.d.}{\sim} N(0,I_d),
     \end{aligned}
 \end{equation}
 where $ \eta_i = 1- \frac{R_i^2}{R_{i+1}^2},   \psi_i = \frac{R_i^2}{R_{i+1}^2} \frac{ 1 - R_{i+1}^2}{1  - R_i^2} \br{1 - \frac{R_i^2}{R_{i+1}^2}}$, and $\widehat s_{t_i}(x) = \frac{t_i \widehat v_{t_i}(x) - x}{1 - t_i}$. Here $\eta_i$ denotes the stochastic sampler coefficient; the time increment remains $t_{i+1}-t_i$. See Appendix \ref{app: Stoc-RF convergence} for details. Lastly, we impose an assumption on the approximation error of $\widehat s_{t_i}$.
 \begin{assumption}
     \label{assumption: score approximation error stoc-RF}
     Define $(\escore_i)^2:= \bbE \Norm{s_{t_i}(X_{t_i}) - \widehat{s}_{t_i}(X_{t_i})}_2^2$. Then, 
    $\frac{1}{N} \sum_{i =0}^{N-2} (\escore_i)^2 \le \meanescore^2.$
 \end{assumption}
The above error terms are typically hard to control due to singularities near $t\approx 1$. However, higher weights can be assigned to training loss near end-points (\cf~\cite{lee2024improvingtrainingrectifiedflows}) to achieve training-time stability. Next,
the equivalence between \eqref{eq: DDPM sampler liang2025} and \eqref{eq: stoc-RF sampler non-Eulerian} leads to the main convergence guarantee for STOC-RF.
 \begin{theorem}
     \label{thm: stoc-RF sampler non-Eulerian}
     Under Assumption \ref{assumption: intrinsic dimension}, \ref{assumption: bounded support} and \ref{assumption: score approximation error stoc-RF}, the stoc-RF sampler \eqref{eq: stoc-RF sampler non-Eulerian} satisfies 
     \begin{align*}
     & \TV(Y_{t_{N-1}}, X_{t_{N-1}}) \lesssim \frac{k \log^3 N}{N} +  \meanescore \sqrt{\log N}.
     \end{align*}
 \end{theorem}
 Note that $t_{N-1} = \frac{\sqrt{\omega_1}}{\sqrt{\omega_1} + \sqrt{1 - \omega_1}} \approx 1$. Therefore, $\Law(X_{t_{N-1}}) \approx p_1$. In comparison with Theorem \ref{thm: RF - convergence of Eulerian scheme}, STOC-RF sampler does not require any conditions on higher-order derivatives (\Cref{assumption: higher-order approximation error}) of the score or velocity field. That is, STOC-RF sampler can achieve similar performance under less stringent requirements. This is primarily due to the Gaussian perturbation term in \eqref{eq: stoc-RF sampler non-Eulerian} which acts as a corrector step. The proof is in Appendix \ref{app: Stoc-RF convergence}.
\begin{figure*}[t]
  \centering
  \begin{subfigure}[b]{0.24\textwidth}
    \centering
    \includegraphics[width=\textwidth]{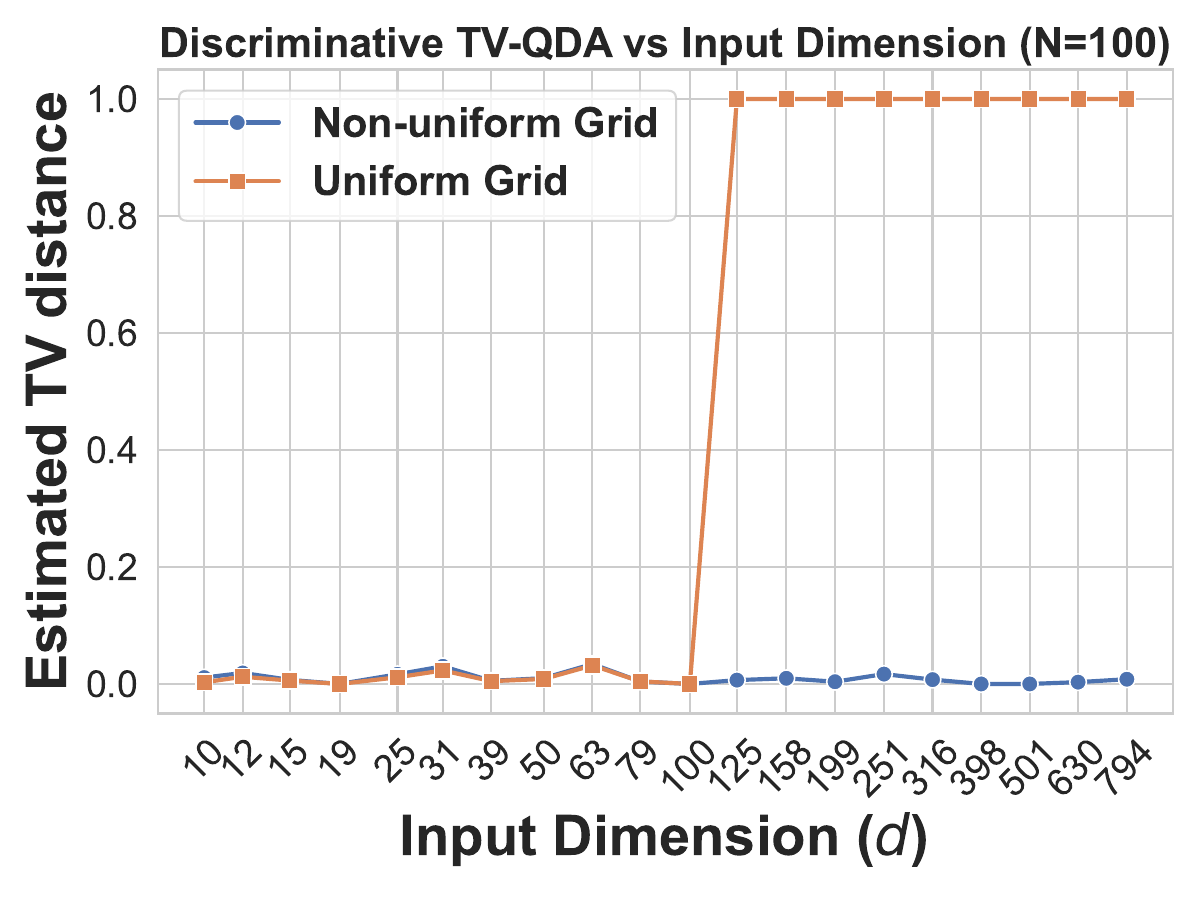}
    \caption{}
  \end{subfigure}\hfill
  \begin{subfigure}[b]{0.24\textwidth}
    \centering
    \includegraphics[width=\textwidth]{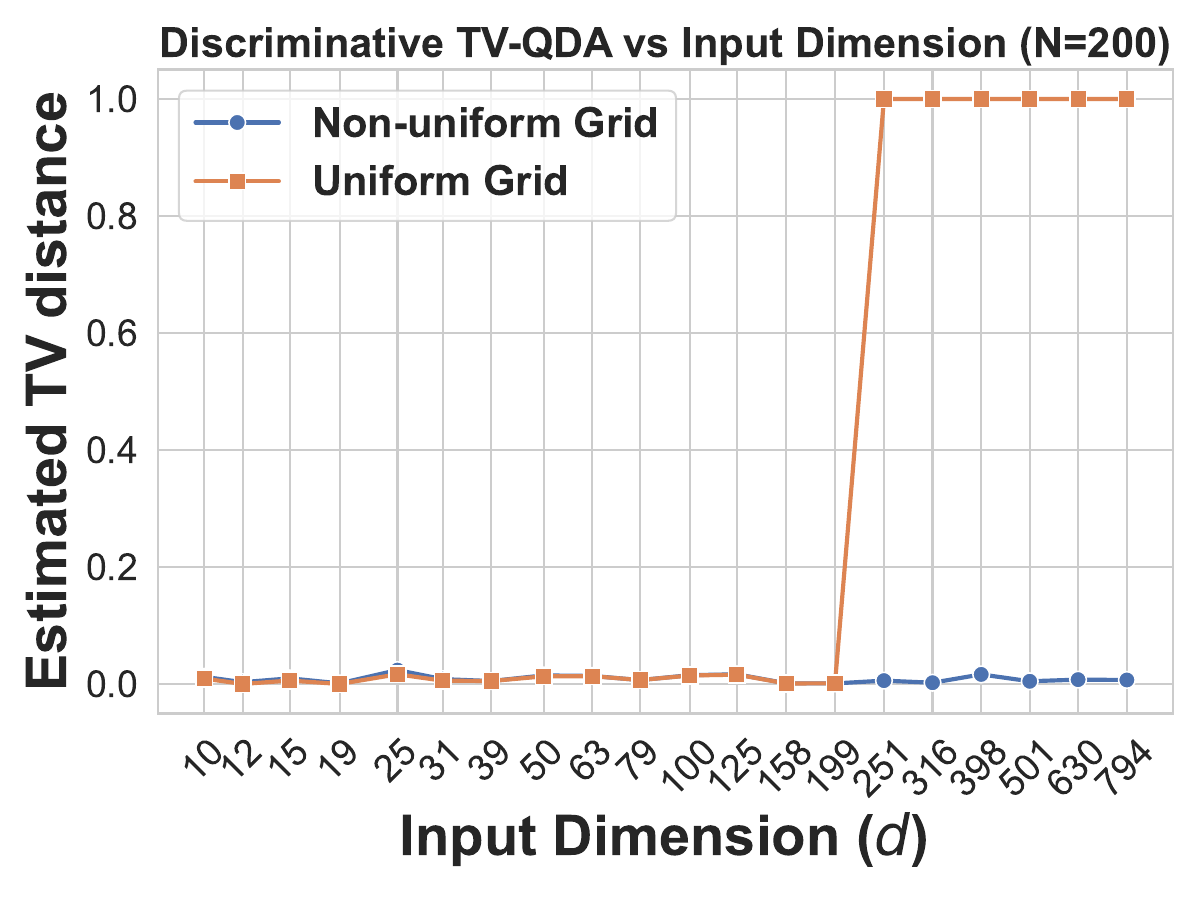}
    \caption{}
  \end{subfigure}\hfill
  \begin{subfigure}[b]{0.24\textwidth}
    \centering
    \includegraphics[width=\textwidth]{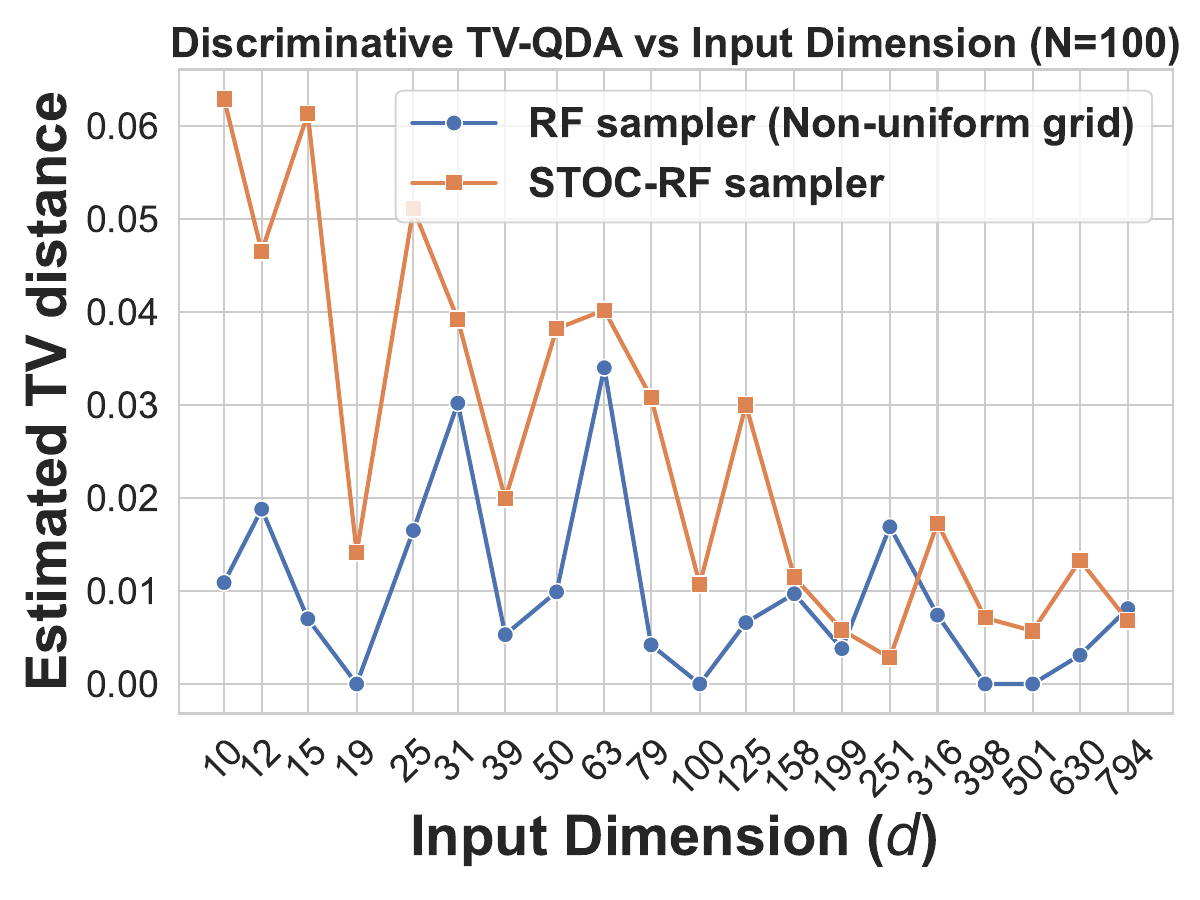}
    \caption{}
    \end{subfigure}\hfill
  \begin{subfigure}[b]{0.24\textwidth}
    \centering
    \includegraphics[width=\textwidth]{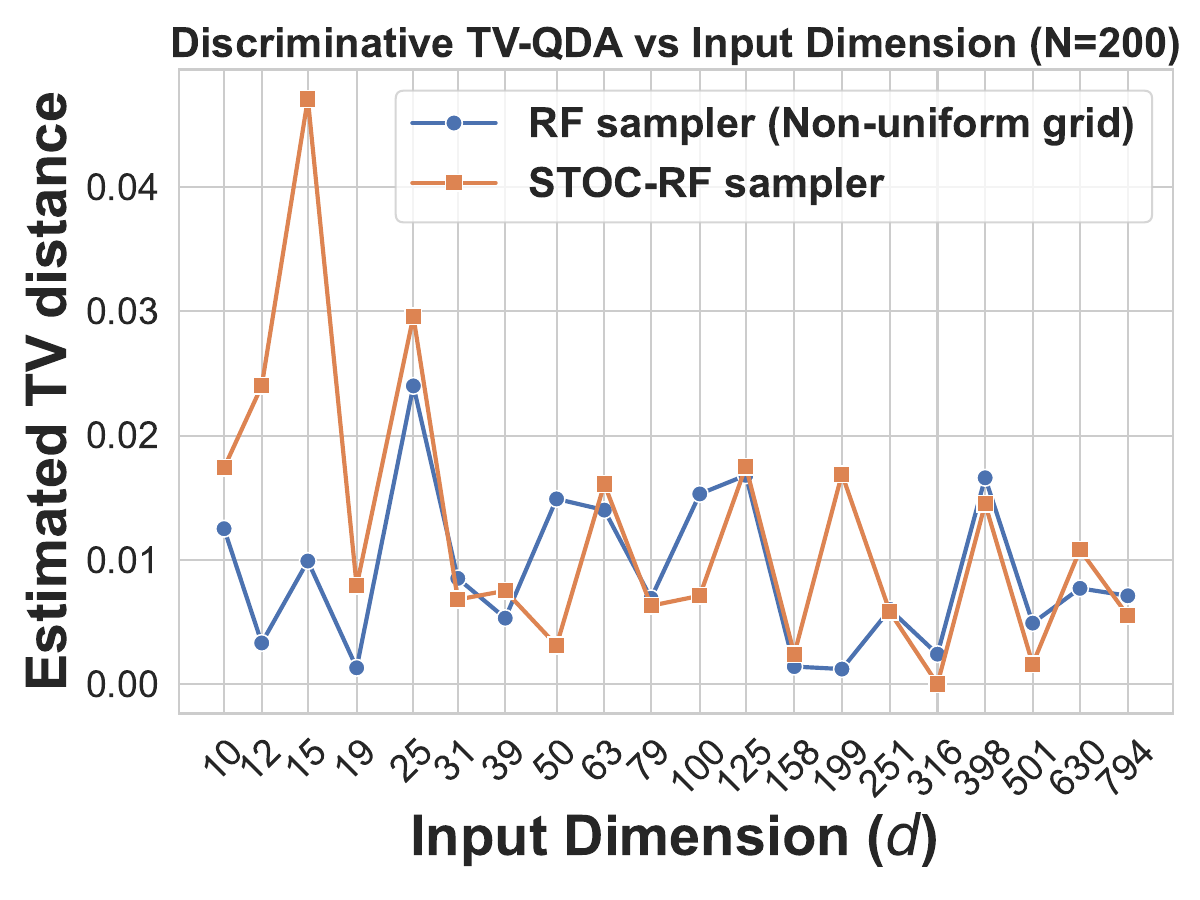}
    \caption{}
    \end{subfigure}
  \caption{(a)-(b) \textit{Estimated} \TV~ distance between \textit{blurred} low-rank Gaussian distribution and the generated samples by the RF sampler \eqref{eq: emp-ode-disc} under uniform and non-uniform time-grid under varying $d$. (c)-(d) \textit{Estimated} \TV~ distance between \textit{blurred} low-rank Gaussian distribution and the generated samples by the RF and STOC-RF samplers under uniform and non-uniform time-grid for varying $d$.}
  \label{fig: TV distance comparison low-rank Gaussian}
\end{figure*}
\begin{remark}
\label{remark: new RF sampler}
Using the equivalence relation \eqref{eq: DDPM and RF equivalent}, it is also possible to directly relate the deterministic RF sampler to the DDIM sampler studied by \citep[Theorem 1]{liang2025low} (Appendix \ref{app: new RF sampler}) via the same time change transformation considered in this section. The resulting time schedule also follows a U-shaped distribution and the sampler will also adapts to the low-dimensionality of $p_1$.
\end{remark}

\section{Experiments}
\blfootnote{Code for the experiments is available at \url{https://github.com/roysaptaumich/Low-dim-RF}.}
\textbf{Synthetic data.}
We set the intrinsic dimension $k =8$, and set our target distribution to be $p_1 = \cN(\mu, \Sigma)$, where $\mu = (8, 8, \ldots, 8)^\top \in \bbR^d$  and $\Sigma$ is a block-diagonal matrix with blocks $I_{k}$ and $ \boldsymbol{0}_{(d-k) \times (d-k)}$. 
Tweedie's formula (Section \ref{sec: Rf}) yields
\begin{equation}
\label{eq: low-dim velocity field}
\begin{aligned}
 v_{t_i}(x) 
 = \begin{bmatrix}
    \frac{2t_i-1}{\sigma_{t_i}^2} I_k & \boldsymbol{0}_{k \times (d-k)}\\
    \boldsymbol{0}_{(d-k) \times k} & - \frac{1}{1-t_i}I_{d-k}
\end{bmatrix} x + \begin{bmatrix}
    \frac{1-t_i}{\sigma_{t_i}^2} I_k  & \boldsymbol{0}_{k \times (d-k)}\\
    \boldsymbol{0}_{(d-k) \times k} & \frac{1}{1 - t_i} I_{d-k}
\end{bmatrix} \mu.
\end{aligned}
\end{equation}


We use \eqref{eq: low-dim velocity field} for generating new samples from $p_1$. We vary $d$ between 10 and 800, and set $N \in \bc{100, 200}$. We use both a non-uniform time-grid \eqref{eq: gemoetric time discretization} (with $\delta = \min\{1/N, 1/d\}$) and a uniform time-grid for the generation task. We generate 2000 samples using \eqref{eq: gemoetric time discretization} and measure the \TV~distance with respect to another batch of 2000 samples independently generated from the slightly blurred target distribution $\Law((1 - \delta) X_1 + \delta Z)$, where $Z \sim \cN(0, I_d)$. We use the discriminative approach proposed in \cite{tao2024discriminative} to estimate the \TV~distance between the distributions via 10 independent Monte Carlo rounds. \blue{For brevity, we refer to the non-uniform grid \eqref{eq: gemoetric time discretization} as the geometric (Geom) grid  in this section.} Figure \ref{fig: TV distance comparison low-rank Gaussian}(a)-(b) show that \blue{Geom grid} has better adaptation to the low-dimensional structure of $p_1$ for large $d$. \blue{We observe a sharp phase transition in \TV~distance across time schedules. This may reflect the sensitivity of \TV~to support mismatch, which becomes more pronounced as discretization errors accumulate in higher dimensions. The geometric schedule appears to mitigate this effect by better controlling errors near the endpoints.}
For the SDE sampler \eqref{eq: stoc-RF sampler non-Eulerian}, we also compare its \TV~convergence with that of the deterministic sampler \eqref{eq: emp-ode-disc} with the non-uniform time grid \eqref{eq: gemoetric time discretization}. Figure \ref{fig: TV distance comparison low-rank Gaussian}(c)-(d) show that both samplers have similar performance across varying dimensions.
\begin{table*}[t]
\centering

\begin{subtable}[t]{0.54\textwidth}
\centering
\caption{T2I experiments.}
\label{tab:t2i-results}
\vspace{2mm}
\resizebox{\linewidth}{!}{
\begin{tabular}{lrrr}
\toprule
\rowcolor[HTML]{C5E1A5}
Method & Mean CLIP ($\uparrow$) &
Mean Pick ($\uparrow$) &
Mean CER ($\downarrow$) \\
\midrule
\textbf{ODE + Lin}   & 34.1130 & 0.2295 & 1.0904 \\
\textbf{ODE + Geom}  & \textbf{34.1650} & 0.2301 & 1.1339 \\
\textbf{SDE (vanilla)} & 29.3820 & 0.2085 & 1.6738 \\
\textbf{STOC-RF}     & 33.8270 & \textbf{0.2307} &
\textbf{0.9718} \\
\bottomrule
\end{tabular}
}
\end{subtable}
\hfill
\begin{subtable}[t]{0.44\textwidth}
\centering
\caption{MNIST experiments.}
\label{tab:time-grid-results}
\vspace{2mm}
\resizebox{\linewidth}{!}{
\begin{tabular}{lcc}
\toprule
\rowcolor[HTML]{C5E1A5}
Time Grid & FID ($\downarrow$) & MMD ($\downarrow$) \\
\midrule
Linear
& $6.373 \pm 0.733$
& $0.002026 \pm 0.000003$ \\
\addlinespace[6pt]
Geometric
& $\mathbf{6.135 \pm 0.727}$
& $\mathbf{0.002023 \pm 0.000004}$ \\
\addlinespace[6pt]
Logit
& $6.358 \pm 1.143$
& $0.002024 \pm 0.000004$ \\
\bottomrule
\end{tabular}
}
\end{subtable}

\caption{\small
Comparison of sampling methods and time discretization schedules.
(a) CLIP, Pick, and CER metrics for text-to-image (T2I) generation.
(b) FID and MMD metrics on MNIST under different time grids.
Best results are highlighted in bold.
}
\label{tab:results}

\end{table*}

\paragraph{MNIST experiments.}~We train an RF model on MNIST and compare sampling time schedules using FID and MMD averaged over 10 independent runs. Table~\ref{tab:results}(b) shows that the geometric schedule achieves the lowest FID and MMD by modest margins, suggesting competitive performance on real-world data exhibiting low-dimensional structure.

\paragraph{T2I experiments.}~We perform experiments using Flux \citep{flux2024, labs2025flux1kontextflowmatching} to demonstrate the usefulness of the U-shaped time schedule \eqref{eq: gemoetric time discretization}. 
We collect 6 different text prompts from \cite{liu2023character, hu2025amo} and apply different samplers for 5 independent rounds and report the mean score in Table \ref{tab:results}(a). 
\blue{We see that generation quality of RF is competitive under the Geom grid \eqref{eq: gemoetric time discretization}. In terms of mean CLIP and Pick score, Geom grid is slightly better than the linear/ uniform grid. On the contrary, linear grid shows better performance in mean CER score.} Figure \ref{fig: Flux SDE/RF experiments} compares the generation quality of STOC-RF sampler \eqref{eq: stoc-RF sampler non-Eulerian} and the vanilla DDPM (uniform grid) and ODE sampler. Table \ref{tab:results}(a) shows that the generation quality is slightly better (on average) for STOC-RF in terms of mean PickScore and CER. The limitations of deterministic samplers have also been observed in prior works \cite{failure_de2024schrodinger, failure_reu2026gradient, failure_shi2023diffusion}. More experimental details and studies are deferred to \Cref{app: T2I exp}--\ref{app:quality-speed}.

\begin{figure}[!ht]
    \centering
    
    \begin{subfigure}[t]{0.3\textwidth}
        \centering
        \includegraphics[width=\textwidth]{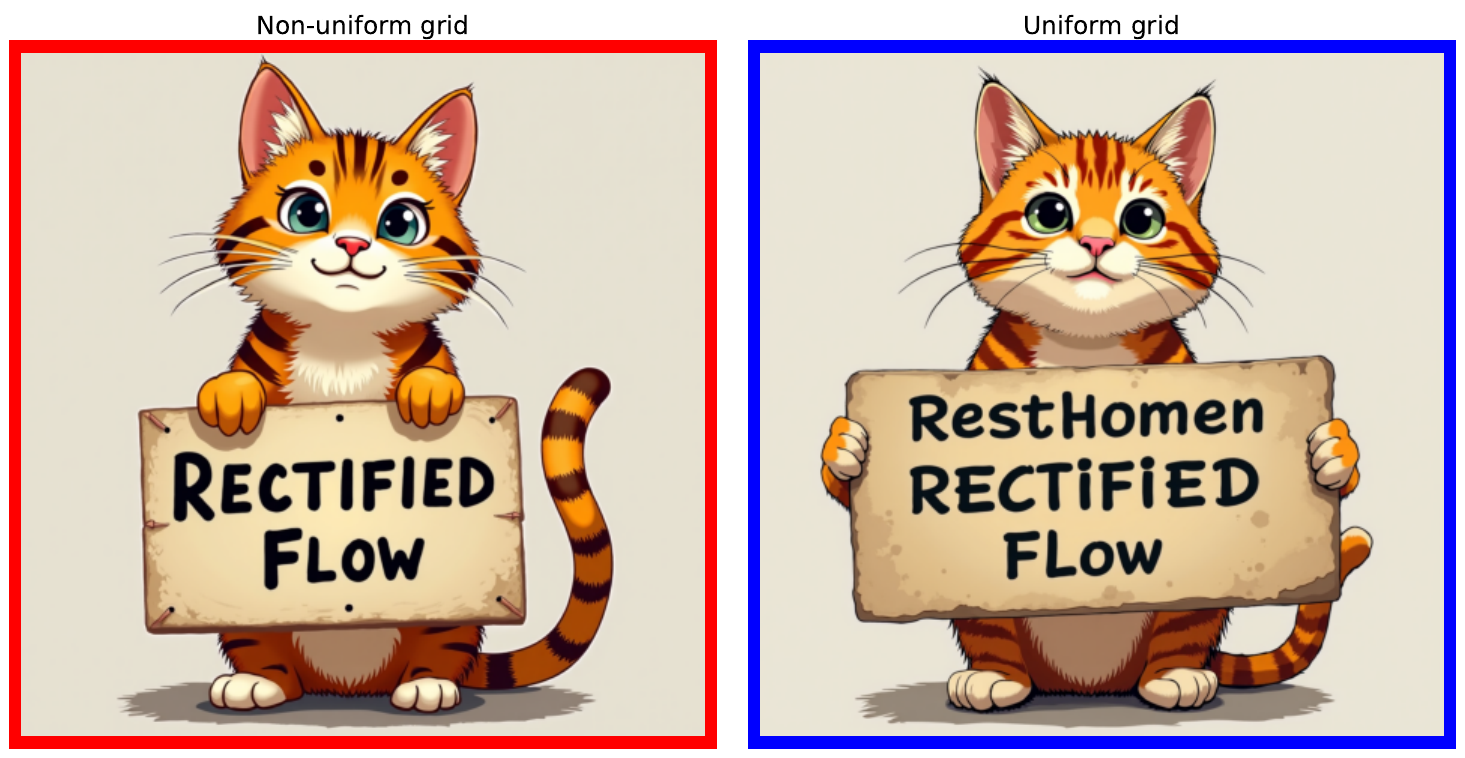}
        \caption{\scriptsize\sf \textbf{Prompt:} A photo of a cat holding a sign \orange{Rectified flow}. (\underline{Less hallucination}, $N = 50$)}
    \end{subfigure}\hfill
    \begin{subfigure}[t]{0.3\textwidth}
        \centering
        \includegraphics[width=\textwidth]{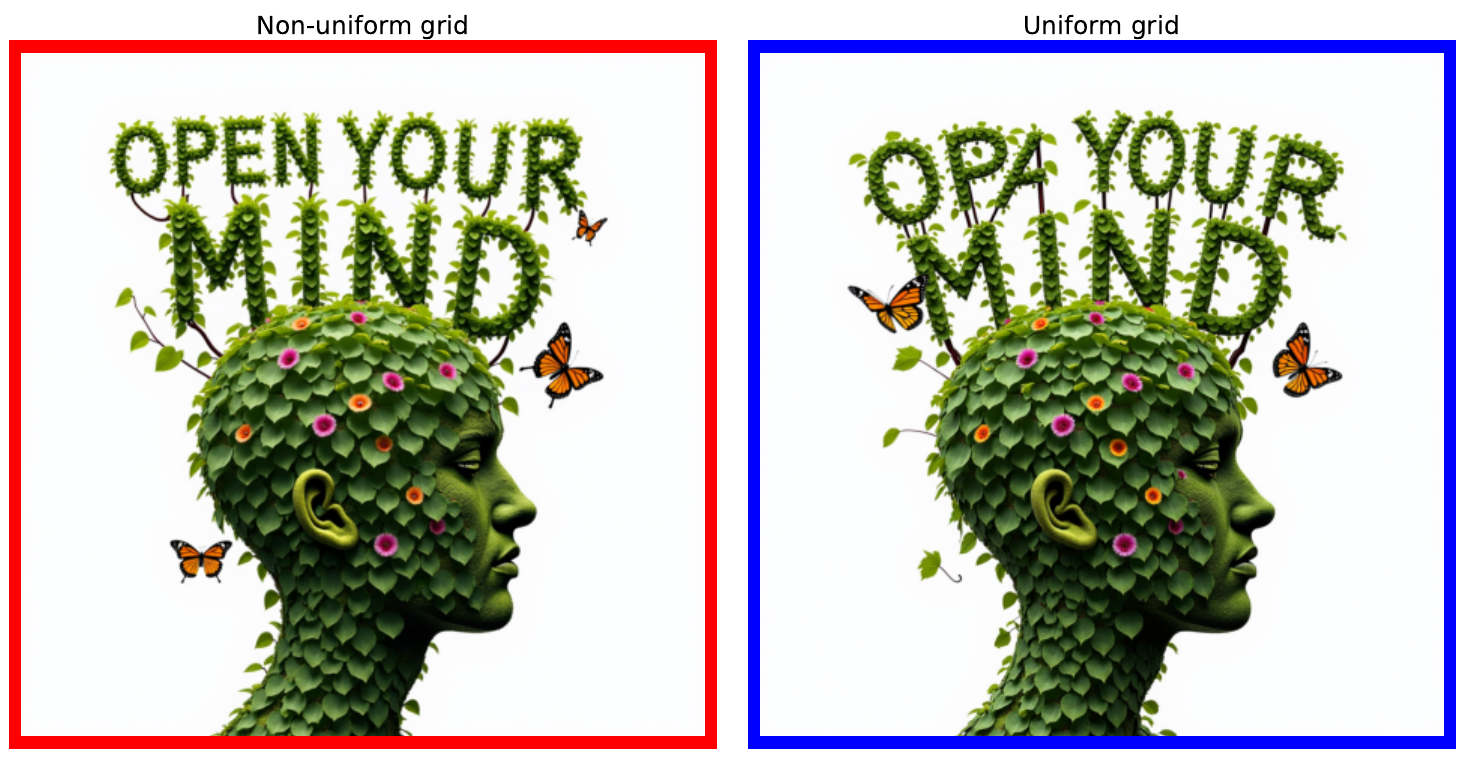}
        \caption{\scriptsize\sf \textbf{Prompt:} Grape vines in the shape of text \orange{`open your mind’} sprouting out of a head with flowers and butterflies. DSLR photo. (\underline{Better text rendering}, $N = 50$)}
    \end{subfigure}\hfill
    \begin{subfigure}[t]{0.3\textwidth}
        \centering
        \includegraphics[width=\textwidth]{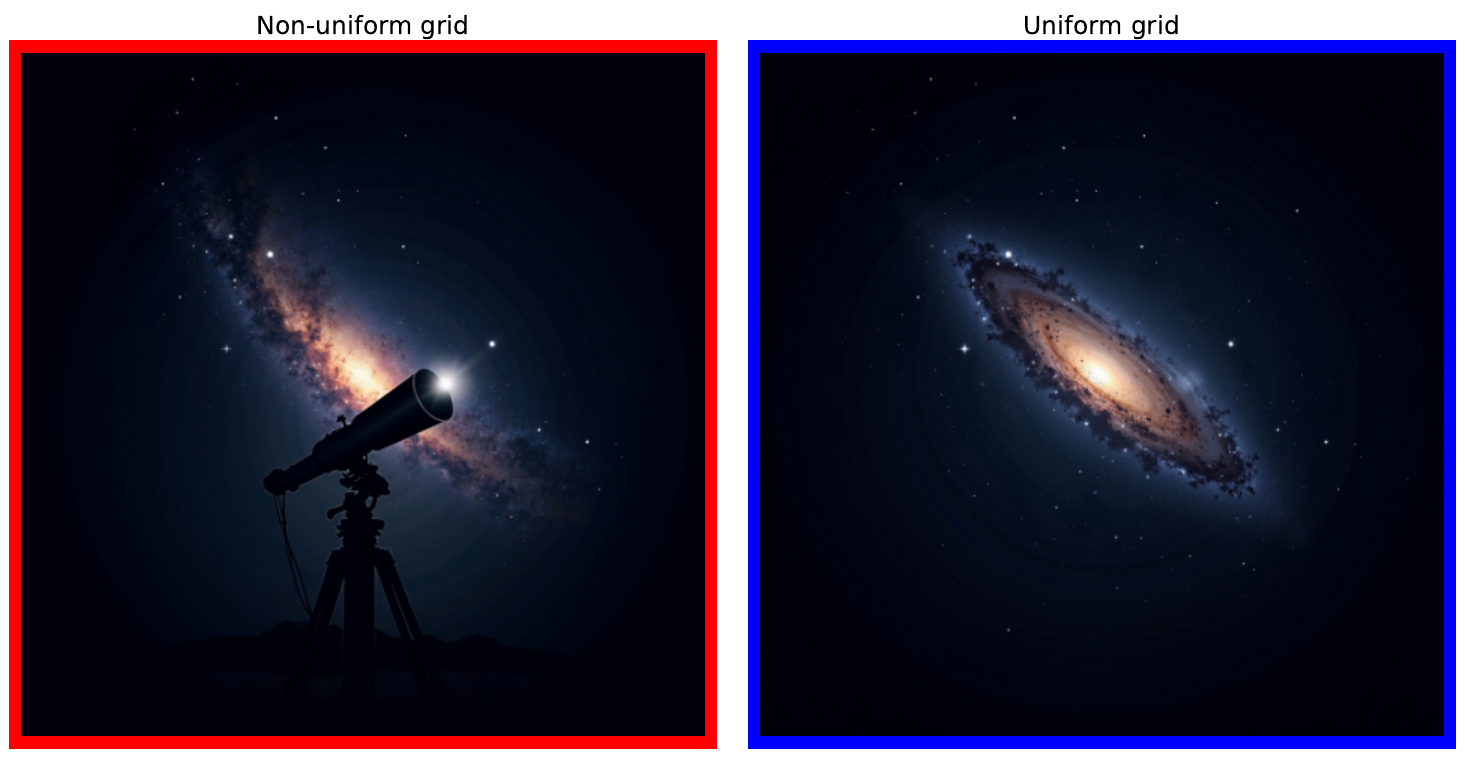}
        \caption{\scriptsize \sf \textbf{Prompt:} the \orange{hubble telescope} and the milky way, with the text \orange{``the universe is a mystery, but we are here to solve it"} (\underline{Better fidelity/ prompt adherence}, $N = 50$)}
    \end{subfigure}
    \vspace{1em} 
    \begin{subfigure}[t]{0.3\textwidth}
        \centering
        \includegraphics[width=\textwidth]{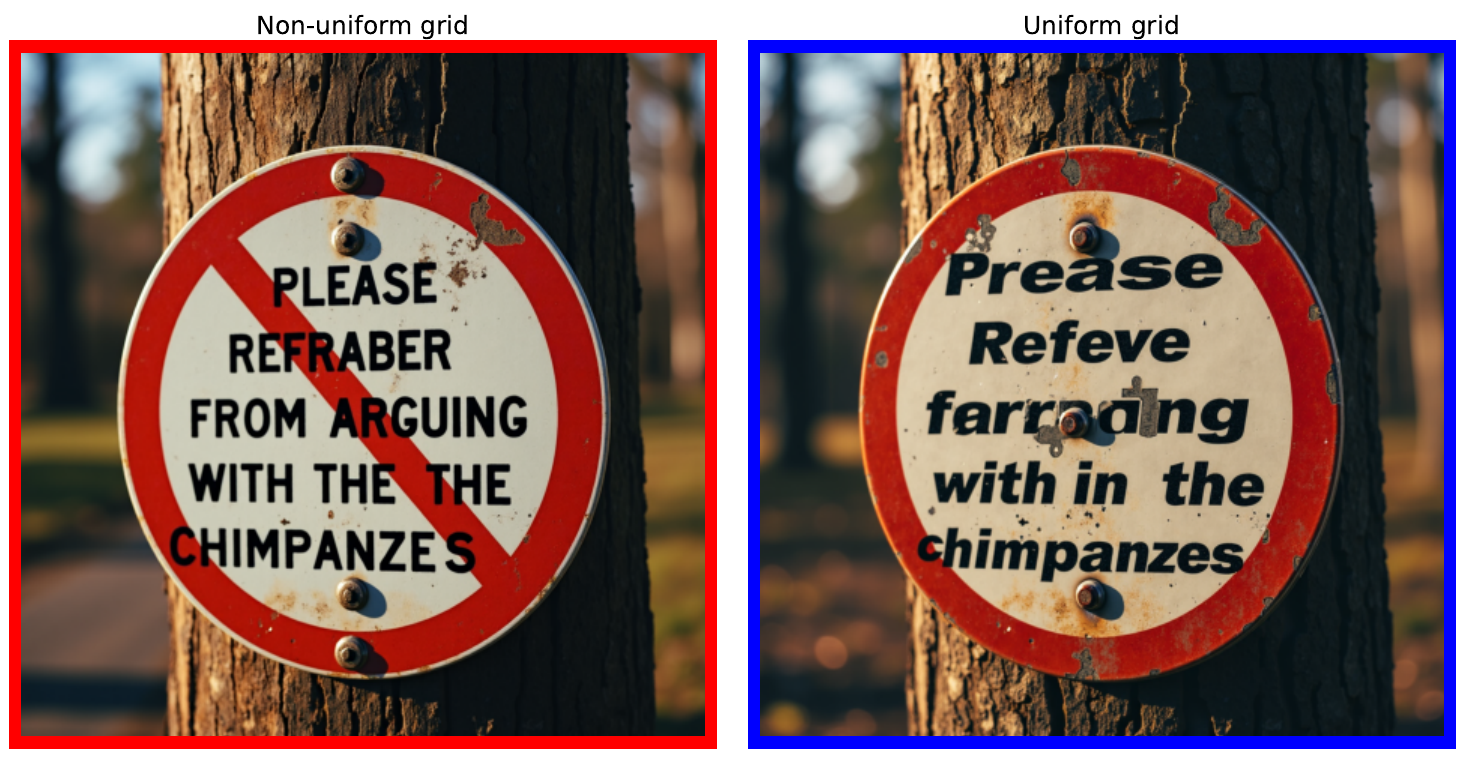}
        \caption{\scriptsize \sf \textbf{Prompt:} A sign that says \orange{"Please refrain from arguing with the
chimpanzees"}. (\underline{Minimal mistakes}, $N = 50$)}
        \label{fig:sub4}
    \end{subfigure}\hfill
    \begin{subfigure}[t]{0.3\textwidth}
        \centering
        \includegraphics[width=\textwidth]{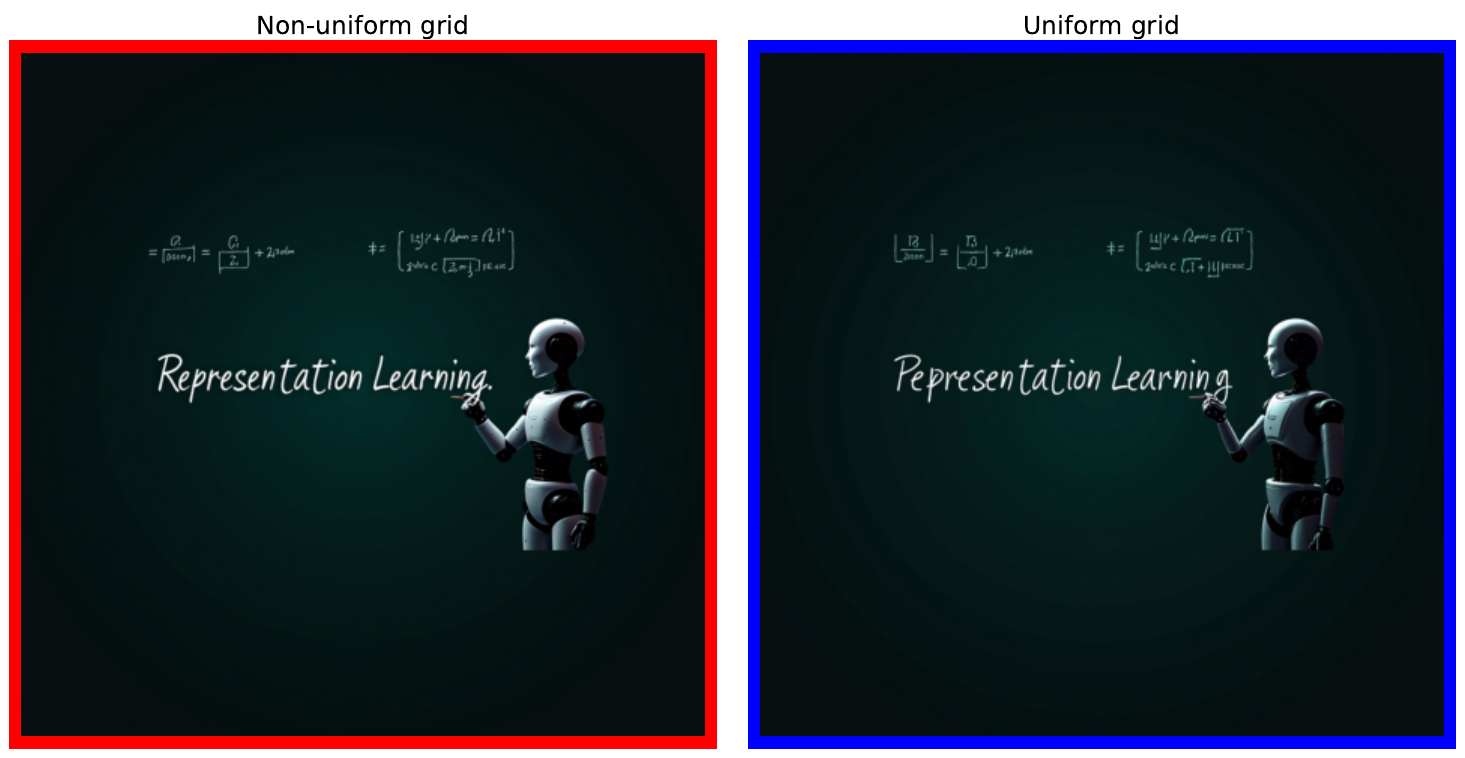}
        \caption{\scriptsize \sf \textbf{Prompt:} Photo of a robot lecturer writing the words \orange{`Representation Learning'} in cursive on a blackboard, with math formulas and diagrams. (\underline{Correct spelling}, $N = 50$)}
        \label{fig:sub5}
    \end{subfigure}\hfill 
    \begin{subfigure}[t]{0.3\textwidth}
        \centering
        \includegraphics[width=\textwidth]{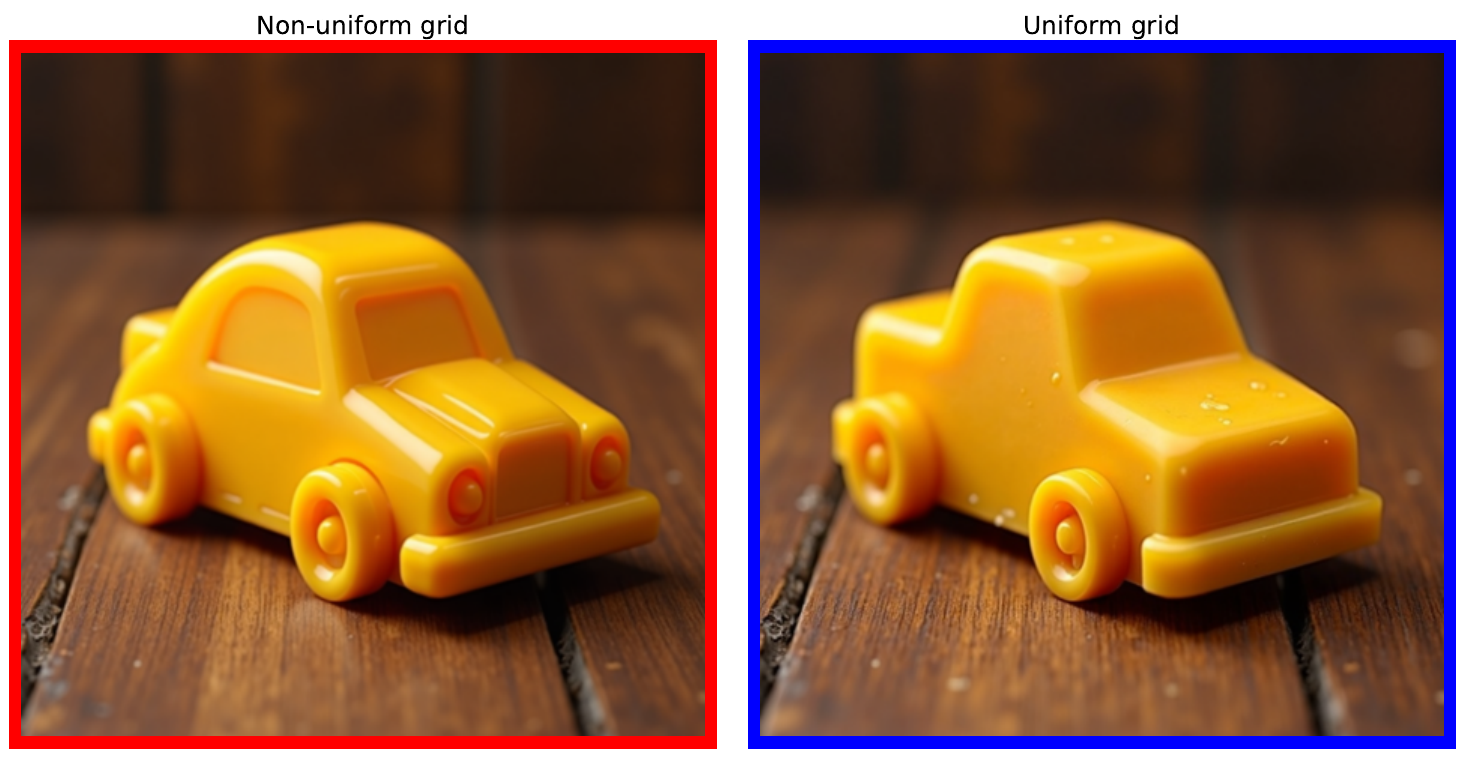}
        \caption{\scriptsize \sf \textbf{Prompt:} A \orange{car made of cheese} on a wooden table. (\underline{Semantic completeness}, $N = 80$)}
        \label{fig:sub6}
    \end{subfigure}

    \caption{Generated images using Flux for specific prompts via Euler's scheme \eqref{eq: emp-ode-disc} under non-uniform time-grid \eqref{eq: gemoetric time discretization} (\red{Red frame}) and uniform time-grid ({\color{blue}Blue frame}).}
    \label{fig: Flux experiments}
\end{figure}
\section{Conclusions}
\vspace{-5pt}
\blue{We show that U-shaped time schedules allow pretrained RF samplers to accelerate convergence by adapting to low-dimensional target structure under the stated assumptions. Endpoint step ratios, controlled by $h$, govern the leading discretization error and provide a principled criterion for comparing schedules beyond empirical performance. The benefits arise without additional reflow iterations.} Through stochastic localization, we unify DDPM and stochastic RF, and establish low-dimensional adaptation for STOC-RF under weaker approximation assumptions.

\blue{\textbf{Limitations.} Our guarantees require regularity and estimation assumptions; practical performance depends on the time grid and velocity/score accuracy. Dimension-dependent constants, schedule optimality, broader empirical evaluation, and tighter finite-step guarantees remain subjects for future theoretical and empirical investigation.}

\section*{Acknowledgment and Disclosures}

We gratefully acknowledge NSF grants CCF-2019844, CCF-2505865, and 2217069.

\paragraph{AI use statement.} The main technical ideas in the paper emerged from the authors. GPT-5.6 Sol assisted with writing and editing this paper in conjunction with the authors. The authors take full responsibility for the paper's content.

\bibliography{arxiv_bib}
\bibliographystyle{plainnat}

\newpage
\appendix
\section*{Appendix}
The appendix is organized as follows.
\begin{enumerate}
    \item Section~\ref{sec: more details on SL} discusses more deatils on the SL process.
    \item Section~\ref{sec:appendix:exp} contains additional experimental results.
    \item Section~\ref{app: proof of main results} contains detailed proofs from Sections~\ref{sec: main results} and the description of the RF sampler.
    \item Section~\ref{app: SL lemmas} contains detailed proofs of results about Section~\ref{sec: SL}.
    \item Finally, Section~\ref{app: auxiliary results} contains all other helpful lemmas and their proofs.
\end{enumerate}

\section{More details on Stochastic localization}
\label{sec: more details on SL}
In this section, we present a unified framework for studying and analyzing the stochastic DDPM sampler and the deterministic RF sampler.

\paragraph{Stochastic localization (SL).} 
SL  is a probabilistic technique to study the properties of high-dimensional distributions due to \cite{eldan2013thin,eldan2020taming} that provides a powerful, unifying framework for understanding, analyzing, and improving score-based generative models and sampling algorithms. See  \cite{el2022information}, \cite{montanari2023sampling} and the survey by \cite{shi2025perspectivesstochasticlocalization}. To sample from the target data distribution $p_1$, SL focuses on an appropriate measure-valued stochastic process $\{\nu_s\}_{s\ge 0}$ such that $\nu_s$ ``localizes'' towards $p_1$ as $s \to \infty$. A straightforward way to formulate this  process is by letting, for any $s \geq 0$, $\nu_s := \Law (X_1 \mid U_s)$, where $X_1 \sim p_1$ and
\begin{equation*}
    U_s = s X_1 + B_s,
\end{equation*}
with $\{B_s\}_{s \ge 0}$  a standard Brownian motion in $\bbR^d$ \citep{montanari2023sampling}, independent of $X_1$. Thus, $\nu_s$ is a random measure depending on $U_s$. Note that $U_s/s \to X_1$ almost surely as $s \to \infty$, i.e., $\nu_s$ does indeed localize around $p_1$. 


While enlightening, the SL process, as stated, does not yield a sampler from $p_1$, as it uses $X_1 \sim p_1$ in the definition of $U_s$. Therefore, it is instead convenient to consider the alternative process $\{\tilde U_s\}_{s \geq 0}$ defined as the solution of the  SDE
\begin{equation}
    \label{eq: SL SDE}
    d\tilde U_s = a_s(\tilde U_s)\,ds + d \tilde B_s,
\end{equation}
where $\tilde B_s$ is an independent Brownian motion and, for any $s \geq 0$,  $a_s(u) = \bbE[X_1 \mid U_s = u]$. By classical SDE results -- see \cite[Theorem 8.4.3 ]{oksendal2003stochastic} or \cite[Theorem 7.12 ]{liptser2013statistics} -- $\{U_s\}_{s\ge 0}$ and $\{\tilde U_s\}_{s\ge 0}$ are equivalent in law. Therefore, \eqref{eq: SL SDE} can be used for sampling from $p_1$ as long as we have access to the regression functions $\{ a_s(\cdot), s \geq 0\}$, which can be estimated from the data. We also define the posterior covariance as $\bA_s = \Cov(X_1 \mid U_s)$. The next lemma from \cite{eldan2020taming} allows us to control the growth of the expected posterior covariance.
\begin{lemma}
    \label{lemma: eldan ddt cov lemma}
    For all $s\ge 0$, $\frac{d}{ds} \bbE\bs{\bA_s} = - \bbE\bs{\bA_s^2}$.
\end{lemma}
This is a crucial result for controlling the discretization error of ODE and SDE samplers of RF. Crucially, \Cref{lemma: SL and RF are equivalent} yields the following Corollary to \Cref{lemma: eldan ddt cov lemma} which is used to control the discretization error of the RF sampler.
\begin{corollary}
    \label{cor: my ddt cov version}
    Let $\bSigma_t := \Cov(X_1\mid X_t)$ for $t \in [0,1)$. Then $\frac{d}{dt} \bbE (\bSigma_t) = - \frac{2t}{(1-t)^3} \bbE(\bSigma_t^2)$.
\end{corollary}
See \Cref{sec: RF covariance} for a complete proof. The corollary above is extremely useful for controlling the quadratic term involving the Frobenius norm of $\bSigma_t$ (\cf~\Cref{prop: Main cov bound via trace}) in the discretization error of the RF sampler.
One may use vanilla Cauchy-Schwarz inequality to bound $\bbE\left(\Norm{\bSigma_t}_F^2\right)$. However, this will lead to a quadratic dependence on $k$. \Cref{cor: my ddt cov version} provides much sharper control over the Frobenius norm by connecting it with the rate of change of the trace of $\bSigma_t$ that only scales linearly in $k$. We point the readers to \emph{Step 7} of the main proof of Theorem \ref{thm: RF - general convergence of Eulerian scheme} in Section \ref{sec: proof of general convergence}.

\paragraph{SL and Stochastic Interpolants.}
\label{sec: RF and deterministic interpolants}
Lemma  \ref{lemma: SL and RF are equivalent} and the equivalence result in \eqref{eq: DDPM and RF equivalent} can be extended to cover general stochastic interpolation \cite{albergo2023building, albergo2023stochastic} schemes beyond RF. 
In detail, instead of a linear interpolation, one can consider a general deterministic affine interpolation process $I_\theta = a_\theta X_1 + b_\theta X_0$, where $a_{(\cdot)}$ and $b_{(\cdot)}$ are smooth functions of $\theta \in [0,1]$, with $a_0 = b_1 = 0$, and $a_1 = b_0 = 1$. When $X_0 \sim N(0,I_d)$, the associated stochastic interpolant is the process $(\tilde I_\theta)$ at times $\theta>0$ for which $a_\theta\ne0$, where $\tilde I_\theta = a_\theta X_1 + a_\theta W_{r^2_{\theta}}$, with $r_\theta = b_\theta/a_\theta$ and $\bc{W_t}_{t\ge 0}$  a Brownian motion. It is easy to see that $\tilde I_\theta \overset{d}{=} I_\theta$, for each $\theta$. Moreover, each stochastic interpolant can be rescaled and time-transformed to return the SL process, in a sense made precise by the next result. 
\begin{proposition}
\label{prop: SL and general interpolant are equivalent}
  $\bc{\frac{\tilde I_{\theta(s)}}{a_{\theta(s)}}}_{s>0 } \overset{d}{=} \bc{\frac{U_s}{s}}_{s>0}$, for positive $s$ for which the time change $r^2_{\theta(s)}=1/s$ is defined.
\end{proposition}

In particular, in virtue of Lemma \ref{lemma: SL and RF are equivalent}, RF is equal in law to any interpolant process $I_\theta$ under the time change $t(\theta) = \frac{a_\theta}{a_\theta + b_\theta}$. This equivalence among stochastic interpolants is not new, as it has also been shown in \cite{Liu2024_FlowNotes}  by direct arguments. Connection between diffusion and flow matching has also been studied in \cite{albergo2023stochastic, albergo2023building,gao2025diffusionmeetsflow, ma2024sit, rout2025semantic} through different perspectives.
By leveraging the representation properties of the SL process, we establish this result in a more unified manner, which has the added benefit of directly applying to both stochastic interpolants and SDEs. The proof of \Cref{prop: SL and general interpolant are equivalent} is deferred to Appendix \ref{app: proof of SL and general interpolants equivalence}.

\section{Additional Experiments}\label{sec:appendix:exp}
Synthetic data experiments are conducted in standard Macbook with 8GB RAM and M2 processor. The T2I experiments have been done in the Lightning AI cloud with NVIDIA L40S GPU. 
\subsection{Checking tightness of the scaling of \TV~upperbounds}
\label{sec: empirical validation of upper bounds}
\subsection*{Experimental Setup and Findings}

\textbf{Setup:} We evaluated the ODE-geometric sampler across a grid of:
\begin{itemize}[leftmargin=*, itemsep=0pt]
    \item \textbf{Matrix dimensions:} $d \in \{100, 122, \dots, 300\}$ (10 values evenly spaced)
    \item \textbf{Ranks:} $k \in \{3, 7, 11, 15, 20\}$ (5 values evenly spaced)
    \item \textbf{Tolerance levels:} $TV_{tol} \in \{0.01, 0.03, 0.05, 0.07, 0.10\}$ (5 levels)
\end{itemize}
For each triple $(d, k, TV_{tol})$, we recorded $N(d, k)$ --- the \textbf{minimum number of discretization steps} required to achieve TV distance below $TV_{tol}$. We then compared the two candidate rates $k/N(d,k)$ and $d/N(d,k)$ as empirical proxies for $TV_{tol}$ across all $(d,k)$ pairs. To reduce noise, we binned each ratio across all $(d,k)$ pairs and computed the \textbf{mean $TV_{tol}$ per bin}, then fitted an OLS line to the binned means. We show the plots in Figure \ref{fig: bound tightness}.

\vspace{0.5em}
\noindent \textbf{Linearity of $k/N$:} The binned mean $TV_{tol}$ exhibits a \textbf{clear and monotone linear trend} with $k/N$, with OLS achieving $R^2 = 0.878$ and slope $= 0.0094$ ($p = 1.83 \times 10^{-3}$). The through-origin fit
\[
TV_{tol} \approx C \cdot \frac{k}{N}
\]
also has smaller OLS residuals per bin (std $= 0.0082$) compared to
\[
TV_{tol} \approx C \cdot \frac{d}{N}
\]
fit. See the Figures for OLS statistics.

\subsection*{Superiority of $k/N$ over $d/N$}
By contrast, the same analysis applied to $d/N$ yields substantially weaker results. The key comparisons are summarised below:

OLS slope for $d/N$ is \textbf{not statistically significant} ($p = 0.078 > 0.05$) - \textbf{no evidence} of a linear relationship between $d/N$ and $TV_{tol}$. In contrast, $k/N$ yields a highly significant slope with residual std more than \textbf{twice as small} (0.0082 vs 0.0190).

\vspace{0.5em}
\noindent \textbf{Tightness of the bound.} \\
For each ratio $r \in \{\frac{k}{N}, \frac{d}{N}\}$, we estimate the tightest valid constant
\[
C = \max_{d,k} \frac{TV_{tol}}{r},
\]
and report two diagnostics per bin:
\begin{itemize}[leftmargin=*]
    \item \textbf{Bound utilisation:} 
    \[ \frac{TV_{tol}}{Cr} \in (0, 1], \]
    where 1 indicates a perfectly tight bound.
    \item \textbf{Slack:}
    \[ \frac{Cr - TV_{tol}}{Cr} \in [0, 1), \]
\end{itemize}

The normalized gap between bound and observation: the $k/N$ bound achieves mean utilization \textbf{0.763}, while $d/N$ drops to \textbf{0.438}, making it roughly \textbf{1.74$\times$ looser}. Mean slack shows the same pattern: \textbf{0.237} for $k/N$ vs. \textbf{0.562} for $d/N$, meaning the $d/N$ bound leaves over twice the unused margin. 

\begin{figure}
    \centering
    \includegraphics[width=0.95\linewidth]{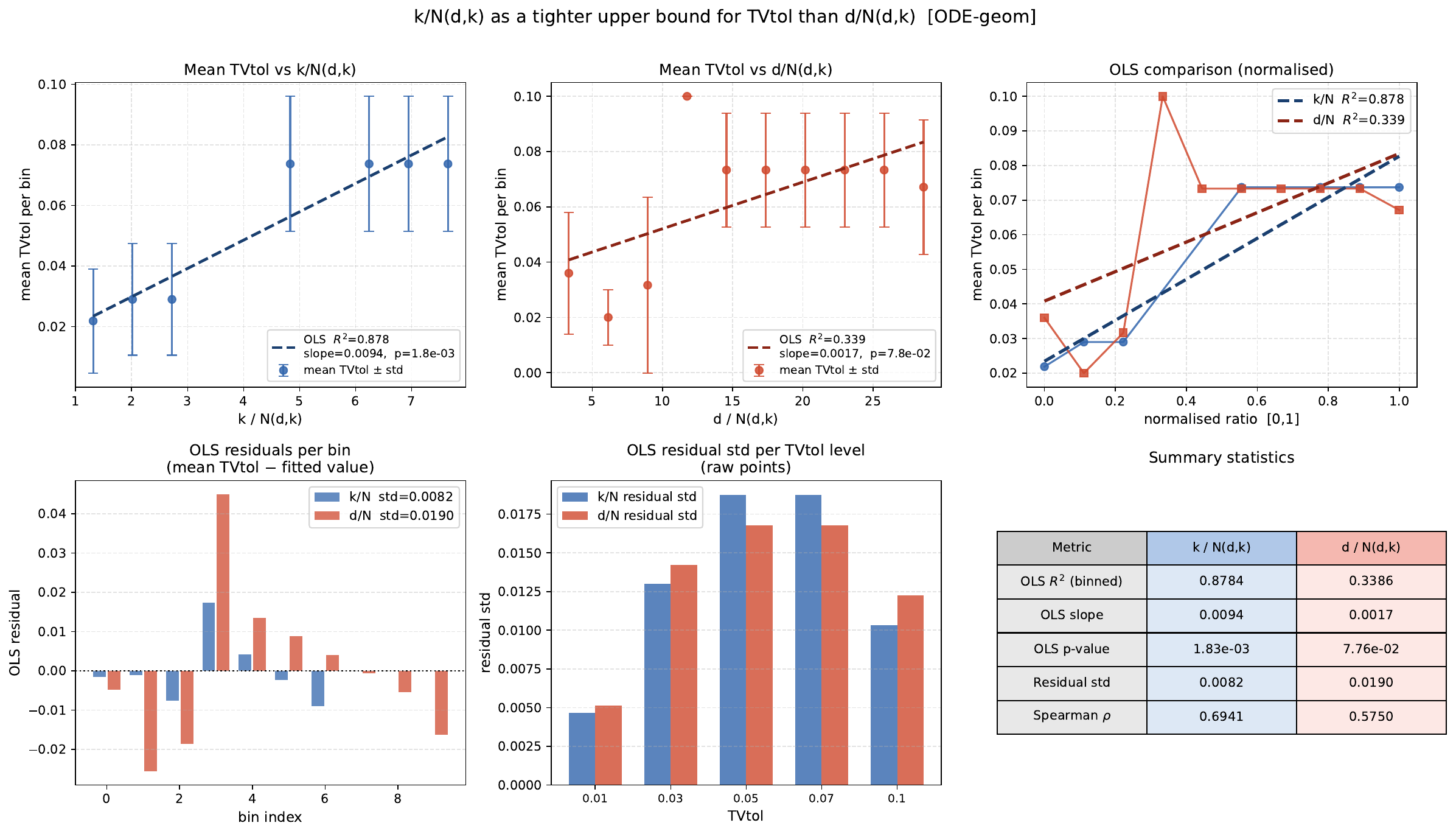}
    \includegraphics[width=0.95\linewidth]{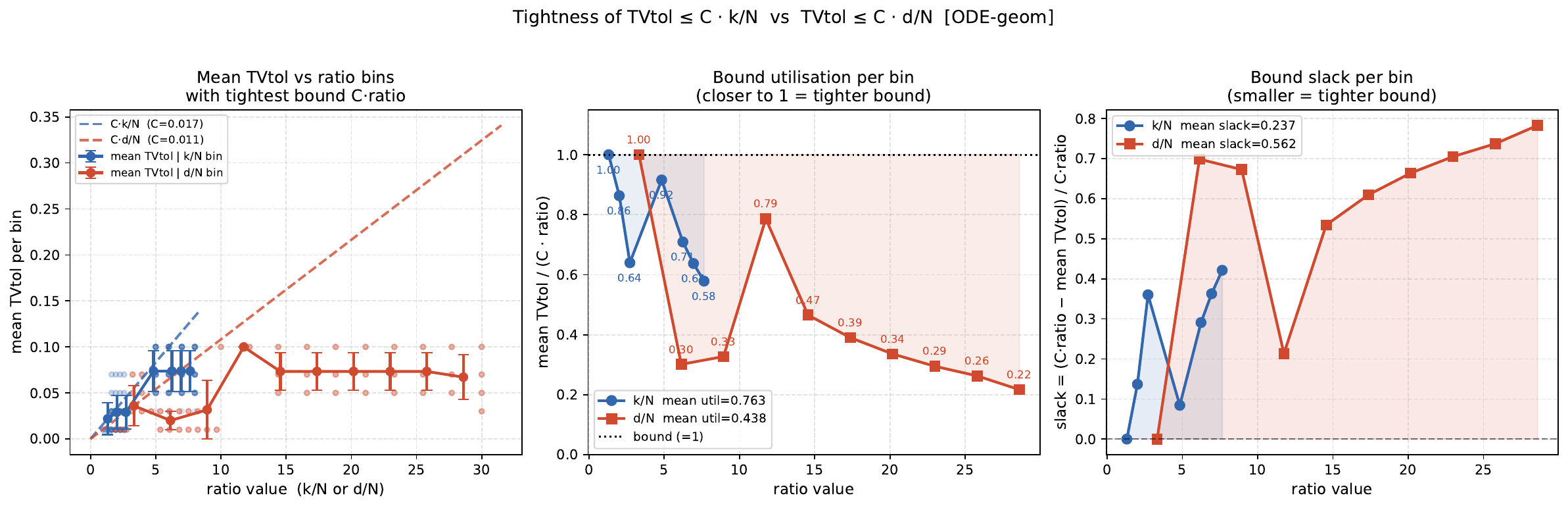}
    \caption{Empirical Validation of TV Distance Scaling Laws and Bound Utilisation.}
    \label{fig: bound tightness}
\end{figure}
\subsection{Comparison of RF sampler under uniform and non-uniform grid in T2I setting}
\label{app: T2I exp}
We have performed various experiments to demonstrate the usefulness of the U-shaped time schedule \eqref{eq: gemoetric time discretization} in the context of T2I. We use the Flux model \citep{flux2024, labs2025flux1kontextflowmatching} which is a pretrained RF model. For our experiments, we use the prompts designed by
\cite{liu2023character, hu2025amo} to test and assess the generation quality. The list is given below:
\begin{enumerate}
    
    \item \texttt{Photo of a robot lecturer writing the words ‘Representation Learning’ in cursive on a black-
board, with math formulas and diagrams. }

     \item \texttt{A close-up of a t-shirt with the phrase "I love AI".}

    \item \texttt{Grape vines in the shape of text `open your mind’ sprouting out of a head with flowers and butterflies. DSLR photo.}

    \item \texttt{A sign that says "Please refrain from arguing with the
chimpanzees"}

\item \texttt{A photo of a cat holding a sign Rectified flow.}

    \item \texttt{ car made of
cheese on a wooden table.}

     \item \texttt{the hubble telescope and the milky way, with the text ``the universe is a mystery, but we are here to solve it"}
\end{enumerate}
In these examples, we always set $\delta = 1/N$  for a specified choices of $N$ and report the mean CLIP, PickScores and CER scores in Table \ref{tab: extended table T2I} across 5 different rounds for each of the first 6 prompts above. In summary, from Table \ref{tab:results} we observe the following:
\begin{itemize}
    \item SDE (vanilla) is, on average, comparatively worse than other methods.
    \item STOC-RF is very competitive with other methods, both in terms of text rendering quality and text-image alignment.
    \item Among the ODE methods, non-uniform scheduling clearly improves or remains competitive in terms of CLIP score and PickScore. In terms of text rendering, ODE + uniform is slightly better than ODE + non-uniform.
\end{itemize}

\begin{table*}[t]
\centering
\caption{\small Comparison of generative sampling methods across
six T2I prompts from Appendix~\ref{app: T2I exp} using CLIP, PickScore, and CER.
Best results for each prompt are highlighted in bold.}
\label{tab: extended table T2I}
\small
\setlength{\tabcolsep}{12pt}
\begin{tabular}{lccc}
\toprule
\rowcolor[HTML]{C5E1A5}
Method & CLIP ($\uparrow$) &
PickScore ($\uparrow$) & CER ($\downarrow$) \\
\midrule

\multicolumn{4}{l}{\textit{Prompt: Robot Lecturer}} \\
ODE + Lin     & 31.420 & 0.225 & 2.574 \\
ODE + Geom    & 30.520 & 0.228 & 2.957 \\
SDE (vanilla) & \textbf{33.710} & 0.211 & 4.939 \\
STOC-RF       & 31.310 & \textbf{0.237} & \textbf{2.044} \\
\midrule

\multicolumn{4}{l}{\textit{Prompt: T-shirt (``I love AI'')}} \\
ODE + Lin     & 39.590 & 0.239 & 0.844 \\
ODE + Geom    & 40.910 & \textbf{0.241} & 0.711 \\
SDE (vanilla) & 27.510 & 0.183 & 1.400 \\
STOC-RF       & \textbf{41.760} & 0.238 & \textbf{0.689} \\
\midrule

\multicolumn{4}{l}{\textit{Prompt: Grape Vines}} \\
ODE + Lin     & 25.900 & \textbf{0.230} & 0.917 \\
ODE + Geom    & 25.460 & 0.229 & 0.901 \\
SDE (vanilla) & 25.830 & 0.221 & 0.944 \\
STOC-RF       & \textbf{26.310} & \textbf{0.230} & \textbf{0.889} \\
\midrule

\multicolumn{4}{l}{\textit{Prompt: Sign (``Chimpanzees'')}} \\
ODE + Lin     & 38.710 & \textbf{0.207} & 0.522 \\
ODE + Geom    & \textbf{38.830} & \textbf{0.207} & 0.528 \\
SDE (vanilla) & 24.430 & 0.172 & 0.976 \\
STOC-RF       & 38.470 & 0.206 & \textbf{0.508} \\
\midrule

\multicolumn{4}{l}{\textit{Prompt: Cat with sign}} \\
ODE + Lin     & \textbf{35.540} & \textbf{0.240} & \textbf{0.685} \\
ODE + Geom    & 35.450 & \textbf{0.240} & 0.706 \\
SDE (vanilla) & 32.460 & 0.227 & 0.783 \\
STOC-RF       & 35.440 & \textbf{0.240} & 0.702 \\
\midrule

\multicolumn{4}{l}{\textit{Prompt: Cheese Car}} \\
ODE + Lin     & 33.520 & 0.237 & 1.000 \\
ODE + Geom    & \textbf{33.820} & 0.235 & 1.000 \\
SDE (vanilla) & 32.350 & \textbf{0.238} & 1.000 \\
STOC-RF       & 29.670 & 0.234 & 1.000 \\
\bottomrule
\end{tabular}
\end{table*}

In Figure \ref{fig: Flux experiments} we see that generation quality of RF is superior under a non-uniform time-grid \eqref{eq: gemoetric time discretization} compared to uniform time-grid. In particular, the experiments indicate that a non-uniform grid helps to reduce hallucinations, enjoys better text rendering and achieves better semantic completeness.

We also compare the generation quality of STOC-RF sampler \eqref{eq: stoc-RF sampler non-Eulerian} and the vanilla DDPM sampler (uniform grid) in Figure \ref{fig: Flux SDE/RF experiments}. The generated plots show that the quality is generally better for STOC-RF sampler \eqref{eq: stoc-RF sampler non-Eulerian} compared to the RF sampler and the vanilla DDPM sampler (uniform grid).

\begin{figure}[h!]
    \centering
    
    \begin{subfigure}[t]{0.7\textwidth}
        \centering
        \includegraphics[width=\textwidth]{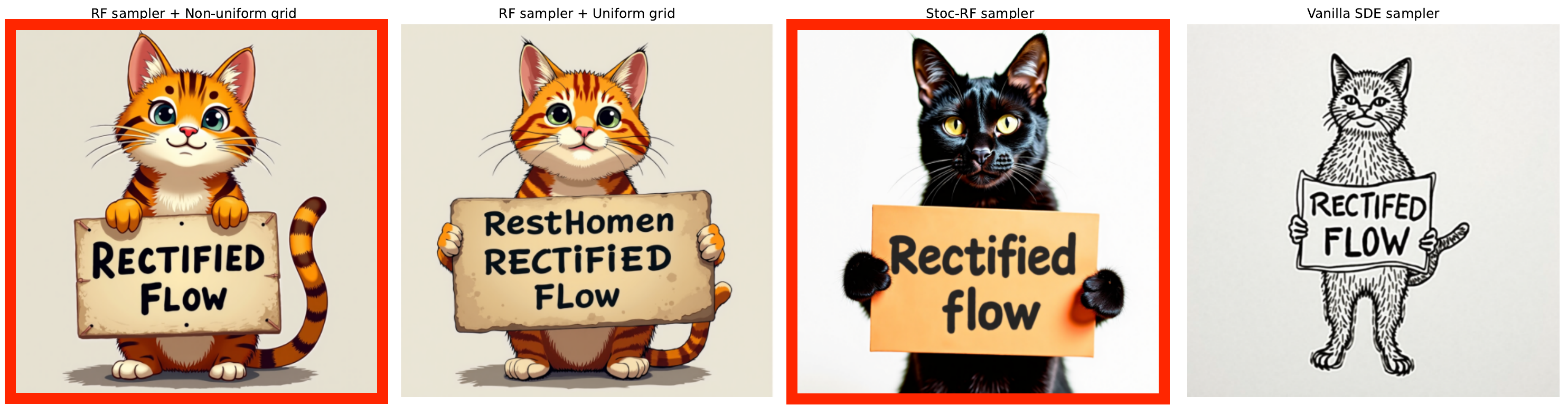}
        \caption{\sf \textbf{Prompt:} A photo of a cat holding a sign \orange{Rectified flow}. ($N = 50$)}
    \end{subfigure}\hfill
    \begin{subfigure}[t]{0.7\textwidth}
        \centering
        \includegraphics[width=\textwidth]{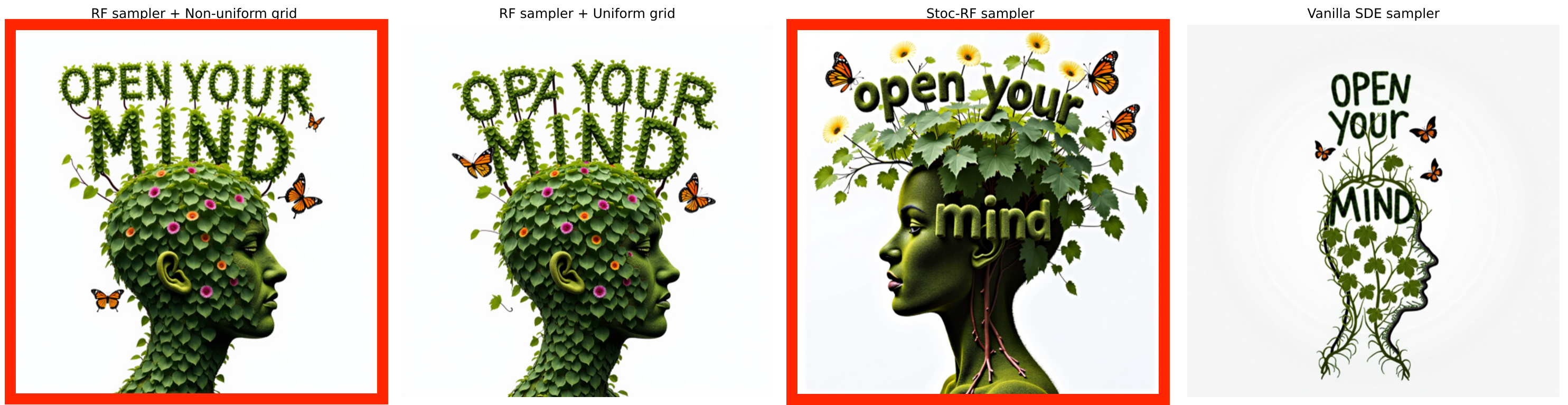}
        \caption{\sf \textbf{Prompt:} Grape vines in the shape of text \orange{`open your mind’} sprouting out of a head with flowers and butterflies. DSLR photo. ($N = 50$)}
    \end{subfigure}
    \vspace{1em}
    \begin{subfigure}[t]{0.7\textwidth}
        \centering
        \includegraphics[width=\textwidth]{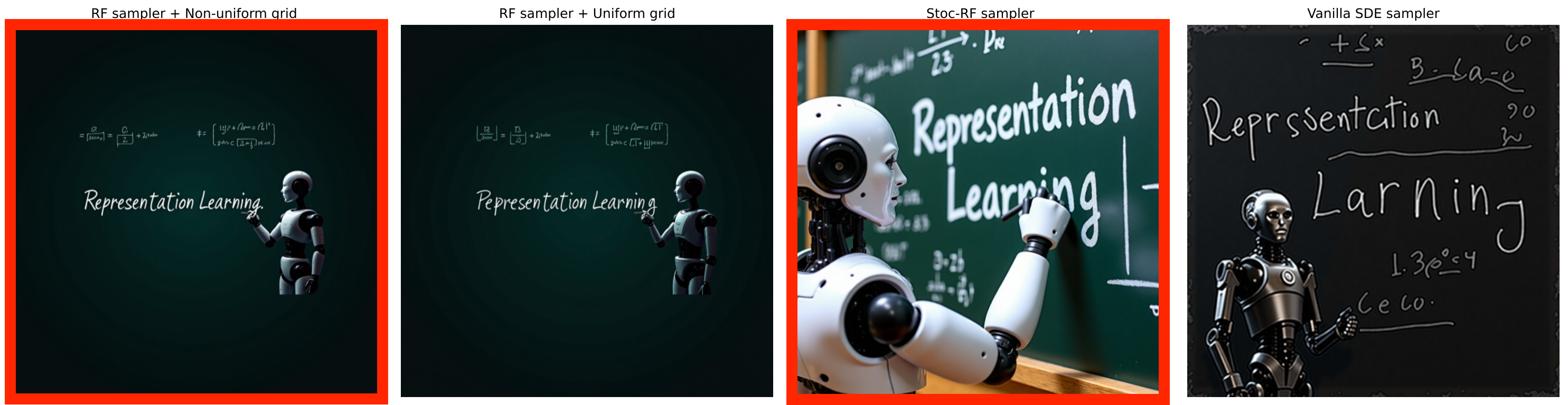}
        \caption{\sf \textbf{Prompt:} Photo of a robot lecturer writing the words \orange{`Representation Learning'} in cursive on a blackboard, with math formulas and diagrams. ($N = 50$)}
    \end{subfigure}
    \vspace{1em}
    \begin{subfigure}[t]{0.7\textwidth}
        \centering
        \includegraphics[width=\textwidth]{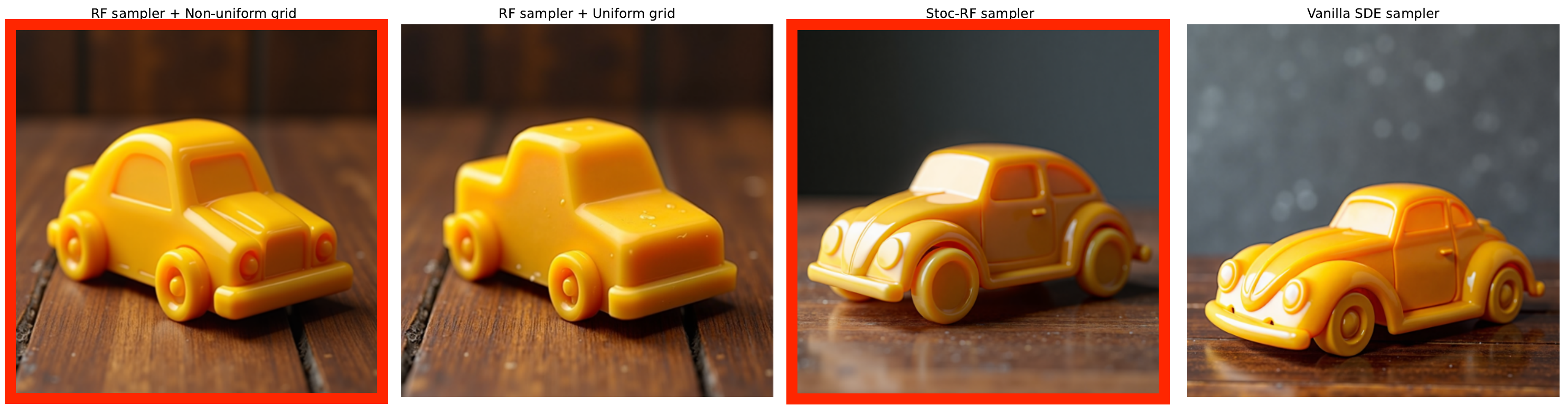}
        \caption{\sf \textbf{Prompt:} A \orange{car made of cheese} on a wooden table. ($N = 80$)}
    \end{subfigure}

    \caption{Generated images using Flux for specific prompts via Euler's scheme \eqref{eq: emp-ode-disc} under non-uniform time-grid \eqref{eq: gemoetric time discretization} (\red{Columns 1}),  uniform time-grid (Column 2), STOC-RF sampler (\red{Column 3}), and vanilla SDE sampler (Column 4). \blue{Red} color indicates the proposed methods in this paper. The generation quality of Stoc-RF is generally better out of the four samplers.}
    \label{fig: Flux SDE/RF experiments}
\end{figure}


\subsection{Ablation on Time Discretization}
\label{app:time-grid-ablation}

We investigate the effect of different time discretization
schedules on text-to-image generation using five representative
prompts from Appendix~\ref{app: T2I exp}. Specifically, we compare linear,
geometric, cosine, and logit schedules while keeping the
sampling budget fixed. Table~\ref{tab:time-grid-ablation}
reports the corresponding CLIP, PickScore, and CER metrics.
The geometric schedule achieves competitive performance
across the considered prompts, although no single schedule
uniformly outperforms the others. These results suggest that
the proposed geometric schedule is a robust choice for
time discretization in practical sampling.

\begin{table*}[t]
\centering
\caption{\small Ablation study of time discretization
schedules on five T2I prompts from Appendix~\ref{app: T2I exp}.
Best results for each prompt are highlighted in bold.}
\label{tab:time-grid-ablation}
\small
\setlength{\tabcolsep}{12pt}
\begin{tabular}{lccc}
\toprule
\rowcolor[HTML]{C5E1A5}
Time Grid & CLIP ($\uparrow$) &
PickScore ($\uparrow$) & CER ($\downarrow$) \\
\midrule

\multicolumn{4}{l}{\textit{Prompt: Cat with sign}} \\
Linear    & $\mathbf{35.54 \pm 1.93}$ & $0.2402 \pm 0.0068$ & $\mathbf{0.6851 \pm 0.0454}$ \\
Geometric & $35.44 \pm 0.98$ & $0.2401 \pm 0.0025$ & $0.7064 \pm 0.0159$ \\
Cosine    & $35.27 \pm 2.33$ & $\mathbf{0.2414 \pm 0.0033}$ & $0.7489 \pm 0.0834$ \\
Logit     & $35.24 \pm 0.94$ & $0.2403 \pm 0.0023$ & $0.7064 \pm 0.0085$ \\
\midrule

\multicolumn{4}{l}{\textit{Prompt: T-shirt (``I love AI'')}} \\
Linear    & $39.59 \pm 2.67$ & $0.2389 \pm 0.0072$ & $0.8444 \pm 0.1333$ \\
Geometric & $40.87 \pm 1.22$ & $0.2414 \pm 0.0025$ & $\mathbf{0.7111 \pm 0.1333}$ \\
Cosine    & $\mathbf{41.64 \pm 0.82}$ & $0.2374 \pm 0.0042$ & $0.7778 \pm 0.1988$ \\
Logit     & $40.84 \pm 1.11$ & $\mathbf{0.2417 \pm 0.0023}$ & $\mathbf{0.7111 \pm 0.1663}$ \\
\midrule

\multicolumn{4}{l}{\textit{Prompt: Robot Lecturer}} \\
Linear    & $31.42 \pm 1.49$ & $0.2247 \pm 0.0112$ & $\mathbf{2.5739 \pm 0.9245}$ \\
Geometric & $30.61 \pm 1.57$ & $0.2267 \pm 0.0079$ & $2.9304 \pm 1.3950$ \\
Cosine    & $31.04 \pm 1.41$ & $\mathbf{0.2318 \pm 0.0048}$ & $2.6696 \pm 0.8947$ \\
Logit     & $\mathbf{31.45 \pm 1.54}$ & $0.2254 \pm 0.0099$ & $2.9304 \pm 1.5200$ \\
\midrule

\multicolumn{4}{l}{\textit{Prompt: Cheese car}} \\
Linear    & $33.52 \pm 1.85$ & $0.2366 \pm 0.0043$ & $\mathbf{1.0000 \pm 0.0000}$ \\
Geometric & $\mathbf{33.77 \pm 0.99}$ & $0.2356 \pm 0.0043$ & $\mathbf{1.0000 \pm 0.0000}$ \\
Cosine    & $32.01 \pm 1.39$ & $\mathbf{0.2390 \pm 0.0054}$ & $\mathbf{1.0000 \pm 0.0000}$ \\
Logit     & $31.61 \pm 1.13$ & $0.2333 \pm 0.0038$ & $\mathbf{1.0000 \pm 0.0000}$ \\
\midrule

\multicolumn{4}{l}{\textit{Prompt: Sign (``Chimpanzees'')}} \\
Linear    & $38.71 \pm 0.65$ & $0.2066 \pm 0.0044$ & $0.5224 \pm 0.0778$ \\
Geometric & $38.82 \pm 0.79$ & $0.2069 \pm 0.0038$ & $0.5254 \pm 0.1385$ \\
Cosine    & $38.42 \pm 0.35$ & $\mathbf{0.2101 \pm 0.0063}$ & $0.5433 \pm 0.0889$ \\
Logit     & $\mathbf{39.06 \pm 0.99}$ & $0.2075 \pm 0.0040$ & $\mathbf{0.4627 \pm 0.0433}$ \\
\bottomrule
\end{tabular}
\end{table*}

\subsection{Quality--Speed Trade-off}
\label{app:quality-speed}

We further investigate the quality--speed trade-off between
deterministic RF (ODE) and STOC-RF (SDE) by varying the
number of function evaluations (NFE). We consider two
prompts from Appendix~\ref{app: T2I exp} and measure generation quality
using CLIP, PickScore, and CER, together with wall-clock
sampling time on an NVIDIA L40S GPU. For deterministic RF,
we use the geometric time schedule.

As shown in Table~\ref{tab:quality-speed}, the sampling
time increases approximately linearly with NFE for both
methods, with comparable computational costs.
Deterministic RF generally achieves better generation
quality, particularly at smaller NFE. At larger NFE,
the performance gap becomes smaller for some metrics,
although deterministic RF remains competitive or superior
across most reported configurations.

\begin{table*}[t]
\centering
\caption{\small Quality--speed trade-off between
deterministic RF (ODE) and STOC-RF (SDE) on two T2I
prompts. Sampling time is measured on an NVIDIA L40S GPU.
NFE denotes the number of function evaluations.}
\label{tab:quality-speed}
\small
\setlength{\tabcolsep}{5pt}
\resizebox{\textwidth}{!}{
\begin{tabular}{c|cccc|cccc}
\toprule
\rowcolor[HTML]{C5E1A5}
& \multicolumn{4}{c|}{\textbf{Deterministic RF (ODE + Geometric)}}
& \multicolumn{4}{c}{\textbf{STOC-RF (SDE)}} \\
\cmidrule(lr){2-5}\cmidrule(lr){6-9}
\rowcolor[HTML]{C5E1A5}
NFE & Time (s) & CLIP ($\uparrow$) &
PickScore ($\uparrow$) & CER ($\downarrow$)
& Time (s) & CLIP ($\uparrow$) &
PickScore ($\uparrow$) & CER ($\downarrow$) \\
\midrule

\multicolumn{9}{l}{\textit{Prompt: Cat with sign}} \\
5  & 2.8  & 31.01 & 0.2124 & 0.8553
   & 2.1  & 23.15 & 0.1910 & 1.0000 \\
10 & 5.5  & 34.15 & 0.2326 & 0.7745
   & 4.8  & 23.15 & 0.1910 & 1.0000 \\
20 & 11.4 & 34.35 & 0.2399 & 0.7362
   & 10.8 & 18.46 & 0.1671 & 1.0000 \\
30 & 17.0 & 34.31 & 0.2403 & 0.7021
   & 16.3 & 34.97 & 0.2294 & 0.8468 \\
50 & 27.9 & 35.44 & 0.2401 & 0.7064
   & 27.4 & 32.39 & 0.2333 & 0.8511 \\
80 & 45.0 & 36.28 & 0.2398 & 0.6766
   & 44.7 & 33.63 & 0.2379 & 0.8340 \\
\midrule

\multicolumn{9}{l}{\textit{Prompt: Sign (``Chimpanzees'')}} \\
5  & 2.7  & 23.02 & 0.1729 & 0.8597
   & 2.1  & 23.30 & 0.1949 & 1.0000 \\
10 & 5.5  & 34.11 & 0.1975 & 0.6388
   & 4.8  & 23.30 & 0.1949 & 1.0000 \\
20 & 11.1 & 37.59 & 0.2071 & 0.4657
   & 10.5 & 18.99 & 0.1527 & 1.0000 \\
30 & 16.5 & 38.57 & 0.2075 & 0.4478
   & 16.1 & 38.31 & 0.2002 & 0.5522 \\
50 & 28.3 & 38.82 & 0.2069 & 0.5254
   & 27.4 & 37.06 & 0.2013 & 0.6358 \\
80 & 45.0 & 39.36 & 0.2074 & 0.5343
   & 44.6 & 37.58 & 0.2052 & 0.5552 \\
\bottomrule
\end{tabular}
}
\end{table*}

\section{Proof of main results}
\label{app: proof of main results}

\subsection{Velocity function and its first order identities}
\label{app: v first order identities}
Before proceeding, we illustrate the connection between the velocity $v_t(x)$ and the score function $s_t(x):= \nabla \log p_t(x)$. 
 To this end, we define two key quantities:
\[
\mu_{1\mid t}(x) = \bbE[X_1 \mid X_t = x], 
\; 
\Cov_{1 \mid t}(x) = \Cov(X_1 \mid X_t = x).
\]
For $0<t<1$, the true drift can be re-written as
\begin{equation}
\label{eq: v_tweedie}
v_t(x) = \bbE[X_1 - X_0 \mid X_t = x] = \frac{1}{1-t}\br{\mu_{1 \mid t}(x) - x}.
\end{equation}

By the first and second order Tweedie's formula \citep{efron2011tweedie, meng2021estimating}, it follows that, for any choice of $t \in (0,1)$ and $x \in \mathbb{R}^d$,
\[
\mu_{1 \mid t}(x) =   \frac{x}{t} + \frac{(1-t)^2}{t} s_t(x),\] 
\[
\Cov_{1 \mid t}(x) = \frac{(1-t)^2}{t^2}I_d + \frac{(1-t)^4}{t^2} \nabla s_t(x).
\]
Substituting this in \eqref{eq: v_tweedie}, we obtain that 
$$
v_t(x) = \frac{x}{t} + \frac{(1-t)}{t} s_t(x), \quad  \text{and} \quad \nabla v_t(x) = - \frac{I_d}{1-t} + \frac{t}{(1-t)^3} \Cov_{1 \mid t}(x).
$$

\subsection{Proof of convergence of RF sampler (Theorem \ref{thm: RF - general convergence of Eulerian scheme} and Theorem \ref{thm: RF - convergence of Eulerian scheme})}
\label{app: RF convergence}
\subsubsection{Proof of Theorem \ref{thm: RF - general convergence of Eulerian scheme}}
\label{sec: proof of general convergence}
\paragraph{Notational conventions.}
We begin with some notational conventions for the ease of presentation. 

\begin{enumerate}
    \item We define transition map for the Euler's scheme as 
\[
\Phi_{t_i}(x):= x + \eta_{i} \widehat{v}_{t_i}(x) \quad \text{with $\eta_{i} = t_{i+1} - t_i$},
\]
where $\widehat v_{t}(x)$ is the velocity estimate. 

\item For notational brevity, we will use index $i$ to denote the time $t_i$. In this index notation we have $X_N \equiv X_1 \sim p_1$ 

\item With the above notations in mind, we have 
\begin{equation}
\label{eq: interp in modified index}
X_i = (1-t_i)X_0 + t_iX_N, \quad \text{and}\quad Y_{i+1}=\Phi_{i}\left(Y_{i}\right).
\end{equation}
\end{enumerate}

\paragraph{Step 1: TV recursion.}
We start by obtaining a recursion relation between the \TV distances between the updates $Y_i$ and the ground truth $X_i$. For any $i \le N-2$, we have
\begin{equation}
\label{eq: TV recursion}
\begin{aligned}
\TV\left(p_{X_{i+1}}, p_{Y_{i+1}}\right) & =\sup _{\mathcal{A} \subseteq \bbR^d}\left\{\mathbb{P}_{X_{i+1}}(\mathcal{A})-\mathbb{P}_{Y_{i+1}}(\mathcal{A})\right\}=\sup _{\mathcal{A} \subseteq \bbR^d}\left\{\mathbb{P}_{X_{i+1}}(\mathcal{A})-\mathbb{P}_{Y_{i}}\left(\Phi_{i}^{-1}(\mathcal{A})\right)\right\} \\
& \leq \sup _{\mathcal{A} \subseteq \bbR^d}\left\{\mathbb{P}_{X_{i+1}}(\mathcal{A})-\mathbb{P}_{X_{i}}\left(\Phi_{i}^{-1}(\mathcal{A})\right)\right\}+\sup _{\mathcal{A} \subseteq \bbR^d}\left\{\mathbb{P}_{X_{i}}\left(\Phi_{i}^{-1}(\mathcal{A})\right)-\mathbb{P}_{Y_{i}}\left(\Phi_{i}^{-1}(\mathcal{A})\right)\right\}  \\
& \leq \sup _{\mathcal{A} \subseteq \bbR^d}\left\{\mathbb{P}_{X_{i+1}}(\mathcal{A})-\mathbb{P}_{\Phi_{i}\left(X_{i}\right)}(\mathcal{A})\right\}+\TV\left(p_{X_{i}}, p_{Y_{i}}\right) \\
& =\TV\left(p_{X_{i+1}}, p_{\Phi_{i}\left(X_{i}\right)}\right)+\TV\left(p_{X_{i}}, p_{Y_{i}}\right)
\end{aligned}
\end{equation}

where the first identity arises from the basic property of the TV distance.

\paragraph{Step 2: Simplifying $\TV(p_{X_{i+1}}, p_{\Phi_i(X_i)})$.}


Next, we recall that $v_i(x) = \frac{x}{t_i} + \frac{1-t_i}{t_i} s_i(x)$. Therefore, by \citep[Lemma 13]{Liu2024_FlowNotes} we have 
\begin{equation}
\label{eq: partial v_x}
\frac{\partial v_i(x)}{\partial x} = -  \frac{1}{1-t_i} I + \frac{ t_i}{(1-t_i)^3} \Cov_{1 \mid t_i}(x_i)
\end{equation}
Further, for any vector $u \in \mathbb{R}^{d}$ with $\|u\|_{2}=1$ 
Assumption \ref{assumption: higher-order approximation error}\ref{assumption: Jone} tells us that for $i \le N-2$
$$
\begin{aligned}
& u^{\top} \frac{\partial \Phi_{i}(x)}{\partial x} u=u^{\top}\left(I+\eta_{i} \frac{\partial \widehat v_{i}(x)}{\partial x}\right) u \geq\left(1 - \frac12\right)\|u\|_{2}^{2} = \frac{1}{2}.
\end{aligned}
$$

Here, the penultimate relation follows from the positive semi-definiteness of the covariance matrix $\operatorname{Cov}_{1 \mid t}(x)$ and \blue{the inequality $\frac{\eta_i}{1 - t_i} \le \frac{\eta_i}{1 - t_{i+1}} 
\le h$ (\cf~ Assumption \ref{assumption: assumption of time schedule}).}


\textbf{Invertibility of $\Phi_i(\cdot)$:} The above inequality implies that $\frac{\partial {\Phi}_{i}(x)}{\partial x}$ is positive definite uniformly over all $x$. This entails that $\Phi_i(\cdot)$ is \textit{proper}, i.e, $\lim_{\Norm{x}_2 \to \infty} \Norm{\Phi(x)}_2 = \infty$. To see this, we define $\phi(t):=  u^\top \Phi_i(t u)$ for a fixed \textit{unit vector} $u$.Then, we have $\phi^\prime(t) = u^\top \nabla \Phi_i(t u) u > 1/2$. This shows that 
\[
\phi(t) - \phi(0) = \int_{0}^t \phi^\prime(u) \; du  \ge t/2.
\]
An application Cauchy-Schwarz inequality yields
\[
\Norm{\Phi_i(tu)}_2 \ge u^\top \Phi_i(t u) \ge t/2 + u^\top \Phi_i(0) \ge t/2 - \Norm{\Phi_i(0)}_2.
\]
Setting $t = \Norm{x}_2$ and $u = x/\Norm{x}_2$, we have $\Norm{\Phi_i(x)}_2 \ge \Norm{x}_2/2 - \Norm{\Phi_i(0)}_2$. This implies that $\Phi_i(\cdot)$ is proper. Then, by Theorem 2.2 of \cite{ruzhansky2015global} shows that $\Phi_i(\cdot)$ is a diffeomorphism.
Therefore, we can find a unique $x_{i} $ such that ${\Phi}_{i}\left(x_{i}\right)=x_{i+1}$, which in turn allows us to derive
$$
p_{{\Phi}_{i}\left(X_{i}\right)}\left(x_{i+1}\right)=p_{X_{i}}\left({\Phi}_{i}^{-1}\left(x_{i+1}\right)\right) \cdot \operatorname{det}\left(\frac{\partial {\Phi}_{i}^{-1}\left(x_{i+1}\right)}{\partial x_{i+1}}\right)=p_{X_{i}}\left(x_{i}\right) \cdot \operatorname{det}\left(\frac{\partial x_{i}}{\partial x_{i+1}}\right)
$$

Consequently, we see that: for any $i \geq 1$ 
,
\begin{equation}
\label{eq: TV distance p_phi_i and p_i+1}
\begin{aligned}
p_{{\Phi}_{i}\left(X_{i}\right)}\left(x_{i+1}\right) & -p_{X_{i+1}}\left(x_{i+1}\right)=p_{X_{i}}\left(x_{i}\right) \operatorname{det}\left(\frac{\partial x_{i}}{\partial x_{i+1}}\right)-p_{X_{i+1}}\left(x_{i+1}\right) \\
& =\int\left\{p_{X_{i} \mid X_N}\left(x_{i} \mid x_N\right) \operatorname{det}\left(\frac{\partial x_{i}}{\partial x_{i+1}}\right)-p_{X_{i+1} \mid X_N}\left(x_{i+1} \mid x_N\right)\right\} p_{X_N}\left(x_N\right) \mathrm{d} x_N \\
& =\int \Big\{1-\underbrace{\frac{p_{X_{i+1} \mid X_N}\left(x_{i+1} \mid x_N\right)}{p_{X_{i} \mid X_N}\left(x_{i} \mid x_N\right)} \operatorname{det}\left(\frac{\partial x_{i+1}}{\partial x_{i}}\right)}_{=: \mathcal{T}\left(x_{i}, x_N\right)}\Big\} p_{X_{i}, X_N}\left(x_{i}, x_N\right) \operatorname{det}\left(\frac{\partial x_{i}}{\partial x_{i+1}}\right) \mathrm{d} x_N  \\
& =\int\left\{1-\mathcal{T}\left(x_{i}, x_N\right)\right\} p_{X_{i}, X_N}\left(x_{i}, x_N\right) \operatorname{det}\left(\frac{\partial x_{i}}{\partial x_{i+1}}\right) \mathrm{d} x_N
\end{aligned}
\end{equation}

\paragraph{Step 3: Analysis of $\cT(x_i, x_N)$.}
The next step is then to analyze $\mathcal{T}\left(x_{i}, x_N\right)$.
Controlling $\mathcal{T}\left(x_{i}, x_N\right)$ : Given how $X_{i+1}$ and $X_{i}$ are generated, we observe that

\begin{align*}
\mathcal{T}\left(x_{i+1}, x_N\right) & =\frac{p_{X_{i+1} \mid X_N}\left(x_{i+1} \mid x_N\right)}{p_{X_{i} \mid X_{N}}\left(x_{i} \mid x_N\right)} \operatorname{det}\left(\frac{\partial x_{i+1}}{\partial x_{i}}\right)=\frac{\left(\frac{1}{1-t_{i+1}}\right)^{d} \exp \left\{\frac{-\left\|x_{i+1}-t_{i+1}x_N\right\|_{2}^{2}}{2\left(1-t_{i+1}\right)^2}\right\}}{\left(\frac{1}{1-t_i}\right)^{d} \exp \left\{\frac{-\left\|x_{i}-t_i x_N\right\|_{2}^{2}}{2\left(1-t_i\right)^2}\right\}} \operatorname{det}\left(\frac{\partial x_{i+1}}{\partial x_{i}}\right) \\
& =\underbrace{\left(\frac{1-t_i}{1-t_{i+1}}\right)^{d} \operatorname{det}\left(\frac{\partial x_{i+1}}{\partial x_{i}}\right)}_{=: \mathcal{T}_{1}\left(x_{i}, x_N\right)} \underbrace{\exp \left\{\frac{\left\|x_{i}-t_i x_N\right\|_{2}^{2}}{2\left(1-t_i\right)^2}-\frac{\left\|x_{i+1}-t_{i+1}x_N\right\|_{2}^{2}}{2\left(1-t_{i+1}\right)^2}\right\}}_{=: \mathcal{T}_{2}\left(x_{t}, x_{0}\right)}
\end{align*}

leaving us with two terms to control.
Let us first study the term $\mathcal{T}_{1}\left(x_{t}, x_{0}\right)$, towards which we see that

\begin{align*}
\mathcal{T}_{1}\left(x_{t}, x_{0}\right) & =\left(\frac{1-t_i}{1-t_{i+1}}\right)^{d} \operatorname{det}\left(\frac{\partial (x_i + \eta_i \widehat v_i(x_i))}{\partial x_{i}}\right) \\
& =\left(\frac{1-t_i}{1-t_{i+1}}\right)^{d} \operatorname{det}\left(I+\eta_{t} \frac{\partial}{\partial x_{t}} v_i\left(x_{i}\right)+\eta_{i}\left(\frac{\partial}{\partial x_{i}} \widehat v_{i}\left(x_{i}\right)-\frac{\partial}{\partial x_{i}} v_{i}\left(x_{i}\right)\right)\right) \\
& \stackrel{(a)}{=}\left(\frac{1-t_i}{1-t_{i+1}}\right)^{d} \operatorname{det}\left(\left(1-\frac{\eta_{i}}{1-t_i}\right) I+\frac{ \eta_{i} t_i}{(1-t_i)^3} \operatorname{Cov}_{1 \mid t_i}(x)+\eta_{t} \varepsilon_i^{J}\left(x_{i}\right)\right)  \\
& =\operatorname{det}\left(\frac{1 - t_i - \eta_i}{1 - t_{i+1}} I+ \frac{\eta_i t_i}{(1-t_{i+1}) (1-t_i)^2} \operatorname{Cov}_{1 \mid t_i}\left(x_{i}\right)+ \frac{\eta_i(1-t_i)}{(1-t_{i+1})} \varepsilon_i^{J}\left(x_{t}\right)\right)
\\
& \stackrel{(b)}{=} \operatorname{det}\left( I+ \frac{\eta_i t_i}{(1-t_{i+1}) (1-t_i)^2} \operatorname{Cov}_{1 \mid t_i}\left(x_{i}\right)+ \frac{\eta_i(1-t_i)}{(1-t_{i+1})} \varepsilon_i^{J}\left(x_{i}\right)\right).
\end{align*}

 Here, (a) arises from Tweedie's formula, whereas (b) follows since

\begin{equation*}
\frac{1 - t_i - \eta_i}{1 - t_{i+1}} = \frac{1 -t_{i+1}}{1 -t_{i+1}} = 1.
\end{equation*}

Next, we turn attention to the term $\mathcal{T}_{2}\left(x_{t}, x_{0}\right)$, which satisfies

\begin{equation}
\label{eq: logT2}
\begin{aligned}
\log \mathcal{T}_{2}\left(x_{i}, x_{N}\right) & =\frac{\left\|x_{i}-t_i x_N\right\|_{2}^{2}}{2\left(1-t_i\right)^2}-\frac{\left\|x_{i+1}-t_{i+1}x_N\right\|_{2}^{2}}{2\left(1-t_{i+1}\right)^2}\\
& =\frac{\left\|x_{i}-t_i x_N\right\|_{2}^{2}}{2\left(1-t_i\right)^2}-\frac{\left\|x_{i} + \eta_i \widehat v_i(x_i)-t_{i+1}x_N\right\|_{2}^{2}}{2\left(1-t_{i+1}\right)^2}
\end{aligned}
\end{equation}
We focus on the term $x_{i} + \eta_i \widehat v_i(x_i)-t_{i+1}x_N$:
$$
\begin{aligned}
& x_{i} + \eta_i \widehat v_i(x_i)-t_{i+1}x_N \\
& =\left(1-\frac{\eta_{i}}{1-t_i}\right) x_{i}+\frac{ \eta_{i}}{1-t_i} \mu_{1 \mid t_i}\left(x_{i}\right) - (t_{i} + \eta_i) x_N + \eta_{i}\left(\widehat v_{i}\left(x_{i}\right)-v_{i}\left(x_{i}\right)\right)
\\
& =\left(1-\frac{\eta_{i}}{1-t_i}\right)\left(x_{i}- t_{i} x_N\right)+\frac{ \eta_{i}}{1-t_i} \mu_{1 \mid t_i}\left(x_{i}\right)-\frac{\eta_{i}}{1-t_i} t_{i} x_N - \eta_i x_N + \eta_{i}\left(\widehat v_{i}\left(x_{i}\right)-v_{i}\left(x_{i}\right)\right) 
\\
& =\left(\frac{1 - t_{i+1}}{1-t_i}\right)\left(x_{i}-t_i x_N\right)+\frac{\eta_i}{1-t_i}\left(\mu_{1 \mid t_i}\left(x_{i}\right)-x_{N}\right)
+
\eta_{i} \varepsilon_{i}^{\rm v}\left(x_{i}\right) 
\end{aligned}
$$

Substitution into \eqref{eq: logT2} yields

\begin{equation}
\label{eq: logT2_expanded}
\begin{aligned}
& \log \mathcal{T}_{2}\left(x_{i}, x_{N}\right)\\
& = -\frac{\eta_i}{(1-t_i)^2 (1-t_{i+1})}(x_i - t_i x_N)^\top(\mu_{1\mid t_i}(x_i) - x_N) - \frac{\eta_i}{(1-t_i)(1-t_{i+1})}(x_i - t_i \mu_{1 \mid t_i}(x_i))^\top \varepsilon_i^{\rm v}(x_i)
\\
& - \frac{\eta_i t_{i+1}}{(1-t_{i+1})^2} \left(\mu_{1\mid t_i}(x_i) - x_N\right) \varepsilon_i^{\rm v}(x_i)
 - \frac{\eta_i^2 \Norm{\mu_{1\mid t_i}(x_i) - x_N}_2^2 }{2(1-t_i)^2 (1-t_{i+1})^2} - \frac{\eta_i^2 \Norm{\varepsilon_i^{\rm v}(x_i)}_2^2}{2(1-t_{i+1})^2}.
\end{aligned}
\end{equation}
Now Define:
\begin{align*}
\xi(x_i, x_N)& := -\frac{\eta_i(x_i - t_i x_N)^\top(\mu_{1\mid t_i}(x_i) - x_N)}{(1-t_i)^2 (1-t_{i+1})} - \frac{\eta_i t_{i+1}\left(\mu_{1\mid t_i}(x_i) - x_N\right) \varepsilon_i^{\rm v}(x_i)}{ (1-t_{i+1})^2} \\
 & - \frac{\eta_i^2 \Norm{\mu_{1\mid t_i}(x_i) - x_N}_2^2 }{2(1-t_i)^2 (1-t_{i+1})^2} + \frac{\eta_i}{(1-t_i)^2 (1-t_{i+1})} \left(t_i + \frac{\eta_i}{2(1-t_{i+1})}\right) \Tr(\Cov_{1\mid t_i})
\end{align*}
It is easy to see that 
\[
\int \xi(x_i, x_N) p_{X_N\mid X_i = x_i}(x_N) \; dx_N = 0.
\]

Thus, based on \eqref{eq: logT2_expanded}, we can further simplify $\log \mathcal{T}_{2}\left(x_{i}, x_{N}\right)$ as follows:

\begin{align*}
\log \mathcal{T}_{2}\left(x_{i}, x_{N}\right)= & 
\xi(x_i, x_N)-\frac{\eta_i (x_i - t_i \mu_{1\mid t_i}(x_i))^\top \varepsilon_i^{\rm v}(x_i)}{(1-t_i)(1-t_{i+1})} - \frac{\eta_i^2 \Norm{\varepsilon_i^{\rm v}(x_i)}_2^2}{2(1-t_{i+1})^2}\\
& - \frac{\eta_i}{(1-t_i)^2 (1-t_{i+1})} \left(t_i + \frac{\eta_i}{2(1-t_{i+1})}\right) \Tr(\Cov_{1\mid t_i})
\end{align*}

Now, for ease of presentation, define
$$
\begin{aligned}
W\left(x_{i}\right) & :=\log \operatorname{det}\left( I+ \frac{\eta_i t_i}{(1-t_{i+1}) (1-t_i)^2} \operatorname{Cov}_{1 \mid t_i}\left(x_{i}\right)+ \frac{\eta_i(1-t_i)}{(1-t_{i+1})} \varepsilon_i^{J}\left(x_{i}\right)\right)
-  \frac{\eta_i^2 \Norm{\varepsilon_i^{\rm v}(x_i)}_2^2}{2(1-t_{i+1})^2} \\
& - \frac{\eta_i}{(1-t_i)^2 (1-t_{i+1})} \left(t_i + \frac{\eta_i}{2(1-t_{i+1})}\right) \Tr(\Cov_{1\mid t_i})
-
\frac{\eta_i (x_i - t_i \mu_{1\mid t_i}(x_i))^\top \varepsilon_i^{\rm v}(x_i)}{(1-t_i)(1-t_{i+1})}
\end{aligned}
$$

Then for any $x_{i} \in \mathcal{X}_{\text {data }}$, it holds that

\begin{equation}
\label{eq: integral inequality of W}
\int_{x_{N}}\left(1-e^{\xi\left(x_{t}, x_{N}\right)}\right) e^{W\left(x_{i}\right)} p_{X_{N} \mid X_{i}}\left(x_{N} \mid x_{i}\right) \mathrm{d} x_{N} \leq-e^{W\left(x_{i}\right)} \int_{x_{N}} \xi\left(x_{i}, x_{N}\right) p_{X_{N} \mid X_{i}}\left(x_{N} \mid x_{i}\right) \mathrm{d} x_{N}=0
\end{equation}

where the inequality results from the elementary inequality $1-e^{x} \leq-x$ for all $x \in \mathbb{R}$. Using the above quantities and \eqref{eq: TV distance p_phi_i and p_i+1} we get: for any $\mathcal{A} \subseteq \mathcal{X}_{\text {data }} \cap \mathcal{E}_{i}$,

\begin{align*}
&\mathbb{P}_{\Phi_{i}\left(X_{i}\right)}(\mathcal{A})-\mathbb{P}_{X_{i+1}}(\mathcal{A}) \\
& =\int_{\mathcal{A}}\left\{1-\mathcal{T}\left(x_{i}, x_N\right)\right\} p_{X_{i}, X_N}\left(x_{i}, x_N\right) \operatorname{det}\left(\frac{\partial x_{i}}{\partial x_{i+1}}\right) \mathrm{d} x_N \; \dx_{i+1}\\
& =\int_{\mathcal{A} \times \mathcal{X}_{\text {data }}}\left\{1-\mathcal{T}\left(x_{i}, x_{N}\right)\right\} p_{ X_{N} \mid X_i = x_i}\left( x_{N}\right) p_{X_i}(x_i)\operatorname{det}\left(\frac{\partial x_{i}}{\partial x_{i+1}}\right) \mathrm{d} x_{N} \;  \dx_{i+1} \\
& =\int_{\Phi_{i}^{-1}(\mathcal{A}) \times \mathcal{X}_{\text {data }}}\left\{1-\mathcal{T}\left(x_i, x_N\right)\right\} p_{ X_{N} \mid X_i = x_i}\left( x_{N}\right) p_{X_i}(x_i)\mathrm{d} x_N \mathrm{~d} x_{i} \\
& =\int_{\Phi_{i}^{-1}(\mathcal{A}) \times \mathcal{X}_{\text {data }}}\left\{1-e^{\xi\left(x_i, x_N\right)} \cdot e^{W\left(x_i\right)}\right\} p_{ X_{N} \mid X_i = x_i}\left( x_{N}\right) p_{X_i}(x_i) \mathrm{d} x_N \mathrm{~d} x_{i}  \\
& = \int_{\Phi_{i}^{-1}(\mathcal{A}) \times \mathcal{X}_{\text {data }}}\left\{\left(1-e^{\xi\left(x_i, x_N\right)}\right) e^{W\left(x_i\right)}+\left(1-e^{W\left(x_i\right)}\right)\right\} p_{ X_{N} \mid X_i = x_i}\left( x_{N}\right) p_{X_i}(x_i)\mathrm{d} x_N \mathrm{~d} x_i \\
& \stackrel{(a)}{\leq} \int_{x_i, x_N \in \Phi_{i}^{-1}(\mathcal{A}) \times \cX_{\rm data}}\left\{1-e^{W\left(x_i\right)}\right\} p_{X_i, X_N}\left(x_i, x_N\right) \mathrm{d} x_i\; \dx_N \\
& \leq \int_{x_i\in \Phi_{i}^{-1}(\mathcal{A}) }-W\left(x_i\right) p_{X_i}\left(x_i       \right) \mathrm{d} x_i, 
\end{align*}
where, ($a$) invokes \eqref{eq: integral inequality of W}, and the last inequality follows from the elementary inequality $1-e^{x} \leq-x$.

\paragraph{Step 4: Analyzing $W(x_i)$.}
First, we introduce the following lemma.
\begin{lemma}
\label{lemma: lower bound on log det}
Let $M\in\bbR^{d\times d}$ satisfy $(M+M^\top)/2\succeq I/2$. Then $\det M>0$, and
\[
\log\det M\ge\Tr(M-I)-2\Norm{M-I}_F^2.
\]
\end{lemma}
\begin{proof}
Write $B=M-I$ and $M(s)=I+sB$ for $s\in[0,1]$. The symmetric part of $M(s)$ is bounded below by $(1-s/2)I$. Thus, for every $u\in\bbR^d$,
\[
\Norm{u}_2\Norm{M(s)u}_2\ge u^\top M(s)u\ge(1-s/2)\Norm{u}_2^2.
\]
Hence $M(s)$ is invertible and $\Norm{M(s)^{-1}}_{\op}\le2$. Since $\det M(0)=1$, continuity gives $\det M(s)>0$ along the entire path. Setting $f(s)=\log\det M(s)$, we obtain
\[
f'(s)=\Tr(M(s)^{-1}B),\qquad
f''(s)=-\Tr\big((M(s)^{-1}B)^2\big)\ge-4\Norm{B}_F^2,
\]
where we used $\Tr(C^2)\le\Norm{C}_F^2$ for any real square matrix $C$, and $\Norm{M(s)^{-1}B}_F\le2\Norm{B}_F$. Therefore, Taylor's formula yields
\[
\log\det M=f(0)+f'(0)+\int_0^1(1-s)f''(s)\,ds
\ge\Tr(B)-2\Norm{B}_F^2.
\]
\end{proof}

To apply the lemma, recall that $\varepsilon_i^J(x_i)=\nabla\widehat v_{t_i}(x_i)-\nabla v_{t_i}(x_i)$. Using \eqref{eq: partial v_x} and $\eta_i=t_{i+1}-t_i$, the matrix inside the log determinant can be written as
\[
\begin{aligned}
M_i&:=I+\frac{\eta_it_i}{(1-t_{i+1})(1-t_i)^2}\Cov_{1\mid t_i}(x_i)
+\frac{\eta_i(1-t_i)}{1-t_{i+1}}\varepsilon_i^J(x_i)\\
&=\frac{1-t_i}{1-t_{i+1}}\left(I+\eta_i\nabla\widehat v_{t_i}(x_i)\right).
\end{aligned}
\]
Assumption \ref{assumption: higher-order approximation error}\ref{assumption: Jone} gives, for every $u\in\bbR^d$,
\[
u^\top M_i u\ge\frac{1-t_i}{2(1-t_{i+1})}\Norm{u}_2^2\ge\frac12\Norm{u}_2^2.
\]
Thus $(M_i+M_i^\top)/2\succeq I/2$, and Lemma \ref{lemma: lower bound on log det} applies. Using also $\Norm{A+D}_F^2\le2\Norm{A}_F^2+2\Norm{D}_F^2$ for the two terms in $M_i-I$, we have
\begin{align*}
& - \log \operatorname{det}\left( I+ \frac{\eta_i t_i}{(1-t_{i+1}) (1-t_i)^2} \operatorname{Cov}_{1 \mid t_i}\left(x_{i}\right)+ \frac{\eta_i(1-t_i)}{(1-t_{i+1})} \varepsilon_i^{J}\left(x_{i}\right)\right) \\
& \le  \frac{4 \eta_i^2 t_i^2}{(1-t_{i+1})^2 (1-t_i)^4} \Norm{\Cov_{1 \mid t_i}}_F^2 + \frac{4 \eta_i^2 (1-t_i)^2}{(1-t_{i+1})^2} \Norm{\varepsilon_i^J(x_i)}_F^2\\
& \quad - \frac{\eta_i t_i}{(1-t_{i+1}) (1-t_i)^2} \Tr(\Cov_{1 \mid t_i}) - \frac{\eta_i (1-t_i)}{(1-t_{i+1})} \Tr(\varepsilon_i^J(x_i)),
\end{align*}
This yields 
\begin{align}
- W\left(x_{i}\right) & = - \log \operatorname{det}\left( I+ \frac{\eta_i t_i}{(1-t_{i+1}) (1-t_i)^2} \operatorname{Cov}_{1 \mid t_i}\left(x_{i}\right)
+ \frac{\eta_i(1-t_i)}{(1-t_{i+1})} \varepsilon_i^{J}\left(x_{i}\right)\right)
+  \frac{\eta_i^2 \Norm{\varepsilon_i^{\rm v}(x_i)}_2^2}{2(1-t_{i+1})^2}\notag\\*
& + \frac{\eta_i}{(1-t_i)^2 (1-t_{i+1})} \left(t_i + \frac{\eta_i}{2(1-t_{i+1})}\right) \Tr(\Cov_{1\mid t_i})
+
\frac{\eta_i (x_i - t_i \mu_{1\mid t_i}(x_i))^\top \varepsilon_i^{\rm v}(x_i)}{(1-t_i)(1-t_{i+1})}\notag\displaybreak[1]\\
& \le \frac{4 \eta_i^2 t_i^2}{(1-t_{i+1})^2 (1-t_i)^4} \Norm{\Cov_{1 \mid t_i}(x_i)}_F^2 + \frac{4 \eta_i^2 (1-t_i)^2}{(1-t_{i+1})^2} \Norm{\varepsilon_i^J(x_i)}_F^2 + \frac{\eta_i^2 \Norm{\varepsilon_i^{\rm v}(x_i)}_2^2}{2(1-t_{i+1})^2}\notag\\*
& \quad + \frac{\eta_i (x_i - t_i \mu_{1\mid t_i}(x_i))^\top \varepsilon_i^{\rm v}(x_i)}{(1-t_i)(1-t_{i+1})} +  \frac{\eta_i}{(1-t_i)^2 (1-t_{i+1})} \left(t_i + \frac{\eta_i}{2(1-t_{i+1})} -  t_i\right) \Tr(\Cov_{1\mid t_i}(x_i) )\notag\\*
& \quad - \frac{\eta_i (1-t_i)}{(1-t_{i+1})} \Tr(\varepsilon_i^J(x_i))\notag\displaybreak[1]\\
& \le \frac{4 \eta_i^2 t_i^2}{(1-t_{i+1})^2 (1-t_i)^4} \Norm{\Cov_{1 \mid t_i}(x_i)}_F^2 + \frac{4 \eta_i^2 (1-t_i)^2}{(1-t_{i+1})^2} \Norm{\varepsilon_i^J(x_i)}_F^2 + \frac{\eta_i^2 \Norm{\varepsilon_i^{\rm v}(x_i)}_2^2}{2(1-t_{i+1})^2}\notag\\*
& \quad  +  \frac{\eta_i^2}{2(1-t_i)^2 (1-t_{i+1})^2} \Tr(\Cov_{1\mid t_i}(x_i) ) + \underbrace{\frac{\eta_i (x_i - t_i \mu_{1\mid t_i}(x_i))^\top \varepsilon_i^{\rm v}(x_i)}{(1-t_i)(1-t_{i+1})}}_{\Delta(\varepsilon_i^{\rm v}(x_i))}\notag\\*
& \quad - \frac{\eta_i (1-t_i)}{(1-t_{i+1})} \Tr(\varepsilon_i^J(x_i))\notag\displaybreak[1]\\
& 
\le \frac{4 \eta_i^2 t_i^2}{(1-t_{i+1})^2 (1-t_i)^4} \Norm{\Cov_{1 \mid t_i}(x_i)}_F^2 + \frac{4 \eta_i^2 (1-t_i)^2}{(1-t_{i+1})^2} \Norm{\varepsilon_i^J(x_i)}_F^2 + \frac{\eta_i^2 \Norm{\varepsilon_i^{\rm v}(x_i)}_2^2}{2(1-t_{i+1})^2}\notag\\*
& \quad  +  \frac{\eta_i^2}{2(1-t_i)^2 (1-t_{i+1})^2} \Tr(\Cov_{1\mid t_i}(x_i) ) + \underbrace{\frac{\eta_i (1-t_i)}{(1-t_{i+1})}\frac{(x_i - t_i \mu_{1\mid t_i}(x_i))^\top \varepsilon_i^{\rm v}(x_i)}{(1-t_i)^2}}_{\Delta(\varepsilon_i^{\rm v}(x_i))}\notag\\*
& \quad - \frac{\eta_i (1-t_i)}{(1-t_{i+1})} \Tr(\varepsilon_i^J(x_i))\label{eq: bound on W}
\end{align}

\paragraph{Step 5: Controlling the error arising from $\Delta(\varepsilon_i^{\rm v}(x_i))$.}
We start with the following integral quantity for any set $\cA$:
\begin{equation}
\label{eq: Delta(e_i) simplified}
\small
\begin{aligned}
    & \int_{x_i \in \cA} \frac{1}{(1-t_i)^2}(x_i - t_i \mu_{1 \mid t_i}(x_i))^\top \varepsilon^{\rm v}_i(x_i)p_{X_i}(x_i) \; \dx_i\\
    & = \int_{x_i, x_N \in \cA \times \cX_{\rm data}} \frac{1}{(1-t_i)^2}(x_i - t_i \mu_{1 \mid t_i}(x_i))^\top \varepsilon^{\rm v}_i(x_i)p_{X_i, X_N}(x_i, x_N) \; \dx_i\; \dx_N\\
    &  = \int_{x_i, x_N \in \cA \times \cX_{\rm data}} \frac{1}{(1-t_i)^2}(x_i - t_i \mu_{1 \mid t_i}(x_i))^\top \varepsilon^{\rm v}_i(x_i)
    p_{X_i \mid X_N = x_N}(x_i) p_{X_N}(x_N) \; \dx_i\; \dx_N
    \\
    & = \int_{x_i, x_N \in \cA \times \cX_{\rm data}} \frac{1}{(1-t_i)^2}\left\{(x_i - t_i x_N)^\top \varepsilon^{\rm v}_i(x_i) p_{X_i \mid X_N = x_N}(x_i)\; \dx_i \right\} p_{X_N}(x_N) \; \dx_N\\
    & \le \int_{x_i, x_N \in \bbR^d \times \cX_{\rm data}} \frac{1}{(1-t_i)^2}\left\{\left\vert(x_i - t_i x_N)^\top \varepsilon^{\rm v}_i(x_i)\right\vert p_{X_i \mid X_N = x_N}(x_i)\; \dx_i  \right\} p_{X_N}(x_N) \; \dx_N \\
    & \le \left(\int_{x_i, x_N \in \bbR^d \times \cX_{\rm data}} \left\{\frac{1}{(1-t_i)^4}\left\vert(x_i - t_i x_N)^\top \varepsilon^{\rm v}_i(x_i)\right\vert ^2 p_{X_i \mid X_N = x_N}(x_i)\; \dx_i  \right\} p_{X_N}(x_N) \; \dx_N\right)^{1/2}\\
    &  = \left(\int_{x_i, x_N \in \bbR^d \times \cX_{\rm data}} \left\{\frac{1}{(1-t_i)^4}\left \langle (x_i - t_i x_N) (x_i - t_i x_N)^\top , \varepsilon^{\rm v}_i(x_i) \varepsilon^{\rm v}_i(x_i)^\top \right\rangle  p_{X_i \mid X_N = x_N}(x_i)\; \dx_i  \right\} p_{X_N}(x_N) \; \dx_N\right)^{1/2} 
\end{aligned}
\end{equation}
Now, recall that $X_i \mid X_N = x_N \sim N(t_i x_N , (1-t_i)^2 I)$. Using standard techniques from \citep[Section B.5]{liang2025low}, we can obtain bound on \eqref{eq: Delta(e_i) simplified}:
\[
\int_{x_i \in \cA} \frac{1}{(1-t_i)^2}(x_i - t_i \mu_{1 \mid t_i}(x_i))^\top \varepsilon^{\rm v}_i(x_i)p_{X_i}(x_i) \; \dx_i \le \frac{2 \evel_i}{(1-t_i)} + 
\evelJone_i + \evelJtwo_i + \evelH_i.
\]
Substituting this in \eqref{eq: bound on W}, we have 
\begin{align*}
& \int_{x_i \in \cA}
\left\{\Delta(\evel_i(x_i))- \frac{\eta_i (1-t_i)}{(1-t_{i+1})} \Tr(\varepsilon_i^J(x_i))\right\} p_{X_i}(x_i) \\
& \le  \frac{2 \evel_i}{(1-t_i)} + 
\evelJone_i + \evelJtwo_i + \evelH_i +  \frac{\eta_i (1-t_i)}{(1-t_{i+1})} \br{\int \Tr(\varepsilon_i^J(x_i))^2 p_{X_i}(x_i)\dx_i}^{1/2}\\
& \le \frac{2 \eta_i }{(1-t_{i+1})}( \evel_i + \evelJone_i + \evelJtwo_i + \evelH_i) .
\end{align*}
Here, we used the fact that $\frac{\eta_i (1-t_i)}{(1 - t_{i+1})} \le \frac{\eta_i}{(1-t_{i+1})}$.

\paragraph{Step 6: Bounding the terms involving $\Cov_{1 \mid t_i}(x_i)$.} 
Here we always assume that $k\ge \log d$.\footnote{If $k <\log d$, then we can redefine $k$ as $\log d$. This does not change the results in Theorem \ref{thm: RF - convergence of Eulerian scheme} and Theorem \ref{thm: stoc-RF sampler non-Eulerian}.
}
We begin defining some quantities first. 
\begin{itemize}
    \item Let $\{x_i^\ast\}_{1 \le i \le N_{\epsilon_0}}$ be an $\epsilon_0$-net of $\mathcal{X}_{\mathrm{data}}$, with $N_{\epsilon_0}$ denoting its cardinality. Let $\{B_i\}_{1 \le i \le N_{\epsilon_0}}$ be a disjoint $\epsilon_0$-cover for $\mathcal{X}_{\mathrm{data}}$ such that $x_i^\ast \in B_i$ for each $i$. 
    
    \item Define the following two sets:
    \begin{equation}
        \mathcal{I} := \left\{ 1 \le i \le N_{\epsilon_0} : \mathbb{P}(X_N \in B_i) \ge \exp(-C_1 k \log N) \right\}
        \label{eq:I}
    \end{equation}
    and
    \begin{equation}
        \mathcal{G} := \left\{ w \in \mathbb{R}^d : \|w\|_2 \le 2\sqrt{d} + \sqrt{C_1 k \log N}, \;
        \left| (x_i^\ast - x_j^\ast)^\top w \right|
        \le \sqrt{C_1 k \log N}\, \|x_i^\ast - x_j^\ast\|_2,
        \; \forall\, 1 \le i,j \le N_{\epsilon_0} \right\}.
        \label{eq:G}
    \end{equation}
    for some sufficiently large universal constant $C_1 > 0$. As it turns out, $\bigcup_{i \in \mathcal{I}} B_i$ and $\mathcal{G}$ form certain high-probability sets related to the random vector $X_N\sim p_1$ and the standard Gaussian random vector in $\mathbb{R}^d$, respectively.
    
    \item Recall \eqref{eq: interp in modified index}, we can express
    \begin{equation}
        X_{t_i} \equiv X_{i} = t _i  X_N + (1 - t_i) X_0,
        \label{eq:xt}
    \end{equation}
    for some random vector $X_0 \sim \mathcal{N}(0, I_d)$. 
\end{itemize}
Next, we define the set  $\cT_t := \left\{t v_1 + (1-t) \omega: v_1 \in \cup_{i \in \cI}B_i, \omega \in \cG \right\}$. 

By the argument of \citep[Lemma 1]{liang2025low}, we have

\begin{equation}
    \label{eq: prob of exclusion}
    \pr\br{X_{i} \notin \cT_{t_i} } \le \exp\br{- C_2 k \log N}, \quad \text{for $C_2 \gg C_{\rm cover}$, and $0 \le t_i<1$.}
\end{equation}
For the trace bound, we additionally take the covering resolution sufficiently fine that
\begin{equation}
\label{eq: covering resolution compatibility}
\epsilon_0\ll \frac{\delta}{1-\delta}
\min\left\{1,\sqrt{\frac{k\log N}{d}}\right\}.
\end{equation}
Under the normalization $\alpha=t^2/(t^2+(1-t)^2)$, we have $(1-\alpha)/\alpha=(1-t)^2/t^2$. Thus, the covering-resolution condition used in \citep[Appendix G.2]{liang2025low} holds uniformly for $t\in[t_1,t_{N-1}]$. The same proof and the posterior-moment argument of \citep[Corollary 1]{liang2025low} then yield
\begin{equation}
    \label{eq: trace_Cov bound}
    \Tr(\Cov_{1 \mid t_i}(x_i) )\le \frac{C_3 (1 - t_i)^2}{t_i^2} k \log N, \quad \text{for $t_i\in[t_1,t_{N-1}]$ and $x_i \in \cT_{t_i}$.}
\end{equation}

\paragraph{Step 7: Bound on the term involving squared Frobenius norm using linearization:}
One can use \eqref{eq:  trace_Cov bound} directly to bound the term $\bbE_{X_i}\left(\Norm{\Cov_{1\mid t_i}(X_i)}_F^2\right)$ in \eqref{eq: bound on W}. However, this will lead to a quadratic dependence on $k$. Therefore, we present the following lemma that will yield a linear dependence on $k$. The proof is deferred to Appendix \ref{app: proof of Cov Frobenius bound}.
\begin{proposition}
\label{prop: Main cov bound via trace}
    For all $1\le i\le N-2$, provided $h\le1/4$ and $16C_3kh\log N<1$, the following holds:
    \begin{equation*}
    \begin{aligned}
    \frac{\eta_i t_i}{(1-t_{i+1})(1-t_i)^2} \bbE(\Norm{\Cov_{1\mid t_i}(X_i)}_F^2)
    &\le \frac{3}{2} \left(\bbE[\Tr(\Cov_{1\mid t_i}(X_i))]-\bbE[\Tr(\Cov_{1\mid t_{i+1}}(X_{i+1}))]\right)\\
    &\quad + \frac{24h\eta_i t_i}{(1-t_{i+1})^3(1-t_i)^2N^{10}}.
    \end{aligned}
\end{equation*}
    
\end{proposition}
Using Proposition \ref{prop: Main cov bound via trace} and retaining the expected covariance traces until summation, we have

\begin{equation}
\label{eq: probability bound decomposition}
\begin{aligned}
& \mathbb{P}_{\Phi_{i}\left(X_{i}\right)}(\mathcal{A})-\mathbb{P}_{X_{i+1}}(\mathcal{A}) \\
 & \le \frac{6\eta_i t_i}{(1-t_i)^2(1-t_{i+1})}
 \left(\bbE[\Tr(\Cov_{1\mid t_i}(X_i))]-\bbE[\Tr(\Cov_{1\mid t_{i+1}}(X_{i+1}))]\right)\\
 &\quad + \frac{\eta_i^2}{2(1-t_i)^2(1-t_{i+1})^2}\bbE[\Tr(\Cov_{1\mid t_i}(X_i))]\\
 &\quad + \frac{4\eta_i^2(1-t_i)^2}{(1-t_{i+1})^2}(\evelJone_i)^2
 + \frac{\eta_i^2}{2(1-t_{i+1})^2}(\evel_i)^2\\
  &\quad  + \frac{2 \eta_i }{(1-t_{i+1})}( \evel_i + \evelJone_i + \evelJtwo_i + \evelH_i) 
  + \frac{96h\eta_i^2t_i^2}{(1-t_{i+1})^4(1-t_i)^4N^{10}}.
\end{aligned}
\end{equation}

We next sum \eqref{eq: probability bound decomposition} using only Assumption \ref{assumption: assumption of time schedule}. Throughout this argument, take $h\le 1/4$ and $16C_3kh\log N<1$, as required by Proposition \ref{prop: Main cov bound via trace}. First, bounded support and \eqref{eq: prob of exclusion} extend \eqref{eq: trace_Cov bound} to
\[
\bbE[\Tr(\Cov_{1\mid t_i}(X_i))]
\le C_3\frac{(1-t_i)^2}{t_i^2}k\log N+N^{-10}.
\]
Indeed, the exceptional-set contribution is at most $R^2\exp(-C_2k\log N)\le N^{-10}$ after choosing $C_2$ sufficiently large relative to $c_R$.

The coefficients of the trace differences in \eqref{eq: probability bound decomposition} need not be increasing for a general admissible schedule. However, the schedule assumption gives the increasing upper bound
\[
\frac{6\eta_it_i}{(1-t_i)^2(1-t_{i+1})}
\le 24h\frac{t_i^2}{(1-t_i)^2}.
\]
For $t_i\le1/2$, this follows from $\eta_i\le ht_i$ and $1-t_{i+1}\ge(1-h)/2$; for $t_i>1/2$, it follows from $\eta_i\le h(1-t_{i+1})$ and $t_i\le2t_i^2$. Lemma \ref{lemma: RF ddt cov lemma} shows that each expected trace difference is nonnegative. We may therefore apply this upper bound and then collect the equal-index trace terms:
\begin{align*}
&\sum_{i=1}^{N-2}\frac{6\eta_it_i}{(1-t_i)^2(1-t_{i+1})}
\left(\bbE[\Tr(\Cov_{1\mid t_i}(X_i))]-\bbE[\Tr(\Cov_{1\mid t_{i+1}}(X_{i+1}))]\right)\\
&\le 24h\Bigg\{
\frac{t_1^2}{(1-t_1)^2}\bbE[\Tr(\Cov_{1\mid t_1}(X_1))]\\
&\qquad +\sum_{i=2}^{N-2}\left(\frac{t_i^2}{(1-t_i)^2}-\frac{t_{i-1}^2}{(1-t_{i-1})^2}\right)
\bbE[\Tr(\Cov_{1\mid t_i}(X_i))]\\
&\qquad -\frac{t_{N-2}^2}{(1-t_{N-2})^2}\bbE[\Tr(\Cov_{1\mid t_{N-1}}(X_{N-1}))]\Bigg\}\\
&\le 24C_3h k\log N\left\{1+\sum_{i=2}^{N-2}
\left[1-\left(\frac{t_{i-1}(1-t_i)}{t_i(1-t_{i-1})}\right)^2\right]\right\}
+\frac{24h}{\delta^2N^{10}}.
\end{align*}
In the last step we dropped the nonpositive terminal term, applied the unconditional trace bound only to nonnegative coefficients, and telescoped the exceptional-set coefficients.

For every step covered by the schedule assumption,
\[
\frac{\eta_i}{t_i}\le h,\qquad
\frac{\eta_i}{1-t_{i+1}}\le 2h.
\]
Consequently, using $1-e^{-x}\le x$ for $x\ge0$,
\begin{align*}
1-\left(\frac{t_{i-1}(1-t_i)}{t_i(1-t_{i-1})}\right)^2
&\le 2\log\left(\frac{t_i(1-t_{i-1})}{t_{i-1}(1-t_i)}\right)\\
&=2\log\left(1+\frac{\eta_{i-1}}{t_{i-1}}\right)
+2\log\left(1+\frac{\eta_{i-1}}{1-t_i}\right)
\le6h.
\end{align*}
Thus the sum of the trace-difference terms is at most
\[
24C_3h(1+6Nh)k\log N+\frac{24h}{\delta^2N^{10}}.
\]
For convenience, write
\begin{equation}
\label{eq: low-d scaling factor}
s_N:=\sum_{i=1}^{N-2}\frac{\eta_i^2}{t_i^2(1-t_{i+1})^2}.
\end{equation}
This quantity is used only to collect the linear trace terms and the remainders before applying $s_N\lesssim Nh^2$.
The remaining linear trace terms satisfy
\[
\sum_{i=1}^{N-2}\frac{\eta_i^2}{2(1-t_i)^2(1-t_{i+1})^2}
\bbE[\Tr(\Cov_{1\mid t_i}(X_i))]
\le\frac{C_3}{2}s_Nk\log N+\frac{s_N}{2\delta^2N^{10}}.
\]
Also, the same schedule assumption gives $s_N\lesssim Nh^2$. The corrected remainders in \eqref{eq: probability bound decomposition} obey
\begin{align*}
\sum_{i=1}^{N-2}\frac{96h\eta_i^2t_i^2}{(1-t_{i+1})^4(1-t_i)^4N^{10}}
&\le\frac{96h}{\delta^6N^{10}}
\sum_{i=1}^{N-2}\frac{\eta_i^2}{t_i^2(1-t_{i+1})^2}\\
&=\frac{96h s_N}{\delta^6N^{10}}
\lesssim\frac{h^3}{\delta^6N^9}.
\end{align*}
Combining these estimates, using the averaged approximation-error assumptions, and iterating \eqref{eq: TV recursion}, we obtain
\begin{equation*}
\begin{aligned}
\TV(p_{X_{N-1}},p_{Y_{N-1}})
&\le C_1^\prime(h+Nh^2)k\log N
+C_2^\prime Nh^2\meanevelJone^2+C_3^\prime Nh^2\meanevel^2\\
&\quad+C_4^\prime Nh(\meanevel+\meanevelJone+\meanevelJtwo+\meanevelH)\\
&\quad+C_5^\prime\left(\frac{h+Nh^2}{\delta^2N^{10}}+\frac{h^3}{\delta^6N^9}\right)
+\TV(p_{X_1},p_{Y_1}).
\end{aligned}
\end{equation*}
The $h k\log N$ term comes from the initial boundary term in the weighted sum. It is retained for a general schedule; no relation between $h$, $N$, and $\delta$ specific to the geometric schedule has been imposed.
\paragraph{Handling $\TV(p_{X_1}, p_{Y_1})$ separately:}
Recall that \eqref{eq: bound on W} is also valid for $t_{0} = 0, t_1 = t_0 + \eta_0 = \eta_0$. Therefore, we have 
\begin{align*}
& - W(x_0) \\
& \le \frac{4 \eta_0^2 }{(1-\eta_0)^2} \Norm{\varepsilon_0^J(x_0)}_F^2 + \frac{\eta_0^2 \Norm{\varepsilon_0^{\rm v}(x_0)}_2^2}{2(1-\eta_0)^2}  +  \frac{\eta_0^2}{2(1-\eta_0)^2} \tr(\Cov_{1} ) + \underbrace{\frac{\eta_0}{1 - \eta_0} x_0 ^\top \varepsilon_0^{\rm v}(x_0)}_{\Delta(\varepsilon_0^{\rm v}(x_0))}\\
& \quad - \frac{\eta_0^2 }{(1-\eta_0)} \tr(\varepsilon_0^J(x_0)). 
\end{align*}
Therefore, applying similar argument, we get 

\begin{align*}
    \TV(p_{X_1}, p_{Y_1}) \le A_1 \eta_0^2 \tr(\Cov_1) + A_2 \eta_0^2 (\evel_0)^2 + A_3 \eta_0^2 (\evelJone_0)^2 + A_4 (\evel_0 + \evelJone_0 + \evelJtwo_0 + \evelH_0) + \underbrace{\TV(p_{X_0},p_{ Y_0})}_{=0}.
\end{align*}

Therefore, combining the preceding bounds gives
\begin{equation}
\label{eq: final TV bound general}
\begin{aligned}
    & \TV(p_{X_{N-1}}, p_{Y_{N-1}})\\
    & \le C_1'' (h+Nh^2) k\log N + C_2''   N h^2 \meanevelJone^2 + C_3''  N h^2 \meanevel^2 \\
    & \quad + C_4'' N h (\meanevel + \meanevelJone + \meanevelJtwo + \meanevelH) + C_5'' \left(\frac{h+Nh^2}{\delta^2N^{10}}+\frac{h^3}{\delta^6N^9}\right)\\
     &\quad +   A_1 \eta_0^2 {\Tr(\Cov_1)} + A_2 \eta_0^2 (\evel_0)^2 + A_3 \eta_0^2 (\evelJone_0)^2 + A_4 (\evel_0 + \evelJone_0 + \evelJtwo_0 + \evelH_0) , 
\end{aligned}
\end{equation}
where $\Cov_1:= \Cov_{X_1 \sim p_1}(X_1)$.
Now, note that one can estimate $v_0(x) = \mu_{p_{X_1}} - x$ by $\widehat v_0(x):= \widehat \mu_1 - x$, where $\widehat \mu_1 = \sum_{i=1}^n x_{i,1}/n$ and $n$ is the (training) sample size. In this case, we have 
\begin{itemize}
    \item $\evelJone_0 = \evelJtwo_0 = \evelH = 0$.
    \item $(\evel_0)^2 = \bbE \Norm{\widehat \mu_1 - \mu_{p_{X_1}}}_2^2 = O\left( \frac{\tr(\Cov_1)}{n }\right)$.
\end{itemize}
This finishes the proof of Theorem \ref{thm: RF - general convergence of Eulerian scheme}.

\subsubsection{Proof of Theorem \ref{thm: RF - convergence of Eulerian scheme}}
The proof follows by applying the result of the previous section to the time schedule \eqref{eq: gemoetric time discretization}.
\paragraph{Recalling properties of  time scheduling \eqref{eq: gemoetric time discretization} (see Lemma \ref{lemma: properties of time-steps}):}
\begin{enumerate}  
    \item For all $i
    \in [N-2]$, we have $\frac{\eta_i}{1-t_i} \le h$ and $\frac{\eta_i}{1-t_{i+1}} \le h$.
    
     \item For all $i\in [N-2]$, we have $\frac{\eta_i^2(1-t_i)^2}{(1-t_{i+1})^2}\le h^2(1-t_i)^2\le h^2$, as $\frac{\eta_i}{1-t_{i+1}}\le h$ and $(1-t_i) \le 1$.
     
     \item $\sum_{i=1}^{N-2}\frac{ \eta_i^2 }{(1-t_{i+1})^2 t_i^2} = \sum_{\substack{1\le i\le N-2\\t_i<1/2}}\frac{ \eta_i^2 }{(1-t_{i+1})^2 t_i^2} + \sum_{\substack{1\le i\le N-2\\t_i\ge1/2}}\frac{ \eta_i^2 }{(1-t_{i+1})^2 t_i^2} \le 4 h^2 N$.
\end{enumerate}
This shows that Assumption \ref{assumption: assumption of time schedule} is satisfied. We take $N$ sufficiently large so that $h\le1/4$ and $16C_3kh\log N<1$, as required in the general theorem. By \eqref{eq: value of h}, these conditions hold when $N$ is sufficiently large relative to $k\log N\log(1/\delta)$.

Now, recall that $\delta=1/(N\vee d)\le1/N$, and that the geometric schedule satisfies
\[
\log(1+h)=\frac{2\log(1/(2\delta))}{N-2}.
\]
Using $h/2\le\log(1+h)\le h$, we obtain $h\asymp\log(1/\delta)/N$ and $Nh\gtrsim1$. Therefore,
\[
s_N\lesssim Nh^2\lesssim\frac{\log^2(1/\delta)}{N},
\qquad h+Nh^2\lesssim Nh^2.
\]
In particular, the initial boundary term in the weighted trace sum is absorbed into the leading term. Since $\log N\le\log(1/\delta)$, we have
\[
(h+Nh^2)k\log N\lesssim\frac{k\log^3(1/\delta)}{N}.
\]
The remainder terms in \eqref{eq: final TV bound general} satisfy
\[
\frac{h+Nh^2}{\delta^2N^{10}}+\frac{h^3}{\delta^6N^9}
\lesssim
\frac{\log(1/\delta)+\log^2(1/\delta)}{\delta^2N^{11}}
+\frac{\log^3(1/\delta)}{\delta^6N^{12}}.
\]
Also, $\eta_0=t_1=\delta$, so the bound for the first step is already included in \eqref{eq: final TV bound general}. Putting all of these facts together and substituting in \eqref{eq: final TV bound general}, we get
\begin{align*}
    &\TV(p_{X_{N-1}},p_{Y_{N-1}})\\
    &\le C_1''\frac{k\log^3(1/\delta)}{N}
    +C_2''\frac{\log^2(1/\delta)}{N}\meanevelJone^2
    +C_3''\frac{\log^2(1/\delta)}{N}\meanevel^2\\
    &\quad+C_4''\log(1/\delta)(\meanevel+\meanevelJone+\meanevelJtwo+\meanevelH)\\
    &\quad+C_5''\left(\frac{\log(1/\delta)+\log^2(1/\delta)}{\delta^2N^{11}}
    +\frac{\log^3(1/\delta)}{\delta^6N^{12}}\right)\\
    &\quad+A_1\delta^2\Tr(\Cov_1)+A_2\delta^2(\evel_0)^2
    +A_3\delta^2(\evelJone_0)^2
    +A_4(\evel_0+\evelJone_0+\evelJtwo_0+\evelH_0),
\end{align*}
where $\Cov_1:= \Cov_{X_1 \sim p_1}(X_1)$.
Now, note that one can estimate $v_0(x) = \mu_{p_{X_1}} - x$ by $\widehat v_0(x):= \widehat \mu_1 - x$, where $\widehat \mu_1 = \sum_{i=1}^n x_{i,1}/n$ and $n$ is the (training) sample size. In this case, we have 
\begin{itemize}
    \item $\evelJone_0 = \evelJtwo_0 = \evelH_0 = 0$.
    \item $(\evel_0)^2 = \bbE \Norm{\widehat \mu_1 - \mu_{p_{X_1}}}_2^2 = O\left( \frac{\tr(\Cov_1)}{n }\right)$.
\end{itemize}
This finishes the proof of Theorem \ref{thm: RF - convergence of Eulerian scheme}.


\subsection{Proof of convergence of STOC-RF (Theorem \ref{thm: stoc-RF sampler non-Eulerian})}
\label{app: Stoc-RF convergence}
The proof of this result essentially follows by reduction and leverages key results of \cite{liang2025low}. We start by recalling some key facts:
 $t(\tau) = \frac{\sqrt{\omega_\tau}}{\sqrt{1 - \omega_\tau} + \sqrt{\omega_\tau}}$, and $\frac{Y^\prime_\tau}{\sqrt{\omega_\tau}} \overset{d}{=} \frac{\tilde Z_{t(\tau)}}{t(\tau)}$. With this time transformation, we have 
 \begin{equation}
 \label{eq: alpha_t expression stoc-RF}
 \alpha_\tau = \frac{\omega_\tau}{\omega_{\tau -1}} = \frac{(t_i/ \sigma_{t_i})^2}{(t_{i+1}/\sigma_{t_{i+1}})^2} = \frac{R_i^2}{R_{i+1}^2}, \quad \text{ where $t_i = t(\tau)$}.
 \end{equation}
 We also have $s_{Y^\prime_\tau }(y) = \sigma_t s_{\tilde Z_t}(\tilde z)$, where $\tilde z = \sigma_t y$.
 Next, we recall the DDPM sampler considered in \cite{liang2025low}:

\begin{equation}
    \label{eq: DDPM sampler liang2025 2}
    \begin{aligned}
    & \hat Y^\prime_{\tau -1} = \frac{1}{\sqrt{\alpha_\tau}} \bc{ \hat Y^\prime_\tau + \delta_\tau \widehat s_{Y^\prime_\tau}(\hat Y^\prime_\tau) + \nu_\tau \xi_\tau}, \hat Y^\prime_N \sim \cN(0, I_d)\\
    & \delta_\tau = 1 - \alpha_\tau,  \nu_\tau = \sqrt{\frac{(\alpha_\tau - \omega_\tau)(1 - \alpha_\tau)}{1 - \omega_\tau}},\\
    & \tau   = N , N-1, \ldots, 2, \quad \{\xi_\tau\}_{\tau \ge 1} \overset{i.i.d.}{\sim} N(0,I_d),
    \end{aligned}
\end{equation}

Now, if we substitute all the quantities of the above sampler by their corresponding RF versions, i.e.,  we get 
$\hat Y_\tau \rightarrow  Y_{t}/\sigma_t, \widehat s_{Y^\prime_\tau }(y) \rightarrow \sigma_t \widehat s_{t_i}(Y_{t_i})$, and $\alpha_\tau \rightarrow \frac{R_i^2}{R_{i+1}^2} $, then it leads to the following sampler:
\begin{equation}
     \begin{aligned}
      &\frac{Y_{t_{i+1}}}{\sigma_{t_{i+1}}} = \frac{R_{i+1}}{R_i} \bc{ \frac{Y_{t_i}}{\sigma_{t_i}} + \eta_i \sigma_{t_i} \widehat s_{t_i}(Y_{t_i}) + \sqrt{\psi_i} W_{t_i}},\\
      &  \frac{Y_{t_0}}{\sigma_{t_0}} \sim N(0,  I_d), \quad \{W_{t_i}\}_{i \ge 0}\overset{i.i.d.}{\sim} N(0,I_d),\\
      & i = 0, 1, \ldots, N-2,
     \end{aligned}
 \end{equation}
 where $ \eta_i = 1- \frac{R_i^2}{R_{i+1}^2},   \psi_i = \frac{R_i^2}{R_{i+1}^2}.\frac{ 1 - R_{i+1}^2}{1  - R_i^2}.\br{1 - \frac{R_i^2}{R_{i+1}^2}}$, and $\widehat s_{t_i}(x) = \frac{t_i \widehat v_{t_i}(x) - x}{1 - t_i}$. Here $\eta_i$ has the same meaning as in \eqref{eq: stoc-RF sampler non-Eulerian}. The above is exactly the sampler described in \eqref{eq: stoc-RF sampler non-Eulerian}. In other words, we have equivalence between two above samplers. Therefore, we can use the convergence result of \eqref{eq: DDPM sampler liang2025 2}.

 \paragraph{Defining few auxiliary processes:}
 All updates below are for $\tau=N,\ldots,2$, so that the final index is $\tau=1$, corresponding to $i=N-1$.
 We define the auxiliary processes in exact way defined in \citep[Section C.1]{liang2025low}.
 \begin{itemize}
    \item First, we define $\{Y^\star_\tau\}_{\tau = N}^1$ by 
    \[
    Y^\star_N = Y^\prime_N \sim \cN(0, I_d), \quad Y^\star_{\tau - 1} = \frac{1}{\sqrt{\alpha_\tau}} \br{Y^\star_\tau + \delta_\tau s_{Y^\prime_\tau}(Y^\star_\tau) + \nu_\tau \xi_\tau}.
    \]
Note that, the above update rule is just the analog of $\{\hat Y_\tau\}_\tau$ process with the true score.
 
  \item We construct an auxiliary process $\bar{Y}_{\tau}$ that follows the same transition dynamics as $Y_{\tau}^{\star}$:
\end{itemize}

\begin{equation}
\label{eq: aux process bar_Y}
\bar{Y}_{\tau - 1}^{-}\left|\bar{Y}_{\tau} \sim Y_{\tau - 1}^{\star}\right| Y_{\tau}^{\star}, \quad \bar{Y}_{\tau - 1} \left\lvert\,\left\{\bar{Y}_{\tau - 1}^{-}=y_{\tau - 1}^{-}\right\}= \begin{cases}y_{\tau - 1}^{-}, & \text {with prob. } \frac{p_{Y^\prime_{\tau - 1}}\left(y_{\tau - 1}^{-}\right)}{p_{\bar{Y}_{\tau - 1}^{-}}\left(y_{\tau - 1}^{-}\right)} \wedge 1  \\ \infty, & \text { otherwise }\end{cases}\right.
\end{equation}

for any $y_{\tau - 1}^{-} \neq \infty$, where we recall that $a \wedge b:=\min \{a, b\}$. It is straightforward to show that

\begin{equation}
\label{eq: pdf of aux process bar_Y}
p_{\bar{Y}_{\tau}}\left(y_{\tau}\right)=\int_{\mathbb{R}^{d}}\left(p_{Y^\prime_{\tau}}\left(y_{\tau}^{-}\right) \wedge p_{\bar{Y}_{\tau}^{-}}\left(y_{\tau}^{-}\right)\right) \delta\left(y_{\tau}-y_{\tau}^{-}\right) \mathrm{d} y_{\tau}^{-}=p_{Y^\prime_{\tau}}\left(y_{\tau}\right) \wedge p_{\bar{Y}_{\tau}^{-}}\left(y_{\tau}\right) 
\end{equation}

for any $y_{\tau} \neq \infty$, where $\delta(\cdot)$ denotes the Dirac measure.

\begin{itemize}
  \item To account for the score estimation error, we introduce another auxiliary process $\widetilde{Y}_{\tau}$ based on the dynamics of $Y_{\tau}$ :
\end{itemize}

\begin{equation}
\label{eq: aux process hat_Y}
\widetilde{Y}_{\tau - 1}^{-}\left|\widetilde{Y}_{\tau} \sim \hat Y^\prime_{\tau - 1}\right| \hat Y^\prime_{\tau}, \quad \widetilde{Y}_{\tau - 1} \left\lvert\,\left\{\widetilde{Y}_{\tau - 1}^{-}=y_{\tau - 1}^{-}\right\}= \begin{cases}y_{\tau - 1}^{-}, & \text {with prob. } \frac{p_{Y^\prime_{\tau - 1}}\left(y_{\tau - 1}^{-}\right)}{p_{\bar{Y}_{\tau - 1}^{-}}\left(y_{\tau - 1}^{-}\right)} \wedge 1 \\ \infty, & \text { otherwise. }\end{cases}\right.
\end{equation}


 Next, we observe that for $t_i = t(\tau)$ 
 \begin{equation}
     \label{eq: score error equivalence}
      (\varepsilon^s_\tau)^2 := \bbE \Norm{\widehat s_{Y^\prime_{\tau}}(Y^\prime_{\tau}) - s_{Y^\prime_{\tau}}(Y^\prime_{\tau})}_2^2  = \sigma_{t_i}^2 \bbE \Norm{ \widehat s_{t_i}(X_{t_i}) - s_{t_i}(X_{t_i})}_2^2  \le (\escore_i)^2.
 \end{equation}
 To see the weight of the score error, the Gaussian transitions with the true and approximate scores have the same covariance $\nu_\tau^2 I_d/\alpha_\tau$. Their conditional KL divergence at $y$ is therefore
 \[
 \frac{(1-\alpha_\tau)^2}{2\nu_\tau^2}
 \Norm{s_{Y^\prime_\tau}(y)-\widehat s_{Y^\prime_\tau}(y)}_2^2
 =\frac{(1-\alpha_\tau)(1-\omega_\tau)}{2(\alpha_\tau-\omega_\tau)}
 \Norm{s_{Y^\prime_\tau}(y)-\widehat s_{Y^\prime_\tau}(y)}_2^2.
 \]
 For sufficiently large $N$, the schedule in \eqref{eq: DDPM time scheduling} gives $\alpha_\tau\ge 1/2$ and $\beta_\tau\le 2\beta_{\tau-1}$ for $\tau\ge2$. Since $1-\omega_{\tau-1}\ge\beta_{\tau-1}$, we have
 \[
 \frac{1-\omega_\tau}{\alpha_\tau-\omega_\tau}
 =1+\frac{\beta_\tau}{\alpha_\tau(1-\omega_{\tau-1})}\le5.
 \]
 Thus, the coefficient is bounded by a constant times $1-\alpha_\tau\lesssim\log N/N$. The common acceptance step cannot increase KL divergence, and \eqref{eq: pdf of aux process bar_Y} bounds the finite part of $p_{\bar Y_\tau}$ by $p_{Y^\prime_\tau}$. Hence, the expected score error is bounded by $(\varepsilon_\tau^s)^2$. Finally, $i=N-\tau$ maps $\tau=2,\ldots,N$ to $i=0,\ldots,N-2$.
 Now, we can exactly replicate the analysis in \citep[Section C.2]{liang2025low} till Equation (106) in Step 6, which yields the following for a universal constant $L>0$:

 \[
 \KL\br{p_{\bar Y_1} \| p_{\widetilde{Y}_1} } \le \sum_{\tau = 2}^N L(1 - \alpha_\tau) (\varepsilon_\tau^s)^2 \lesssim \frac{\log N}{N}\sum_{i=0}^{N-2} (\escore_i)^2 \lesssim  \meanescore^2 \log N.
 \]
 After this, Step 7 of the proof follows verbatim to yield the final bound in Theorem \ref{thm: stoc-RF sampler non-Eulerian}.

\subsection{Description of the new RF sampler in Remark \ref{remark: new RF sampler}}
\label{app: new RF sampler}
The DDIM sampler considered in \cite{liang2025low} takes the form

\begin{equation}
    \label{eq: DDIM sampler liang2025 2}
    \begin{aligned}
    & \hat Y^\prime_{\tau -1} = \frac{1}{\sqrt{\alpha_\tau}} \bc{ \hat Y^\prime_\tau + \delta_\tau \widehat s_{Y^\prime_\tau}(\hat Y^\prime_\tau)}, \; \hat Y^\prime_N \sim \cN(0, I_d)\\
    & \delta_\tau = \frac{1 - \alpha_\tau}{1 + \sqrt{\frac{\alpha_\tau - \omega_\tau}{1 - \omega_\tau}}},\\
    & \tau   = N , N-1, \ldots, 1.
    \end{aligned}
\end{equation}
Using the same transformations detailed in Section \ref{app: Stoc-RF convergence} and plugging them in \eqref{eq: DDIM sampler liang2025 2}, we arrive at the RF sampler

\begin{equation}
     \label{eq: DDIM sampler non-Eulerian}
     \begin{aligned}
      &\frac{Y_{t_{i+1}}}{\sigma_{t_{i+1}}} = \frac{R_{i+1}}{R_i} \bc{ \frac{Y_{t_i}}{\sigma_{t_i}} + \eta_i \sigma_{t_i} \widehat s_{t_i}(Y_{t_i}) },\\
      &  \frac{Y_{t_0}}{\sigma_{t_0}} \sim N(0,  I_d),\\
      & i = 0, 1, \ldots, N-1,
     \end{aligned}
 \end{equation}
where $\eta_i = \frac{1 - \frac{R_i^2}{R_{i+1}^2}}{1 + \sqrt{\frac{R_i^2}{R_{i+1}^2}. \frac{1 - R_{i+1}^2}{1 -  R_{i}^2}}}, R_i = t_i/\sigma_{t_i}$.

    

\paragraph{Algebraic simplification.}
Using $R_i=t_i/\sigma_{t_i}$, we have
\[
\frac{R_{i+1}}{R_i}
=
\frac{t_{i+1}}{t_i}
\frac{\sigma_{t_i}}{\sigma_{t_{i+1}}},
\qquad
\frac{R_i^2}{R_{i+1}^2}.\frac{1 - R_{i+1}^2}{1 - R_i^2} = \frac{t_i^2 (1-t_{i+1})^2}{t_{i+1}^2(1-t_i)^2},
\qquad
\eta_i
=
\frac{1-\frac{t_i^2\sigma_{t_{i+1}}^2}{t_{i+1}^2\sigma_{t_i}^2}}{1 + \frac{t_i (1-t_{i+1})}{t_{i+1}(1-t_i)}}.
\]
Multiplying \eqref{eq: DDIM sampler non-Eulerian} by $\sigma_{t_{i+1}}$ yields
\[
Y_{t_{i+1}}
=
\frac{t_{i+1}}{t_i}
\Bigl(
Y_{t_i}
+
\eta_i\,
\sigma_{t_i}^2\,
\widehat s_{t_i}(Y_{t_i})
\Bigr).
\]
A direct calculation using $\sigma_t^2=t^2+(1-t)^2$ shows that
\[
\eta_i\sigma_{t_i}^2
=
\frac{
(t_{i+1}-t_i)
\bigl(1-t_i\bigr)
}{
t_{i+1}
}.
\]
Therefore, the deterministic sampler can be written equivalently as
\begin{equation}
\label{eq:ode-sampler-simplified}
Y_{t_{i+1}}
=
\frac{t_{i+1}}{t_i}\,Y_{t_i}
+
\frac{(t_{i+1}-t_i)(1-t_i)}{t_i}
\widehat s_{t_i}(Y_{t_i}) = Y_{t_i} + \Delta_i \underbrace{\bc{\frac{Y_{t_i}}{t_i} + \frac{1 - t_i}{t_i} \widehat s_{t_i}(Y_{t_i})}}_{ = \widehat v_{t_i}(Y_{t_i})},
\end{equation}
where $ \Delta_i = t_{i+1} - t_i$. This expression matches exactly the form of the RF sampler in \eqref{eq: emp-ode-disc}. However, the time steps $
\Delta_i$ are different from those we propose in \eqref{eq: gemoetric time discretization}, which we used to obtain the RF convergence guarantees given in \Cref{thm: RF - convergence of Eulerian scheme}. Although the time-schedule is still U-shaped as shown in \Cref{fig: hist of sde timegrid} (b), however, it is not symmetric and, as a result, might take more discretization steps at either the beginning or the end of the sampling stage. Whether a symmetric or asymmetric schedule is better remains an open problem. Furthermore, we observe that empirically it appears to be relatively sensitive to the choice of the parameters $c_0$ and $c_1$ in \eqref{eq: DDPM time scheduling}.

\section{Proof of Lemmas related to Stochastic Localization (Section \ref{sec: SL})}
\label{app: SL lemmas}

\subsection{Proof of Lemma \ref{lemma: time change DDPM-SL} and Lemma \ref{lemma: SL and RF are equivalent}}
\label{app: SL = RF = DDPM}
We first show that the solution to the SDE \eqref{eq: forward process} takes the  form
\begin{equation}
\label{eq: forwards process DDS solution}
 \frac{Y'_{\tau}}{\sqrt{\omega_{\tau}}} =  X_1 + \widetilde B_{\frac{1 -\omega_{\tau}}{\omega_{\tau}}}^\prime,  
\end{equation}
where $\omega_{\tau} = \exp(-2 \int_0^\tau \beta(u) du)$ and $\bc{\widetilde B^\prime_\tau}_{\tau\ge 0}$ is a standard Brownian motion. First, it is easy to see that the solution of SDE \eqref{eq: forward process} is 
\begin{equation}
\label{eq: forward process solution}
\frac{Y'_{\tau}}{\sqrt{\omega_{\tau}}} = X_1 + \underbrace{\int_0^\tau \sqrt{2 \beta(t)} \exp \br{\int_0^t \beta(u) du} dB^\prime_t}_{:= M_\tau}.
\end{equation}
We note that $M_\tau$ is continuous local Martingale with $M_0 = 0$. Also, the quadratic variation of $M_\tau$ is 
$$
\innerprod{M}_\tau = \int_{0}^\tau 2 \beta(t) \exp \br{2 \int_{0}^t \beta(u) du} dt = \exp \br{2 \int_{0}^\tau \beta(u) du} -1 = \frac{1 - \omega_\tau}{\omega_\tau}.
$$
Then, by Dambis–Dubins–Schwarz theorem \citep[Theorem 1.6]{revuz2013continuous} there exists a Brownian motion $\bc{\widetilde B_u^\prime}_{u \ge 0}$ such that $\bc{M_\tau}_{\tau \ge 0} = \bc{\widetilde B_{\frac{1 -\omega_{\tau}}{\omega_{\tau}}}^\prime}_{\tau \ge 0}$. Therefore, we have \eqref{eq: forwards process DDS solution} by substituting $M_\tau$ in \eqref{eq: forward process solution}.

Now, for $s>0$, we define two processes: $\hat U_s:= s Y^\prime_{\tau(s)}/\sqrt{\omega_{\tau(s)}} = s X_1 + s \widetilde B^\prime_{\frac{1 - \omega_{\tau(s)}}{\omega_{\tau(s)}}}$, and $\tilde U_s := s \tilde X_{t(s)}/t(s) = s X_1 + s \tilde B_{\frac{(1-t(s))^2}{t^2(s)}}$.
Now, under the prescribed time change in Lemma \ref{lemma: time change DDPM-SL} and Lemma \ref{lemma: SL and RF are equivalent}, we have 
\[
s \widetilde B^\prime_{\frac{1 - \omega_{\tau(s)}}{\omega_{\tau(s)}}} = s \widetilde B^\prime_{1/s}, \quad \text{and} \quad  s \tilde B_{\frac{(1-t(s))^2}{t^2(s)}} = s \tilde B_{1/s}.
\]
By Brownian time inversion, both $\bc{s\widetilde B^\prime_{1/s}}_{s>0}$ and $\bc{s\tilde B_{1/s}}_{s>0}$ extend continuously to Brownian motions by setting their values at $s=0$ to zero. Hence $\hat U_s$ and $\tilde U_s$, with the same zero extension, are equivalent in law to the SL process. The identities involving $U_s/s$ are asserted only for $s>0$.

\subsection{Proof of Proposition \ref{prop: RF-SDE and backwrd process are equivalent}: Equivalence of SDE \eqref{eq: RF-SDE} and SDE \eqref{eq: reverse OU process}}
\label{app: equivalence of backward and stoc-RF}

\paragraph{Behavior at the initial endpoint.}
The process \eqref{eq: RF linear stoc process} is defined for $t>0$. As $t\downarrow0$, its marginals converge weakly to $p_0=\mathcal N(0,I_d)$, but the process has no continuous extension to $t=0$. Indeed, writing $G_t=tW_{(1-t)^2/t^2}$, each coordinate satisfies $\operatorname{Var}(G_t-G_{2t})=(1-t)^2\to1$ for $0<t<1/2$. Thus $G_t$ is not Cauchy in probability at zero; neither is $\tilde X_t$, since $tX_1\to0$ almost surely. Also, for $\gamma_t=(1-t)/t$, we have $\int_0^\epsilon2\gamma_t\,dt=\infty$, so the Brownian integral in \eqref{eq: RF-SDE} cannot start at zero in the usual It\^o formulation. We therefore formulate the SDE equivalence on positive-time intervals with matching initial laws. A Gaussian initialization at positive time approximates $p_{t_0}$ and is distinct from this exact initialization.

\subsection*{SDE \eqref{eq: RF-SDE} to SDE \eqref{eq: reverse OU process}:}
We work on $t\in[t(N),1)$ and $0\le\tau<N$, with the matching initial laws in Proposition \ref{prop: RF-SDE and backwrd process are equivalent}. The differential calculations hold on compact subintervals of these domains. We begin with the solution $\tilde Z_t$ of STOC-RF SDE~\eqref{eq: RF-SDE}:
\begin{equation}
\label{eq: d tilde Z_t}
d \tilde Z_t = \bc{\frac{\tilde Z_t}{t} + 2 \br{\frac{1-t}{t}} s_t(\tilde Z_t)}dt + \sqrt{2 \br{\frac{1-t}{t}}} d B_t,
\end{equation}
where $s_t(\cdot)$ is the score of $\tilde Z_t$, and $\bc{B_t}_{t \ge 0}$ is Brownian motion.

Recall the transformation $t(\tau) = \frac{\sqrt{\omega_{\tau}}}{\sqrt{\omega_\tau} + \sqrt{1 - \omega_\tau}}$, where $\omega(\tau)  = \exp \br{- 2 \int_{0}^\tau \beta(u) du}$. 

Now, we define $\overleftarrow Y_\tau:= \frac{\sqrt{\omega_{N-\tau}} \tilde Z_{t(N-\tau)}}{t(N - \tau)}$. We first begin by noting that 
\begin{align}\label{eq:anidentity}
\frac{\sqrt{\omega_{N-\tau}}}{t(N- \tau)} = \frac{1}{ \sigma_{t(N - \tau)}}
\end{align}
We also have $\log \omega_{N- \tau} = - 2 \int_{0}^{N - \tau} \beta(u) du$. Therefore, we have 
\begin{equation}
\label{eq: domega/dt}
    \begin{aligned}
        &\quad  \frac{1}{\omega_{N - \tau}}. \frac{d \omega_{N-\tau}}{dt} = 2 \beta(N- \tau) \frac{d \tau}{dt}\\
        & \Rightarrow \frac{d \tau}{dt} \stackrel{(i)}{=} \frac{1}{2 \beta(N - \tau) \omega(N - \tau)} . \frac{d\omega_{N - \tau}}{dt} = \frac{1}{2 \beta(N - \tau) \omega(N - \tau)} . \frac{d(t^2/\sigma_t^2)}{dt}\\
        & \Rightarrow \frac{d \tau}{d t } = \frac{1}{\beta(N - \tau)}. \frac{1 - t}{t \sigma_t^2}
    \end{aligned}
\end{equation}
Step (i) follows from~\eqref{eq:anidentity}.
Next, we define the backward time $t_\tau:= t(N- \tau)$, and then we can write $\tilde Z_{t_\tau} = \sigma_{t_\tau} \overleftarrow Y_\tau$. This also yields that the score of $\overleftarrow Y_\tau$ satisfies $s^{\leftarrow}_\tau(\overleftarrow Y_\tau) = \sigma_{t_\tau} s_t(\tilde Z_{t_\tau})$. Moreover, $\eqref{eq: DDPM and RF equivalent}$ yields that $\sigma_{t_\tau} s_t(\tilde Z_{t_\tau}) = \nabla \log q_{N - \tau}(\overleftarrow Y_\tau)$.
Next, with \eqref{eq: domega/dt} and \eqref{eq: d tilde Z_t} in mind, we can write 

\begin{equation*}
\label{eq: d Y_tau part 1}
\small
\begin{aligned}
d \overleftarrow Y_\tau & = \frac{1}{\sigma_{t_\tau}} d \tilde Z_{t_\tau} - \frac{2t_\tau-1}{\sigma_{t_\tau}^3} \tilde Z_{t_\tau} dt_\tau\\
& = \bc{\frac{\tilde Z_{t_\tau}}{t_\tau \sigma_{t_\tau}} + 2 \br{\frac{1 - t_{\tau}}{t_{\tau }\sigma_{t_\tau}} } s_{t_\tau}(\tilde Z_{t_\tau}) - \frac{2 t_{\tau} - 1}{\sigma_{t_\tau}^3} \tilde Z_{t_\tau}} dt_\tau + \sqrt{2 \br{\frac{1 - t_{\tau}}{\sigma_{t_\tau}^2 t_\tau}}} d B_{t_\tau}\\
& = \bc{\frac{1 - t_{\tau}}{t_\tau \sigma^3_{t_\tau}} \tilde Z_{t_\tau} + 2 \br{\frac{1 - t_{\tau}}{t_{\tau }\sigma_{t_\tau}} } s_{t_\tau}(\tilde Z_{t_\tau}) } dt_\tau + \sqrt{2 \br{\frac{1 - t_{\tau}}{\sigma_{t_\tau}^2 t_\tau}}} d B_{t_\tau}\\
& = \bc{\frac{1 - t_{\tau}}{t_\tau \sigma^2_{t_\tau}} \overleftarrow Y_\tau + 2 \br{\frac{1 - t_{\tau}}{t_{\tau }\sigma^2_{t_\tau}} } s^\leftarrow_\tau(\overleftarrow Y_\tau) } \times \beta(N - \tau) \times \frac{t_{\tau} \sigma_{t_\tau}^2}{1 - t_\tau} d\tau + \sqrt{2 \br{\frac{1 - t_{\tau}}{\sigma_{t_\tau}^2 t_\tau}}} \times \sqrt{\beta(N - \tau) \times \frac{t_{\tau} \sigma_{t_\tau}^2}{1 - t_\tau}}  d \overleftarrow B_{\tau}\\
& = \beta(N - \tau) \bc{\overleftarrow Y_\tau + 2 s_\tau^\leftarrow (\overleftarrow Y_\tau)}d\tau + \sqrt{2 \beta(N  - \tau)}d \overleftarrow B_\tau\\
& = \beta(N - \tau) \bc{\overleftarrow Y_\tau + 2 \nabla \log q_{N - \tau}(\overleftarrow Y_\tau)}d\tau + \sqrt{2 \beta(N  - \tau)}d \overleftarrow B_\tau,
\end{aligned}
\end{equation*}
where $\bc{\overleftarrow B_\tau}_{\tau\ge 0}$ is another Brownian motion.
\paragraph{Initial condition check:} Recall that $\overleftarrow Y_0 = \frac{\sqrt{\omega_{N}} \tilde Z_{t(N)}}{t(N)} = (\sqrt{\omega_N} + \sqrt{1 - \omega_N})\tilde Z_{t(N)}$. Therefore, we have 
$$\Law (\overleftarrow Y_0) = \Law\bs{(\sqrt{\omega_N} + \sqrt{1 - \omega_N}) . \bc{t(N) X_1 + (1-t(N)) X_0}} = \Law(\sqrt{\omega_N} X_1 + \sqrt{1 - \omega_N} X_0) = \Law(Y^\prime _N).$$

This shows that $\overleftarrow Y_\tau$ satisfies the backward DDPM SDE \eqref{eq: reverse OU process}.  

\subsection*{SDE \eqref{eq: reverse OU process} to SDE \eqref{eq: RF-SDE}:}
For the other direction, we let $\bc{\overleftarrow Y_\tau}_{0\le\tau<N}$ be a solution of SDE \eqref{eq: reverse OU process}, with $\overleftarrow Y_0\sim q_N$. Now, to generate a solution of SDE \eqref{eq: RF-SDE}, we need to define
\[
\tilde Z_t = \frac{t(N - \tau(t)) \overleftarrow Y_{\tau(t)}}{\sqrt{\omega_{N - \tau(t)}}} = \frac{\overleftarrow Y_{\tau(t)}}{\sqrt{\omega_{N - \tau(t)}} + \sqrt{1 - \omega_{N - \tau(t)}}},
\]
where $\tau(t)$ should satisfy the following: 
\begin{equation}
\label{eq: dtau/dt}
\begin{aligned}
  t(N - \tau(t)) =  & \quad \frac{\sqrt{\omega_{N - \tau(t)}}}{\sqrt{\omega_{N - \tau(t)}} + \sqrt{1  - \omega_{N - \tau(t)}}} = t\\
    & \Rightarrow \omega_{N - \tau(t)} = \frac{t^2}{t^2 + (1-t)^2}\\
    & \Rightarrow \exp \br{- 2 \int_0^{N - \tau(t)} \beta(u) du} = \frac{t^2}{t^2 + (1-t)^2}\\
    & \Rightarrow \int_{0}^{N - \tau(t)} \beta(u) du = \log \br{\sqrt{1 + \frac{(1-t)^2}{t^2}}}.
\end{aligned}
\end{equation}
This shows that $\tilde Z_t = \sigma_{t} \overleftarrow Y_{\tau(t)}$. Additionally, \eqref{eq: DDPM and RF equivalent} yields that the score of $\tilde Z_t$ is $s_t(\cdot)$.
Now, for convenience, we now write $\tau$ to denote $\tau(t)$.
Taking derivative w.r.t. $t$ on both sides of \eqref{eq: dtau/dt} yields
\[
\frac{d \tau}{dt } =  \frac{1}{\beta(N - \tau)}. \frac{(1-t)}{t \sigma_t^2}.
\]

Using these facts, we finally get 
\begin{equation*}
    \begin{aligned}
        d \tilde Z_t & = \frac{(2t - 1)}{\sigma_t} \overleftarrow Y_\tau\,dt + \sigma_t d\overleftarrow Y_{\tau} \\
        & = \frac{(2t -1)}{\sigma_t^2} \tilde Z_t\,dt + \sigma_t \bc{\overleftarrow Y_\tau + 2 s_\tau^\leftarrow (\overleftarrow Y_\tau)} \beta(N - \tau) d\tau + \sigma_t \sqrt{2 \beta(N - \tau)} d \overleftarrow B_\tau\\
        & = \frac{(2t -1)}{\sigma^2_t} \tilde Z_t\,dt + \sigma_t \bc{\frac{\tilde Z_t}{\sigma_t} + 2 \sigma_t s_t(\tilde Z_t)} \frac{(1-t)}{t \sigma_t^2} dt + \sqrt{2 \br{\frac{1-t}{t}}} d \tilde B_t\\
        & = \br{ \frac{2t - 1}{\sigma_t^2} + \frac{1-t}{t \sigma_t^2}} \tilde Z_t\,dt + 2 \br{\frac{1-t}{t}} s_t(\tilde Z_t)\,dt + \sqrt{2 \br{\frac{1-t}{t}}} d \tilde B_t\\
        & = \frac{\tilde Z_t}{t}\,dt +  2 \br{\frac{1-t}{t}} s_t(\tilde Z_t)\,dt + \sqrt{2 \br{\frac{1-t}{t}}} d \tilde B_t.
    \end{aligned}
\end{equation*}
Therefore, we $\tilde Z_t$ satisfies SDE \eqref{eq: RF-SDE}.

\subsection{Proof of Proposition \ref{prop: SL and general interpolant are equivalent}}
\label{app: proof of SL and general interpolants equivalence}
Recall that $\frac{\tilde I_{\theta(s)}}{a_{\theta(s)}} = X_1 + W_{r^2_{\theta(s)}}$, where $r^2_{\theta(s)} = 1/s$. Now, we define $\hat U_s:= \frac{\tilde I_{\theta(s)}}{a_{\theta(s)}}$. Then, we have $s \hat U_s = s X_1 + s W_{1/s}$. Similar to the proofs of Lemma \ref{lemma: time change DDPM-SL} and Lemma \ref{lemma: SL and RF are equivalent}, the process $\hat W_s=sW_{1/s}$ for $s>0$, extended by $\hat W_0=0$, is also a Brownian motion. This finishes the proof by the description of the SL process in \eqref{eq: SL process}.

\section{Auxiliary results and discussions}
\label{app: auxiliary results}

\begin{figure}[h!]
    \centering
    \begin{subfigure}{0.45\textwidth}
        \includegraphics[width=\textwidth]{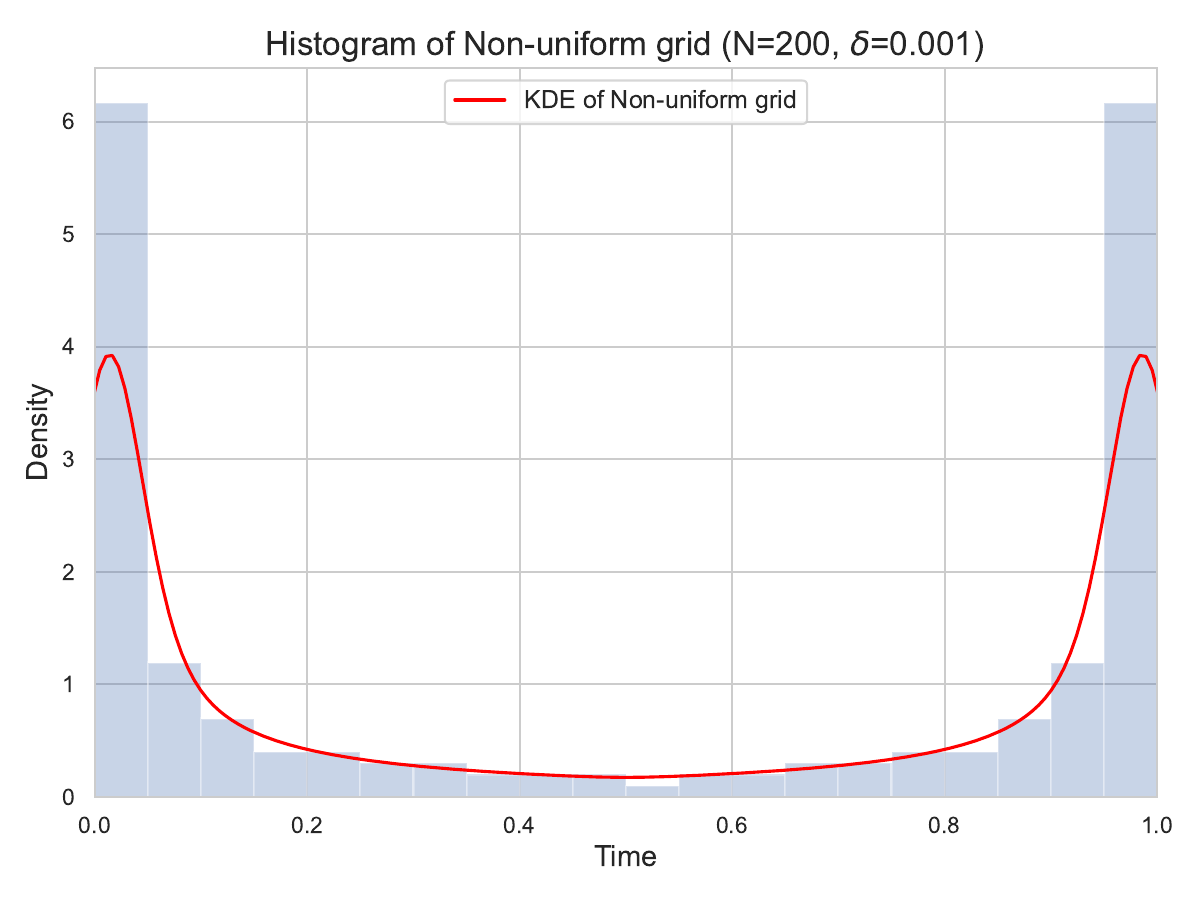}
        \caption{Histogram and kernel density estimation (KDE) plot of time-grid \eqref{eq: gemoetric time discretization} showing U-shaped distribution.}
    \end{subfigure}\hfill
    \begin{subfigure}{0.45\textwidth}
       \includegraphics[width=\textwidth]{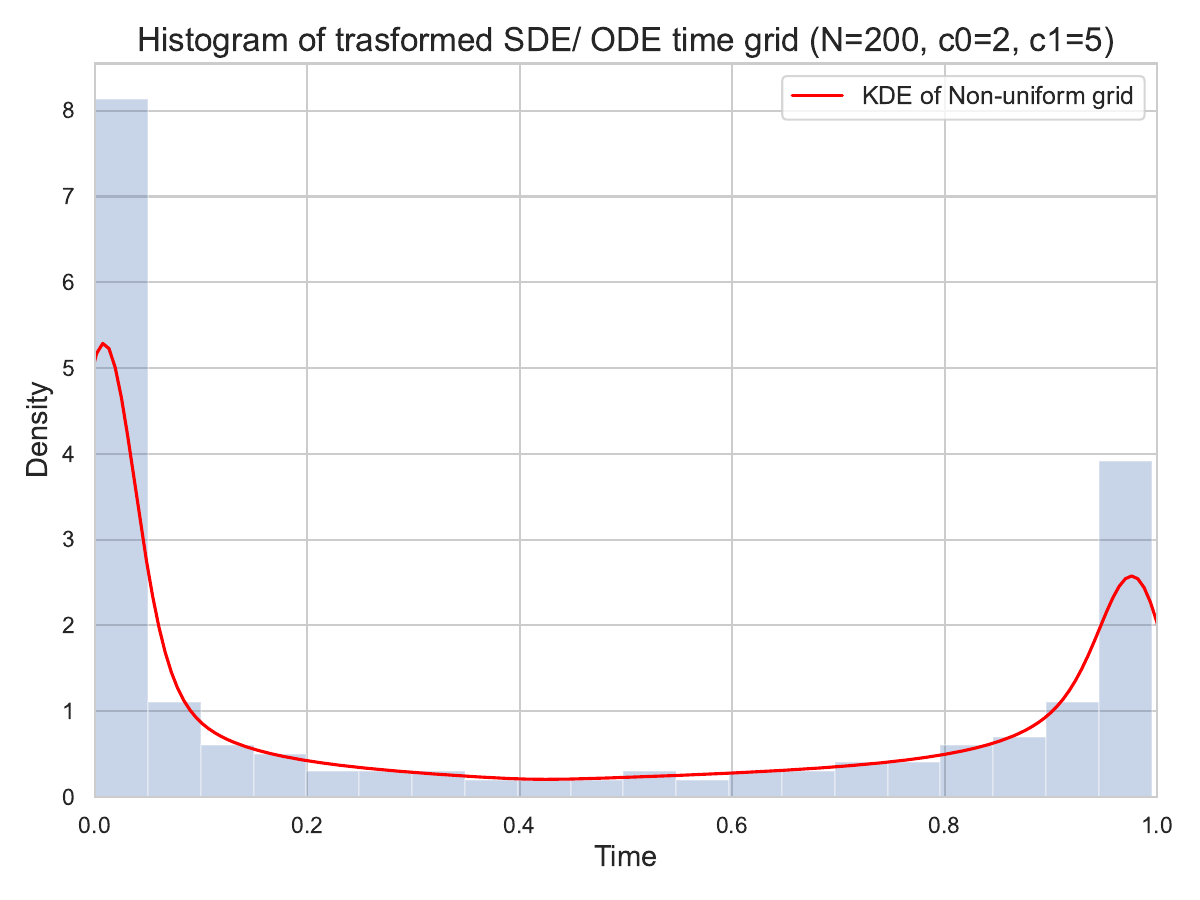} 
       \caption{Histogram and KDE plot of the transformed time-grid \eqref{eq: DDPM time scheduling} showing U-shaped distribution.}
    \end{subfigure}
    \caption{ Histogram of non-uniform time grids.}
    
    \label{fig: hist of sde timegrid}
\end{figure}
\subsection{Time discretization and admissible schedules}
\label{sec:time-schedule admissibility}

The leading discretization term in Theorem \ref{thm: RF - general convergence of Eulerian scheme} is $(h+Nh^2)k\log N$, where $h$ controls the endpoint step ratios in Assumption \ref{assumption: assumption of time schedule}. To compare different time schedules, write
\[
R_N:=\max_{1\le i\le N-2}\left\{
\frac{\eta_i}{t_i}\mathbbm{1}\{t_i\le1/2\}
+\frac{\eta_i}{1-t_{i+1}}\mathbbm{1}\{t_i>1/2\}
\right\},\qquad \eta_i=t_{i+1}-t_i.
\]
We may take $h=R_N$, provided $R_N\le1/4$ and $16C_3kR_N\log N<1$. Therefore, a sufficient condition for the coefficient of $k$ in the leading term to vanish is
\[
(R_N+NR_N^2)\log N=o(1).
\]
In particular, $R_N=o((N\log N)^{-1/2})$ is sufficient. If $k$ also grows with $N$, we additionally require $k(R_N+NR_N^2)\log N=o(1)$ for the leading term itself to vanish. The approximation-error, first-step, and remainder terms in the theorem must also be controlled.

The definition of $R_N$ explains why the time steps should become finer near both endpoints. Near $t=0$, we require $\eta_i$ to be small relative to $t_i$, whereas near $t=1$, we require it to be small relative to $1-t_{i+1}$. This naturally leads to ``U-shaped'' schedules, with more discretization points near the endpoints and larger steps in the interior.

A convenient way to construct such schedules is to introduce the odds coordinate $r(t):=t/(1-t)$ and the normalized increment
\[
a_i:=\frac{\eta_i}{t_i(1-t_{i+1})}
=\frac{r(t_{i+1})}{r(t_i)}-1.
\]
Since both endpoint step ratios are bounded by $a_i$, we have $R_N\le\max_i a_i$. Thus, an upper bound $\max_i a_i\le\gamma_N$ allows us to take $h=\gamma_N$. This gives a simple sufficient condition for admissibility, together with the small-step conditions of the theorem. The auxiliary sum $s_N$ used in the proof is not needed for this comparison.

Table~\ref{tab:time-schedules} summarizes several representative schedules. For the EDM-type and logit schedules, we take
\[
t_1=N^{-\alpha},\qquad t_{N-1}=1-\delta,\qquad \delta=N^{-\beta},
\]
where $\alpha,\beta>0$ are fixed. These choices make the dependence on both endpoints explicit. The choice $\delta=1/(N\vee d)$ has this form when $N\vee d=N^\beta$. The last column concerns only the leading coefficient of $k$; the remaining conditions are discussed below.

\begin{table*}[t]
\centering
\caption{Examples of time discretization schedules and their leading discretization coefficients. The small-step condition $16C_3kh\log N<1$ and the remainder bounds must be checked separately. For the EDM-type row, $\rho>\alpha+\beta$ is assumed.}
\label{tab:time-schedules}
\small
\setlength{\tabcolsep}{5pt}
\resizebox{\textwidth}{!}{%
\begin{tabular}{lllll}
\toprule
\rowcolor[HTML]{C5E1A5}
Schedule & Construction & Endpoint bound & $(h+Nh^2)\log N$ & Tends to zero?\\
\midrule
Linear
& $t_i=i/N$
& $R_N=1$
& $(1+N)\log N$ at $h=R_N$
& No; small-step condition fails\\
Two-sided geometric
& \shortstack[l]{$\eta_i\lesssim h_Nt_i$ for $t_i\le1/2$;\\$\eta_i\lesssim h_N(1-t_{i+1})$ for $t_i>1/2$}
& $R_N=O(h_N)$
& $O((h_N+Nh_N^2)\log N)$
& If $h_N=o((N\log N)^{-1/2})$\\
EDM-type
& \shortstack[l]{$r_i^{1/\rho}$ equally spaced;\\$r_i=t_i/(1-t_i)$}
& $R_N\asymp N^{-1+(\alpha+\beta)/\rho}$
& $\Theta(N^{-1+2(\alpha+\beta)/\rho}\log N)$
& If $\rho>2(\alpha+\beta)$\\
Logit-uniform
& $\logit(t_i)$ equally spaced
& $R_N=O(\log N/N)$
& $O(\log^3N/N)$
& Yes, for fixed $\alpha,\beta>0$\\
Logit-warped
& \shortstack[l]{$\logit(t_i)=a_N+(b_N-a_N)$\\$\times g((i-1)/(N-2))$}
& $R_N=O(\log N/N)$
& $O(\log^3N/N)$
& Yes, if $\|g'\|_\infty$ is bounded\\
\bottomrule
\end{tabular}%
}
\end{table*}

\subsubsection{Verification of admissibility}
\label{sec: discussion of time schedule admissibility}
We briefly verify the endpoint step bounds and the leading coefficients for the schedules in Table~\ref{tab:time-schedules}.

\paragraph{Linear schedule.}
For $t_i=i/N$, the two relevant ratios are $1/i$ and $1/(N-i-1)$, respectively. Hence $R_N=1$. In particular, no $h\le1/4$ can satisfy the endpoint step bound, so this schedule does not meet the hypotheses of the general theorem.

\paragraph{Two-sided geometric schedule.}
Suppose $\eta_i\lesssim h_Nt_i$ for $t_i\le1/2$ and $\eta_i\lesssim h_N(1-t_{i+1})$ for $t_i>1/2$. Then $R_N=O(h_N)$ directly from its definition. Taking $h$ to be a sufficiently large constant multiple of $h_N$, the leading coefficient is $O((h_N+Nh_N^2)\log N)$. Therefore, $h_N=o((N\log N)^{-1/2})$ is sufficient for this coefficient to vanish, and the small-step condition holds when $kh_N\log N$ is sufficiently small.

For the particular schedule \eqref{eq: gemoetric time discretization}, we have $t_1=\delta$, $h\asymp\log(1/\delta)/N$, and $h\lesssim Nh^2$. Thus, for $\delta=N^{-\beta}$ with fixed $\beta>0$, the leading coefficient is $O(\log^3N/N)$. The sufficient condition $k\log^2N/N=o(1)$ ensures $16C_3kh\log N<1$ for sufficiently large $N$.

\paragraph{EDM-type schedule.}
Consider the power-law interpolation of the odds coordinate in the table. Let $r_i=z_i^\rho$, where
\[
z_i=z_{\min}+(i-1)\Delta,\qquad
z_{\min}=r(t_1)^{1/\rho},\qquad
z_{\max}=r(t_{N-1})^{1/\rho},\qquad
\Delta=\frac{z_{\max}-z_{\min}}{N-2}.
\]
The prescribed endpoints give
\[
z_{\min}\asymp N^{-\alpha/\rho},\qquad
z_{\max}\asymp N^{\beta/\rho},\qquad
\frac{\Delta}{z_{\min}}\asymp N^{-1+(\alpha+\beta)/\rho}.
\]
In particular, $\Delta$ is not of order $N^{-1}$ when the terminal odds grow with $N$. For $\rho>\alpha+\beta$, uniformly in $i$,
\[
a_i=\left(1+\frac{\Delta}{z_i}\right)^\rho-1
=(1+o(1))\frac{\rho\Delta}{z_i}.
\]
Thus $R_N\le\max_i a_i=a_1$. Also, $\eta_1/t_1=a_1(1-t_2)\sim a_1$, since $t_2\to0$. Consequently,
\[
R_N\asymp N^{-1+(\alpha+\beta)/\rho},\qquad
(R_N+NR_N^2)\log N\asymp N^{-1+2(\alpha+\beta)/\rho}\log N.
\]
Therefore, $\rho>2(\alpha+\beta)$ ensures decay of the leading coefficient. A sufficient small-step condition is $kN^{-1+(\alpha+\beta)/\rho}\log N=o(1)$. These calculations concern the specific odds-coordinate construction displayed in the table.

\paragraph{Logit-uniform schedule.}
If $r(t_{i+1})=q_Nr(t_i)$, then
\[
a_i=q_N-1,\qquad
\log q_N=\frac{\log r(t_{N-1})-\log r(t_1)}{N-2}
\sim\frac{(\alpha+\beta)\log N}{N}.
\]
It follows that $R_N\le q_N-1=O(\log N/N)$. Taking $h=q_N-1$, we obtain $(h+Nh^2)\log N=O(\log^3N/N)$. Again, $k\log^2N/N=o(1)$ suffices for the small-step condition.

\paragraph{Logit-warped schedule.}
Let $a_N=\logit(t_1)$, $b_N=\logit(t_{N-1})$, and
\[
u_i:=\logit(t_i)=a_N+(b_N-a_N)g\left(\frac{i-1}{N-2}\right),
\]
where $g$ is strictly increasing, $g(0)=0$, $g(1)=1$, and $\|g'\|_\infty$ is bounded independently of $N$. Since $b_N-a_N=O(\log N)$,
\[
0<u_{i+1}-u_i\le\frac{\|g'\|_\infty(b_N-a_N)}{N-2}
=O\left(\frac{\log N}{N}\right).
\]
Using $a_i=e^{u_{i+1}-u_i}-1$, we obtain $R_N=O(\log N/N)$. Taking $h$ of this order gives the same leading coefficient and sufficient small-step condition as the logit-uniform schedule.

\paragraph{Remainder and approximation-error terms.}
The schedule conditions above control the leading coefficient and the applicability of the covariance proposition. To obtain a vanishing total error, we must also control
\[
\frac{h+Nh^2}{\delta^2N^{10}}+\frac{h^3}{\delta^6N^9},
\]
as well as the approximation-error and first-step terms in Theorem \ref{thm: RF - general convergence of Eulerian scheme}. For the geometric and logit schedules with $h=O(\log N/N)$ and $\delta=N^{-\beta}$, the displayed remainder is bounded by
\[
O\left(N^{2\beta-11}\log^2N+N^{6\beta-12}\log^3N\right).
\]
Thus $0<\beta<2$ is sufficient for this remainder bound to vanish. For $\delta=1/(N\vee d)$, the corresponding expressions are
\[
O\left(\frac{(N\vee d)^2\log^2(N\vee d)}{N^{11}}
+\frac{(N\vee d)^6\log^3(N\vee d)}{N^{12}}\right).
\]
These terms need not vanish for arbitrary growth of $d$.

For the EDM-type construction with $\rho>2(\alpha+\beta)$, substitution of $h\asymp N^{-1+(\alpha+\beta)/\rho}$ gives the remainder bound
\[
O\left(N^{2\beta-11+2(\alpha+\beta)/\rho}
+N^{6\beta-12+3(\alpha+\beta)/\rho}\right).
\]
Both exponents must be negative for this bound to vanish. These additional requirements illustrate why decay of the leading coefficient alone is not sufficient for convergence of the full bound.

\subsection{Controlling posterior covariance of RF}
\label{sec: RF covariance}

We recall the SL process \eqref{eq: SL process} to study the dynamics of the conditional covariance matrix $ \bSigma_t := \Cov_{1 \mid t}(X_t)$ where $X_t = t X_1 + (1-t)X_0$ and $X_0 \sim \cN(0, I_d)$. Note that $\bSigma_t$ is a random matrix. Recall the defintion of $\tilde X_t$ in \eqref{eq: RF linear stoc process}. It is evident that $\tilde \bSigma_t:= \Cov_{1 \mid t}(\tilde X_t) \overset{d}{=}\bSigma_t  $. Due to Lemma  \ref{lemma: SL and RF are equivalent}, the SL process \eqref{eq: SL process} is equivalent to the process \eqref{eq: RF linear stoc process} under time change $t(s):= \frac{\sqrt{s}}{1 + \sqrt{s}}$. Then, one can substitute this in Lemma \ref{lemma: eldan ddt cov lemma} to get the following result:
\begin{lemma}[\Cref{cor: my ddt cov version}]
    \label{lemma: RF ddt cov lemma}
    for $t \in [0,1)$, we have $\frac{d}{dt} \bbE (\bSigma_t) = - \frac{2t}{(1-t)^3} \bbE(\bSigma_t^2)$.
\end{lemma}
\begin{proof}
    Now, we note that $\bA_s$ and $\tilde \bSigma_{t(s)}$ (so is $\bSigma_{t(s)}$) are marginally equal in law. Therefore, by chain rule we have the following:

\begin{align*}
    \frac{d}{dt} \bbE(\bSigma_t) & = \frac{d}{dt} \bbE(\tilde \bSigma_t)\\
    &= \frac{d}{dt} \bbE (\bA_s)\\
    & = \frac{2t}{(1-t)^3} \frac{d}{ds} \bbE (\bA_s)\\
    & = - \frac{2t}{(1-t)^3} \bbE(\bA_s^2)\\
    & = - \frac{2t}{(1-t)^3} \bbE(\bSigma_t^2).
\end{align*}
\end{proof}
As mentioned in previous section, the above result allows to control the discretization error of the ODE and SDE sampler of RF. More concretely, Lemma \ref{lemma: RF ddt cov lemma} allows us to bound the spectra of $\bSigma_t^2$ in terms of the spectra of $\bSigma_t$, thereby producing tighter bounds for convergence rate.

\subsection{Proof of Proposition \ref{prop: Main cov bound via trace}}
\label{app: proof of Cov Frobenius bound}


We recall the definitions of $(\bA_s)_{s>0}$ and $(\tilde \bSigma_{t(s)})_{s>0}$ in Section \ref{sec: SL} and Section \ref{sec: RF covariance}. Note that $(\bA_s)_{s>0}$ and $(\tilde \bSigma_{t(s)})_{s>0}$ (see Section \ref{sec: SL} and \ref{sec: RF covariance}) are also equivalent in law. This is simply due to the fact that $\tilde \bSigma_{t(s)} = \Cov_{1 \mid t(s)}(\tilde X_{t(s)}) = \Cov_{1 \mid t(s)}(s \tilde X_{t(s)}/t(s))$, and the distributional equivalence of the two processes $(U_s)_{s>0}$ and $(s \tilde X_{t(s)}/t(s))_{s>0}$. It is known that the random matrix valued process $(\bA_s)_{s>0}$ obeys the following SDE:
\begin{equation}
\label{eq: SDE A_s}
d  \bA_s = -  \bA_s^2 ds + \cM_s^{(3)} dB_s,
\end{equation}
where $\cM_s^{(\ell)} = \bbE[(X_1 - \bbE[X_1 \mid U_s])^{\otimes \ell} \mid U_s]$. More details on this can be found in \citep[Section 4.2.1]{eldan2020taming} or \citep[Section 3]{alberts2025joint}. Since the two processes $(\bA_s)_{s>0}$ and $(\tilde \bSigma_{t(s)})_{s>0}$ are equivalent in law, SDE \eqref{eq: SDE A_s} immediately yields
\begin{equation}
    \label{eq: SDE Sigma_t}
    d \tilde \bSigma_t = - \frac{2t}{(1-t)^3} \tilde \bSigma_t^2 dt + \tilde \cM_t^{(3)} dB_{t^2/(1-t)^2},
\end{equation}
where $\tilde \cM_t^{(\ell)} = \bbE[(X_1 - \bbE[X_1 \mid \tilde X_t])^{\otimes \ell} \mid \tilde X_t]$.
Now, an application of Ito's formula on \eqref{eq: SDE A_s} yields 
\[
d(\tr( \bA_s^2)) = -  2\innerprod{ \bA_s,  \bA_s^2}ds  + \innerprod{ \cM_s^{(3)}, \cM_s^{(3)}} ds + 2 \innerprod{\bA_s, \cM_s^{(3)}} dB_s
\]
Taking expectation on both side we arrive at the following inequality:
\[
d(\tr[\bbE(\bA_s^2)]) = - 2 \bbE\left[\innerprod{\bA_s, \bA_s^2}\right]ds + \bbE\left[\innerprod{ \cM_s^{(3)}, \cM_s^{(3)}}\right]ds.
\]
Finally, using the fact that $\bA_s$ and $\bSigma_{t(s)}$ have the same law, chain rule yields the following:
\begin{equation}
    \label{eq: Moment Sigma_t ODE}
    d(\tr[\bbE(\bSigma_{t}^2)]) = - \frac{4t}{(1-t)^3} \bbE\left[\innerprod{\bSigma_t, \bSigma_t^2}\right]dt + \frac{2t}{(1-t)^3} \bbE\left[\innerprod{ \tilde \cM_t^{(3)}, \tilde\cM_t^{(3)}}\right]dt.
\end{equation}
We first focus on the first term on the right-hand side of \eqref{eq: Moment Sigma_t ODE}. Using the symmetry of $\bSigma_t$, we observe that

\begin{equation}
    \label{eq: bound on the leading term of sde}
    \begin{aligned}
    &\bbE \left[\innerprod{\bSigma_t, \bSigma_t^2}\right] = \bbE \left[\tr(\bSigma_t^3)\right]\\
    & \le \bbE \left[\Norm{\bSigma_t}_{\op} \Norm{\bSigma_t}^2_F\right] \\
    & \le \bbE \left[\tr\left(\bSigma_t\right) \Norm{\bSigma_t}^2_F\right]\\
    & =  \bbE \left[\tr\left(\bSigma_t\right) \ind\{X_t \in \cT_t\} \Norm{\bSigma_t}^2_F\right] + \bbE \left[\tr\left(\bSigma_t\right) \ind\{X_t \notin \cT_t\}\Norm{\bSigma_t}^2_F\right]\\
    & \le C_3 \frac{(1 - t)^2}{t^2}(k \log N) \bbE \left[\Norm{\bSigma_t}_F^2\right] + \bbE \left[\tr\left(\bSigma_t\right) \ind\{X_t \notin \cT_t\}\Norm{\bSigma_t}^2_F\right],
    \end{aligned}
\end{equation}
where the last inequality follows due to \eqref{eq: trace_Cov bound}. Regarding the last term in above display, Assumption \ref{assumption: bounded support} and \eqref{eq: prob of exclusion} yields
\[
\bbE \left[\tr\left(\bSigma_t\right) \ind\{X_t \notin \cT_t\}\Norm{\bSigma_t}^2_F\right] \le R^6\pr(X_t\notin\cT_t)=N^{6c_R}\pr(X_t\notin\cT_t) \le \frac{1}{N^{10}}.
\]
Here we used $\tr(\bSigma_t)\le R^2$ and $\|\bSigma_t\|_F^2\le R^4$, and chose $C_2$ sufficiently large that $C_2k\ge6c_R+10$. Therefore, we arrive at
\begin{equation}
    \label{eq: bound on Sigma_t^3}
    \bbE \left[\innerprod{\bSigma_t, \bSigma_t^2}\right] \le C_3 \frac{(1 - t)^2}{t^2}(k \log N) \bbE \left[\Norm{\bSigma_t}_F^2\right]  + \frac{1}{N^{10}}.
\end{equation}
Substituting this in \eqref{eq: bound on the leading term of sde} yields,
\begin{align*}
    d\bbE(\Norm{\bSigma_t}_F^2) & = -\frac{4t}{(1-t)^3} \bbE\left[\innerprod{\bSigma_t, \bSigma_t^2}\right]dt + \frac{2t}{(1-t)^3} \bbE\left[\innerprod{ \tilde \cM_t^{(3)}, \tilde \cM_t^{(3)}}\right]dt\\
    & \ge - \frac{4t}{(1-t)^3} \bbE\left[\innerprod{\bSigma_t, \bSigma_t^2}\right]dt\\
    &\ge - \frac{4 C_3 k \log N}{(1-t)t}\bbE \left[\Norm{\bSigma_t}_F^2\right] dt - \frac{4t}{(1-t)^3}. \frac{1}{N^{10}} dt\\
    & \ge - \frac{4 C_3 k \log N}{(1-t_i)t_{i-1}}\bbE \left[\Norm{\bSigma_t}_F^2\right] dt - \frac{4t_i}{(1-t_{i})^3}. \frac{1}{N^{10}} dt,
\end{align*}
where the last line follows holds provided $t \in [t_{i-1}, t_i]$ with $i> 1$. Therefore, we can conclude that
\begin{align*}
    & \exp\left\{\frac{4 C_3 k t\log N}{(1-t_i)t_{i-1}}\right\} \bbE (\Norm{\bSigma_{t}}_F^2) -  \exp\left\{\frac{4 C_3 k t_{i-1}\log N}{(1-t_i)t_{i-1}}\right\} \bbE (\Norm{\bSigma_{t_{i-1}}}_F^2)\\
    & \ge - \frac{4t_i}{(1-t_{i})^3}. \frac{1}{N^{10}} \int_{t_{i-1}}^t \exp\left\{\frac{4 C_3 k u\log N}{(1-t_i)t_{i-1}}\right\}du\\
    & \ge - \frac{t_i t_{i-1}}{C_3 k (1-t_{i})^2 \log N}. \frac{1}{N^{10}} \left(\exp\left\{\frac{4 C_3 k t\log N}{(1-t_i)t_{i-1}}\right\} - \exp\left\{\frac{4 C_3 k t_{i-1}\log N}{(1-t_i)t_{i-1}}\right\}\right)
\end{align*}
Dividing both sides by $\exp\left\{\frac{4C_3kt_{i-1}\log N}{(1-t_i)t_{i-1}}\right\}$ and rearranging, we obtain
\begin{align*}
\bbE(\Norm{\bSigma_{t_{i-1}}}_F^2)
&\le \exp\left\{\frac{4C_3k(t-t_{i-1})\log N}{(1-t_i)t_{i-1}}\right\}\bbE(\Norm{\bSigma_t}_F^2)\\
&\quad+\frac{t_it_{i-1}}{C_3k(1-t_i)^2\log N}\frac{1}{N^{10}}
\left(\exp\left\{\frac{4C_3k(t-t_{i-1})\log N}{(1-t_i)t_{i-1}}\right\}-1\right).
\end{align*}
Now, Assumption \ref{assumption: assumption of time schedule} and $h\le1/4$ give
\[
\frac{\eta_i}{t_i(1-t_{i+1})}
\le\begin{cases}
\dfrac{2h}{1-h}\le4h,&t_i\le1/2,\\[4pt]
2h\le4h,&t_i>1/2.
\end{cases}
\]
Indeed, in the first case $\eta_i\le ht_i$ and $1-t_{i+1}\ge(1-h)/2$, including when the step crosses $1/2$. In the second case, $\eta_i\le h(1-t_{i+1})$ and $t_i>1/2$. Hence, for $t\in[t_{i-1},t_i]$,
\[
\frac{4C_3k(t-t_{i-1})\log N}{(1-t_i)t_{i-1}}
\le\frac{4C_3k\eta_{i-1}\log N}{(1-t_i)t_{i-1}}
\le16C_3kh\log N<1.
\]
Using $e^x\le3$ and $e^x-1\le3x$ for $x\in[0,1]$, while retaining the factor $t_it_{i-1}$ in the remainder, gives
\begin{align*}
\bbE(\Norm{\bSigma_{t_{i-1}}}_F^2)
&\le3\bbE(\Norm{\bSigma_t}_F^2)
+\frac{12t_i(t-t_{i-1})}{(1-t_i)^3N^{10}}\\
&\le3\bbE(\Norm{\bSigma_t}_F^2)
+\frac{24h}{(1-t_i)^2N^{10}}.
\end{align*}
Here we used $t_i\le1$, $t-t_{i-1}\le\eta_{i-1}$, and $\eta_{i-1}/(1-t_i)\le2h$. The last bound follows from the same schedule assumption: for $t_{i-1}\le1/2$ it is at most $h/(1-h)\le2h$, and for $t_{i-1}>1/2$ it is at most $h$.
Replacing $i$ by $i+1$, we obtain
\[
\bbE(\Norm{\bSigma_{t_i}}_F^2)
\le3\bbE(\Norm{\bSigma_t}_F^2)+\frac{24h}{(1-t_{i+1})^2N^{10}},
\quad t\in[t_i,t_{i+1}],\quad 1\le i\le N-2.
\]

Now we use Lemma \ref{lemma: RF ddt cov lemma} to get 
\begin{equation*}
    \begin{aligned}
        \bbE [\tr(\bSigma_{t_{i+1}})] - \bbE[\tr(\bSigma_{t_{i}})] & =  -\int_{t_{i}}^{t_{i+1}} \frac{2t}{(1-t)^3} \bbE(\Norm{\bSigma_t}_F^2)dt\\
        & \le - \int_{t_i}^{t_{i+1}} \frac{2t }{3(1-t)^3 }\bbE(\Norm{\bSigma_{t_i}}_F^2)dt + \frac{8 h}{(1-t_{i+1})^2 N^{10}} \int_{t_i}^{t_{i+1}} \frac{2t}{(1-t)^3}dt\\
        & = - \frac{\eta_i(t_i+t_{i+1}-2t_it_{i+1})}{3(1-t_{i+1})^2 (1-t_i)^2} \bbE(\Norm{\bSigma_{t_i}}_F^2)\\
        & \quad + \frac{8 h \eta_i(t_i+t_{i+1}-2t_it_{i+1})}{(1-t_{i+1})^4 (1-t_i)^2 N^{10}}.
    \end{aligned}
\end{equation*}

Rearranging the above inequality, dividing by $t_i+t_{i+1}-2t_it_{i+1}>0$, and using $t_i+t_{i+1}-2t_it_{i+1}\ge 2t_i(1-t_{i+1})$ together with the nonnegativity of $\bbE[\tr(\bSigma_{t_i})]-\bbE[\tr(\bSigma_{t_{i+1}})]$, we obtain
\begin{equation}
    \label{eq: Main cov bound via trace}
    \begin{aligned}
    \frac{\eta_i t_i}{(1-t_{i+1})(1-t_i)^2} \bbE(\Norm{\bSigma_{t_i}}_F^2)
    &\le \frac{3}{2} \left(\bbE[\tr(\bSigma_{t_i})]-\bbE[\tr(\bSigma_{t_{i+1}})]\right)\\
    &\quad + \frac{24h\eta_i t_i}{(1-t_{i+1})^3(1-t_i)^2N^{10}}.
    \end{aligned}
\end{equation}

\subsection{Properties of  time scheduling \eqref{eq: gemoetric time discretization}}

\begin{lemma}
    \label{lemma: properties of time-steps}
    The time-steps \eqref{eq: gemoetric time discretization} satisfies the following for $i \in \{1, \ldots, N-2\}$:
    \begin{enumerate}[label = (\alph*)]
    \item\label{item: part a} We have $\eta_i = h t_i \ind\bc{t_{i} < 1/2} + h (1-t_{i+1}) \ind\bc{t_i \ge 1/2}$.
    
    \item\label{item: part b} We have $\frac{\eta_i}{1-t_i} \le h$ and $\frac{\eta_i}{1-t_{i+1}} \le h$.
    
     \item\label{item: part c} We have $\frac{ \eta_i^2 (1-t_i)^2 }{(1-t_{i+1})^2} \le h^2 $, as $\frac{\eta_i}{(1-t_{i+1})} \le h$ and $(1-t_i) \le 1$.
     
     \item\label{item: part d} We have $\sum_{i : t_i >0}\frac{ \eta_i^2 }{(1-t_{i+1})^2 t_i^2} = \sum_{i: 0 < t_{i} < 1/2 }\frac{ \eta_i^2 }{(1-t_{i+1})^2 t_i^2} + \sum_{i: t_{i}\ge 1/2 }\frac{ \eta_i^2 }{(1-t_{i+1})^2 t_i^2} \le  4 h^2 N$.
\end{enumerate}
\end{lemma}

\begin{proof}
{\color{white} Ghost line}
\begin{itemize}
    \item \textbf{Part \ref{item: part a}:} If $t_i < 1/2$, we have $\eta_i = t_{i+1} - t_i = (1+h)t_i - t_i = h t_i$. For the case $t_{i}\ge 1/2$, we have 
    \begin{align*}
        \eta_i &= t_{i+1} - t_i\\
        & = (1- t_i) - (1 - t_{i+1})\\
        & = h (1- t_{i+1}).
    \end{align*}

    \item \textbf{Part \ref{item: part b}:} For $t_i < 1/2$, we have $\frac{\eta_i}{1 - t_i} = ht_i/(1-t_i) \le h$. We use the fact that $\frac{t}{1-t}<1$ for $t \le 1/2$. Similarly, $\frac{\eta_i}{1 - t_{i+1}} = \frac{h t_i}{1 - t_{i+1}}$. Now, note that $t_i <1/2$ implies that $t_{i+1} \le 1/2$. Therefore, $\frac{t_i}{1 - t_{i+1}}\le 1$, and it yields $\frac{\eta_i}{1 - t_{i+1}} = \frac{h t_i}{1 - t_{i+1}} < h$.
    
    For $t_i\ge 1/2$, we have $\frac{\eta_i}{1-t_i} = \frac{h(1- t_{i+1})}{1 - t_i} < h$. Similarly, $\frac{\eta_i}{1-t_{i+1}} = \frac{h(1- t_{i+1})}{1 - t_{i+1}} = h$. This yields the coveted bound.

    \item \textbf{Part \ref{item: part c}:} The proof immediately follows from Part \ref{item: part b}.

    \item \textbf{Part \ref{item: part d}:} Fist, from Part \ref{item: part a}, we can note that $\frac{\eta_{i}}{t_i} \le h$. Also, Part \ref{item: part b} yields that $\frac{\eta_i}{1 - t_{i+1}} \le h$. Therefore, 
    \begin{align*}
        & \sum_{i : t_i >0}\frac{ \eta_i^2 }{(1-t_{i+1})^2 t_i^2}\\
        & = \sum_{i: 0 < t_{i} < 1/2 } \blue{\frac{ \eta_i^2 }{ t_i^2}}. \frac{1}{(1-t_{i+1})^2} + \sum_{i: t_{i}\ge 1/2 } \blue{\frac{ \eta_i^2 }{(1-t_{i+1})^2 }}. \frac{1}{t_i^2} \\
        & \le \sum_{i: 0 < t_{i} < 1/2 } 4 h^2 + \sum_{i: t_{i}\ge 1/2 } 4 h^2\\
        & \le 4h^2 N.
    \end{align*}
\end{itemize}

\end{proof}

\begin{lemma}
\label{lemma: simple identity 1}
Let $t_{i+1} = t_i + \eta_i$ with $0 < t_i \le t_{i+1} \le 1$. Then
\[
\frac{(1-t_i)^2}{t_i^2}
-
\frac{(1-t_{i+1})^2}{t_{i+1}^2}
\le
\frac{2\eta_i}{t_i^3}.
\]
\end{lemma}

\begin{proof}
Define $f(t) := \frac{(1-t)^2}{t^2}$ for $t > 0$. A direct calculation gives
\[
f'(t) = \frac{2(t-1)}{t^3}.
\]
By the mean value theorem, there exists $\xi \in (t_i, t_{i+1})$ such that
\[
f(t_i) - f(t_{i+1})
= -f'(\xi)\,(t_{i+1}-t_i)
= \frac{2(1-\xi)}{\xi^3}\,\eta_i.
\]
Since $\xi \ge t_i$ and $\xi \le 1$, we have
\(
(1-\xi)/\xi^3 \le 1/t_i^3
\),
which implies the desired bound.
\end{proof}

\end{document}